\documentclass[reqno, a4paper, 11pt]{amsart} 

\usepackage{float}\usepackage{colortbl}
\usepackage{amsmath, amssymb, amsthm, mathrsfs}
\usepackage{enumerate}
\usepackage{placeins}
\usepackage{xparse}
\usepackage{color}

\usepackage[numbers]{natbib}

\usepackage{hyperref}
\hypersetup{
	colorlinks = true,
	linkcolor = blue,
	anchorcolor = blue,
	citecolor = blue,
	filecolor = blue,
	urlcolor = blue
}
\usepackage{txfonts}

\usepackage{newtxtext}

\usepackage{caption}
\usepackage{accents}

\usepackage{booktabs}
\newcommand{\ra}[1]{\renewcommand{\arraystretch}{#1}}
\renewcommand{\arraystretch}{1.25}
\usepackage{lscape}

\newtheorem{theorem}{Theorem}
\newtheorem{definition}[theorem]{Definition}
\newtheorem{assumption}[theorem]{Assumption}

\newtheorem{corollary}[theorem]{Corollary}
\newtheorem{lemma}[theorem]{Lemma}
\newtheorem{proposition}[theorem]{Proposition}

\theoremstyle{remark}
\newtheorem{remark}[theorem]{Remark}

\theoremstyle{definition}
\newtheorem{examplex}[theorem]{Example}
\newenvironment{example}
  {\pushQED{\qed}\examplex}
  {\popQED\endexamplex}

 \newcommand{\eps}{\varepsilon}

 \renewcommand{\phi}{\varphi}
\newcommand{\pp}{\mathcal P}

\newcommand{\N}{\mathbb{N}}

\newcommand{\R}{\mathbb{R}}

\newcommand{\Z}{\mathbb{Z}}

\newcommand{\bes}{\begin{subequations}}
\newcommand{\ees}{\end{subequations}}
\newcommand{\eea}{\end{eqnarray}}

\newcommand{\RR}{{\mathbb R}}

\usepackage{bbm}

\renewcommand{\eps}{\varepsilon}
\renewcommand{\epsilon}{\varepsilon}

\newcommand{\fourIdx}[5]{%
\setbox1=\hbox{\ensuremath{^{#1}}}%
 \setbox2=\hbox{\ensuremath{_{#2}}}%
 \setbox5=\hbox{\ensuremath{#5}}%
 \hspace{\ifnum\wd1>\wd2\wd1\else\wd2\fi}%
 \ensuremath{\copy5^{\hspace{-\wd1}\hspace{-\wd5}#1\hspace{\wd5}#3}%
 _{\hspace{-\wd2}\hspace{-\wd5}#2\hspace{\wd5}#4}%
 }}

\numberwithin{equation}{section}
\numberwithin{theorem}{section}

\renewcommand{\subset}{\subseteq}

\usepackage{subcaption}
\usepackage{graphicx}

\renewcommand{\mathrm}{}

\makeatletter
\newcommand{\mylabel}[2]{#2\def\@currentlabel{#2}\label{#1}}
\makeatother

\usepackage{relsize}
\def\fcmp{\mathbin{\raise 0.6ex\hbox{\oalign{\hfil$\scriptscriptstyle \mathrm{o}$\hfil\cr\hfil$\scriptscriptstyle\mathrm{9}$\hfil}}}}

\newcommand{\Law}{\mathscr L}

\newcommand{\nn}{\mathbb{N}}
\newcommand{\rr}{\mathbb{R}}

\newcommand{\rrd}{\mathbb{R}^d}

\newcommand{\rrflex}[1]{{\ensuremath{\mathbb{R}^{#1}}}}
\newcommand{\zz}{\mathbb{Z}}
\newcommand{\kkk}{\mathscr{K}}
\newcommand{\xxx}{\mathscr{X}}
\newcommand{\yyy}{\mathscr{Y}}

\NewDocumentCommand{\NN}{oo}{
    \ensuremath{
        \mathcal{NN}
        \IfValueT{#1}{_{#1}}\IfValueF{#1}{_{[d]}}
        \IfValueT{#2}{^{#2}}\IfValueF{#2}{^{\sigma}}
    }
}
\NewDocumentCommand{\GT}{oo}{
    \ensuremath{
        \mathcal{GT}
        \IfValueT{#1}{_{#1}}\IfValueF{#1}{_{[d],N,q}}
        \IfValueT{#2}{^{#2}}\IfValueF{#2}{^{\sigma}}
        \left(
        \rr^d
            ,
        \yyy
        \right)
    }
}

\newcommand{\xb}{\boldsymbol{x}}

\usepackage{wasysym} %

\usepackage{marginnote}
\usepackage{xcolor}
\definecolor{darkcyan}{rgb}{0.0, 0.55, 0.55}
\definecolor{MidnightBlue}{RGB}{25,25,112}
\definecolor{MidnightBlueComplementingGreen}{RGB}{25,112,25}
\definecolor{MidnightBlueComplementingPurple}{RGB}{112,25,112}
\definecolor{MidnightBlueComplementingRed}{RGB}{112,25,69}
\definecolor{WowColor}{rgb}{.75,0,.75}
\definecolor{MildlyAlarming}{rgb}{0.85,0.25,0.1}
\definecolor{SubtleColor}{rgb}{0,0,.50}
\definecolor{antiquefuchsia}{rgb}{0.57, 0.36, 0.51}
\definecolor{fashionfuchsia}{rgb}{0.96, 0.0, 0.63}
\definecolor{jade}{rgb}{0.0, 0.66, 0.42}
\definecolor{caribbeangreen}{rgb}{0.0, 0.8, 0.6}
\definecolor{aquamarine}{rgb}{0.5, 0.8, 0.85}
\definecolor{attentioncolor}{RGB}{152,90,81}
\definecolor{burgred}{RGB}{40,3,22}
\definecolor{AnnieGreen}{RGB}{17,123,92}
\definecolor{Turquoise}{RGB}{64,224,208}
\definecolor{darkjade}{RGB}{0,122,84}

\definecolor{Window1}{RGB}{92,150,31}%
    \definecolor{Window1dark}{RGB}{41,67,13}%
\definecolor{Window2}{RGB}{255,168,28}
    \definecolor{Window2dark}{RGB}{114,75,12}
\definecolor{Window3}{RGB}{255,96,33}
    \definecolor{Window3dark}{RGB}{97,36,12}
\definecolor{InputColor}{RGB}{20,255,177}
    \definecolor{InputColorlight}{RGB}{222,237,229}

\usepackage{adjustbox}

\usepackage[colorinlistoftodos]{todonotes}

\NewDocumentCommand{\Bea}{mo}{
    \IfValueF{#2}{
                        {{
                            \textcolor{magenta}{ 
                            \textbf{B:}
                            \textit{{#1}}
                            }
                        }}
        }
    \IfValueT{#2}{
                        \marginnote{{\scriptsize
                            \textcolor{blue}{ 
                            \textbf{B:}
                            \textit{{#1}}
                            }
                        }}
        }
                    }

\NewDocumentCommand{\Gudi}{mo}{
    \IfValueF{#2}{
                        {{
                            \textcolor{caribbeangreen}{ 
                            \textbf{G:}
                            \textit{{#1}}
                            }
                        }}
        }
    \IfValueT{#2}{
                        \marginnote{{\scriptsize
                            \textcolor{caribbeangreen}{ 
                            \textbf{G:}
                            \textit{{#1}}
                            }
                        }}
        }
                    }

\NewDocumentCommand{\Annie}{mo}{
    \IfValueF{#2}{
                        {{
                            \textcolor{jade}{ 
                            \textbf{A:}
                            \textit{{#1}}
                            }
                        }}
        }
    \IfValueT{#2}{
                        \marginnote{{\scriptsize
                            \textcolor{jade}{ 
                            \textbf{A:}
                            \textit{{#1}}
                            }
                        }}
        }
                    }
            \NewDocumentCommand{\AnnieSuggestion}{mo}{
    \IfValueF{#2}{
                        {{
                            \textcolor{jade}{ 
                            \textit{{#1}}
                            }
                        }}
        }
    \IfValueT{#2}{
                        \marginnote{{\scriptsize
                            \textcolor{jade}{ 
                            \textit{{#1}}
                            }
                        }}
        }
                    }
\NewDocumentCommand{\AtoBpG}{mo}{
    \IfValueF{#2}{
                        {{\scriptsize
                            \textcolor{blue}{ 
                            \textbf{$A\rightsquigarrow G+B$:}
                            \textit{{#1}}
                            }
                        }}
        }
            
    \IfValueT{#2}{
                        \marginnote{{\scriptsize
                            \textcolor{blue}{ 
                            \textbf{$A\rightsquigarrow G+B$:}
                            \textit{{#1}}
                            }
                        }}
        }
                    }

\definecolor{deepjunglegreen}{rgb}{0.0, 0.29, 0.29}
\NewDocumentCommand{\AtoA}{mo}{
    \IfValueF{#2}{
                        {{\scriptsize
                            \textcolor{deepjunglegreen}{ 
                            \textbf{$A\rightsquigarrow A$:}
                            \textit{{#1}}
                            }
                        }}
        }
        
    \IfValueT{#2}{
                        \marginnote{{\scriptsize
                            \textcolor{deepjunglegreen}{ 
                            \textbf{$A\rightsquigarrow A$:}
                            \textit{{#1}}
                            }
                        }}
        }
                    }

\NewDocumentCommand{\AtoBpGpA}{mo}{
    \IfValueF{#2}{
                        {{\scriptsize
                            \textcolor{purple}{ 
                            \textbf{$A\rightsquigarrow G+B+A$:}
                            \textit{{#1}}
                            }
                        }}
        }
        
    \IfValueT{#2}{
                        \marginnote{{\scriptsize
                            \textcolor{purple}{ 
                            \textbf{$A\rightsquigarrow G+B+A$:}
                            \textit{{#1}}
                            }
                        }}
        }
                    }

\usepackage[normalem]{ulem}

\usepackage{adjustbox}

\newcommand{\eqdef}{\ensuremath{
        \overset{
                \scalebox{.5}{\mbox{def.}}
            }{=}
}}

\newcommand{\Ball}{\operatorname{Ball}}
\newcommand{\doesnothaveto}{{does not have to}}
\newcommand{\doesnotneedto}{{does not need to}}
\newcommand{\Holderseminorm}{{$\alpha$-H\"older constant}}

\usepackage{tablefootnote}
\title{Designing Universal Causal Deep Learning Models: \hfill\\The Geometric (Hyper)Transformer}

\author{Beatrice Acciaio}
\author{Anastasis Kratsios}
\author{Gudmund Pammer}

\begin{document}
\maketitle

\begin{abstract}
Several problems in stochastic analysis are defined through their geometry, and preserving that geometric structure is essential to generating meaningful predictions.  Nevertheless, how to design principled deep learning (DL) models capable of encoding these geometric structures remains largely unknown.  We address this open problem by introducing a universal causal geometric DL framework in which the user specifies a suitable pair of metric spaces $\mathscr{X}$ and $\mathscr{Y}$ and our framework returns a DL model capable of causally approximating any ``regular'' map sending time series in $\mathscr{X}^{\mathbb{Z}}$ to time series in $\mathscr{Y}^{\mathbb{Z}}$ while respecting their forward flow of information throughout time.  Suitable geometries on $\mathscr{Y}$ include various (adapted) Wasserstein spaces arising in optimal stopping problems, a variety of statistical manifolds describing the conditional distribution of continuous-time finite state Markov chains, and all Fr\'{e}chet spaces admitting a Schauder basis, e.g.\ as in classical finance. Suitable spaces $\mathscr{X}$ are compact subsets of any Euclidean space.  Our results all quantitatively express the number of parameters needed for our DL model to achieve a given approximation error as a function of the target map's regularity and the geometric structure both of $\mathscr{X}$ and of $\mathscr{Y}$.  Even when omitting any temporal structure, our universal approximation theorems are the first guarantees that H\"{o}lder functions, defined between such $\mathscr{X}$ and $\mathscr{Y}$ can be approximated by DL models.  
\end{abstract}


\textbf{MSC:} 
Artificial Neural Networks and Deep Learning (68T07), Optimal Transportation (49Q22), Abstract Approximation Theory (41A65), Analysis on Metric Spaces (30L99), Prediction Theory (60G25), Computational Methods for Stochastic Equations (60H35).

\section{Introduction}
\label{s_Introduction}
Due to breakthroughs in machine learning, optimization, and computing hardware, the last half-decade has seen a paradigm shift in many areas of applied mathematics, moving away from model-based approaches to model-free methods.  This is most apparent in computational stochastic analysis and mathematical finance, where deep learning has unlocked previously intractable problems.  Examples include computation of optimal hedges under market frictions and possibly rough volatility \cite{DeepHedging_2019,DeepHedging_Follow_Ups_QF,gier20,DeepHedging_Blanka_2021}, numerical implementation of complicated local stochastic volatility models \cite{CKT20}, numerical solutions to previously intractable principal-agent problems  \cite{PrincipalAgentProblems_Jaimungal_2021}, pricing of derivatives relying on optimal stopping rules written on high-dimensional portfolios \cite{Deep_Optimal_Stopping_OG_2019,Deep_Optimal_Stopping_New_Cheriditor_Jentzen_Becker_2021,Deep_Optimal_Stopping_Randomized_TeichmanKrachHerrera}, data-driven prediction of price formation using ultra-high dimensional limit orderbook data \cite{Deep_LOB_Sirignano_Cont_2019,Deep_LOB_New_2019}.

These deep learning-based methods are appealing not only because of their empirical success, but they are equally theoretically founded and are known to be able to approximately implement any ``reasonable'' function.  This latter feature of deep neural networks is known as their \textit{universal approximation property} and is the focal topic of this paper, in the context of  non-anticipative functions between discrete-time path spaces.  The universal approximation capabilities of classical neural network models are well-understood  \cite{hornik1990universal,LESHNO1993861,KidgerLyons2020,Kratsios_NEUTA_2020}.  Nevertheless, little is known about whether or not neural network-based models can be used to approximate general stochastic processes or how to design a deep neural model which could.  

The approximate implementation of ``any'' stochastic process' evolution, conditioned on its realized trajectory, is, of course, central to various areas of applied probability since this would help bridge the gap between abstract theoretical models and algorithmically deployed models.  Though our examples are framed in the context of adapted optimal transport and mathematical finance, various other intersectional areas of machine learning and applied probability theory can utilize the results obtained in the present paper, such as computational signal processing, numerical weather predictions, and many others.  

At this point, we make the leap towards abstraction and observe that the universal approximation of a stochastic process' evolution, conditioned on its realized trajectory, is only a special case of a broader phenomenon which we call a \textit{causal map}.  Briefly, given two metric spaces $\xxx$ and $\yyy$, a causal map $F:\xxx^{\zz}\rightarrow \yyy^{\zz}$ is a function which maps discrete-time paths in $\xxx$ to discrete-time paths in $\yyy$ while respecting the \textit{causal forward-flow of information} in time.
We refer to this as the \textit{dynamic case}
which incorporates a temporal flow into the \textit{static case} $f:\xxx\rightarrow \yyy$.  In the case of stochastic processes, $\xxx=\rr^d$ and $\yyy$ can be thought of as a space of laws of a process on a prespecified number of future steps, such as the Adapted Wasserstein space of \cite{Ruschendor_1985_AdaptedOG}, or a Fr\'{e}chet space of random-vectors such as a local $L^p$-space.  Nevertheless, the analysis we develop here is general enough to cover a broad range of discrete-time paths on suitable metric spaces $\xxx$ and $\yyy$, where $\xxx$ is a subspace of an Euclidean space and $\yyy$ is approximately representable by Euclidean information.

Concisely, our paper's main objective is to approximate any causal map between suitable discrete-time path spaces, which can have an arbitrarily long memory but which isn't overly reliant on the infinite past.  We do so by proposing a new geometric deep learning model, called \textit{geometric hypertransformers (GHTs)}, which naturally adapts to $\yyy$'s (non-Euclidean) geometry, and our main quantitative result can be informally summarized as:
\[
\mbox{\textit{GHTs can approximate any causal map $F:\xxx^\zz\rightarrow \yyy^\zz$ over any time-horizon.}}
\]
We intentionally leave the description of what ``any causal map'' is and in which sense they are ``approximated'' vague at this point.  This is to reflect the \textit{modular nature} of our modelling framework, which can accommodate a broad range of different processes and modes of approximation.  Examples include the approximation of any ``integrable'' stochastic process in a way which respects its adapted flow of information in the sense of \cite{Ruschendor_1985_AdaptedOG,BackhoffBartleBeiglebrockEder_2020_AlladatedTopEqual,backhoff2020adapted},
any square-integrable martingale in the Martingale-Hardy sense \citep{Weisz_1994_MartingaleHardySpaceFourier}, 
a broad range of minimal parametric models in the sense of information geometry \cite{amari2016information,ay2015information,AyJost_2017_InformationGeometry} but cast in the stochastic analysis setting, and the extreme but classical case of  discrete-time dynamical systems (which we understand as deterministic processes).

Our neural network model, illustrated in Figure~\ref{fig_metric_hypertransformer_TIKZ}, emulates 
$F$ by only flowing information forward in time.  The vertically-placed \emph{black boxes} in Figure~\ref{fig_metric_hypertransformer_TIKZ} are $\yyy$-valued counterparts of the \textit{transformer networks} of \cite{vaswani2017attention}, and variants of the probabilistic transformer networks of \cite{AB_2021,kratsios2021universal,kratsios2022small}.  Transformer networks are particularly appealing since, unlike recurrent neural networks, they can automatically process any input sequence without any recursion, which makes them much faster and more stable to train.  
Two advantages that transformers have over their recurrent neural network (RNN) \cite{RumelhardIntonWilliams_1986_RNNs_Introduction_OG} predecessors are that they avoid recursion and they learn how to encode any inputs before decoding them as predictions.   This allows transformers to avoid lengthy and unstable training, and it gives them the flexibility to focus on different features of a path without overemphasizing its most recent movements.  Thus, transformers have redefined the state of the art in sequential prediction tasks by effectively replacing LSTMs \cite{HochreiterSchmidhuber_1997_LSTM_OG}.
 
\begin{figure}[ht]%
%
%
%
\centering
\includegraphics[width=1\textwidth]{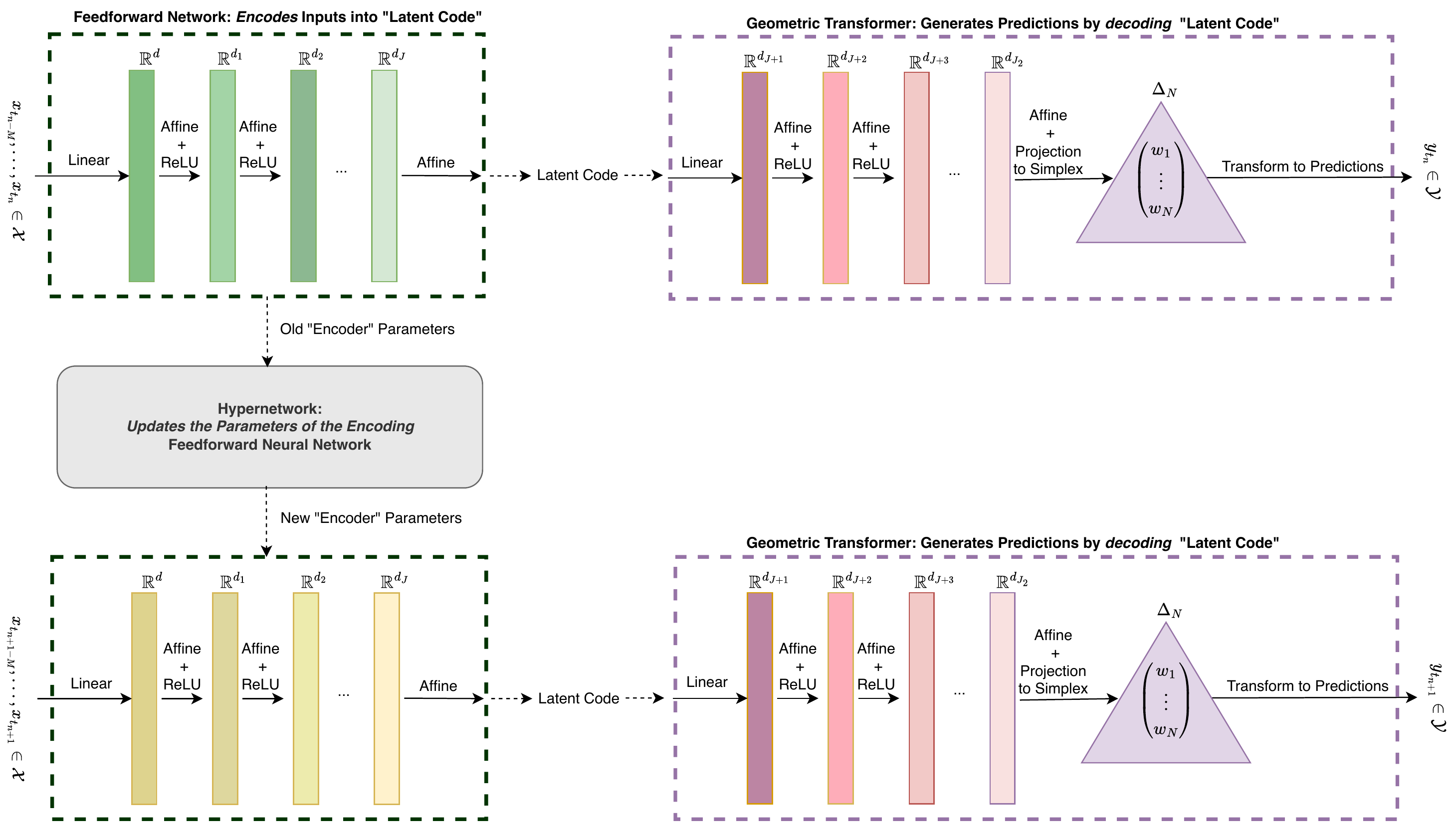}
\caption{
Illustration of the Geometric Hypertransformer (GHT): At every time-step, a feedforward neural network, illustrated by the green and yellow boxes, maps the current time-series segment in $\xxx$ to a ``latent code'' in some Euclidean space $\mathbb{R}^{d_J}$.  The geometric transformer, illustrated by the purple box, then transforms that ``latent code'' into the next prediction on $\yyy$.  Between each prediction the hypernetwork, illustrated by the gray box, updates the feedforward network's hidden weights so that the GHT architecture can adapt to changes in the time-series.}
\label{fig_metric_hypertransformer_TIKZ}
\end{figure}

Since we are not focusing on time-invariant $F$, that is $F$ which commute with time-shifts (i.e. $F(x_{\cdot +1})_{\cdot} = F(x_{\cdot})_{\cdot+1}$), then the parameters defining an GHT  \textit{efficiently} approximating $F$ must progressively update to accommodate changes in $F$. These updates, to our model are depicted in Figure~\ref{fig_metric_hypertransformer_TIKZ} by horizontal arrows are implemented by a very small neural network known as a \textit{hypernetwork} by \cite{ha2016hypernetworks}, acting on our model's parameter space.  Its role is to interpolate the parameters of each of our \textit{geometric transformer networks}, each acting on sequential segments of any input path from time $t_n$ to time $t_{n+1}$ across all $n\in \zz$ (depicted as vertical networks in Figure~\ref{fig_metric_hypertransformer_TIKZ}).  The point here is a quantitative one, namely, it is known that neural networks require far fewer parameters to memorize (or interpolate) a finite set of input and output pairs \citep{YunSraJadbabaie2019GoodMemoryCapacity__Memorization,Memory_Capacity__Memorization,MemoryCapacitySublinear2021SejunParkLeeJaehoChulheeShin__Memorization__Really_Approximation_inthis_paper} than what is required to uniformly approximate a function implementing this memorization.

In the context of stochastic processes, our model can be visualized using Figures~\ref{fig_metric_hypertransformer_TIKZ} and~\ref{fig_adapted_map_Example_visualization}.  Briefly, given a %
(possibly infinite) $d$-dimensional path $\xb$, our GHT model sequentially assimilates long segments of $\xb$ before forecasting the process' distribution over a fixed number\footnote{Here $\nn$ denotes the non-negative integers and $\nn_+$ denotes the positive integers.} $N_F\in \nn_+$ of future steps.  Then, once a new portion of the path is observed, our GHT model's internal parameters are updated, and the next process' law over the next $N_F$ increments, conditioned on the observed path, is predicted.  We illustrate this sequential procedure in Figures~\ref{fig_metric_hypertransformer_TIKZ} and~\ref{fig_adapted_map_Example_visualization} where the future set of steps of the process are colour coded to indicate the network's corresponding weights.

\begin{figure}[ht]%
    \centering
    \includegraphics[width=.40\linewidth]{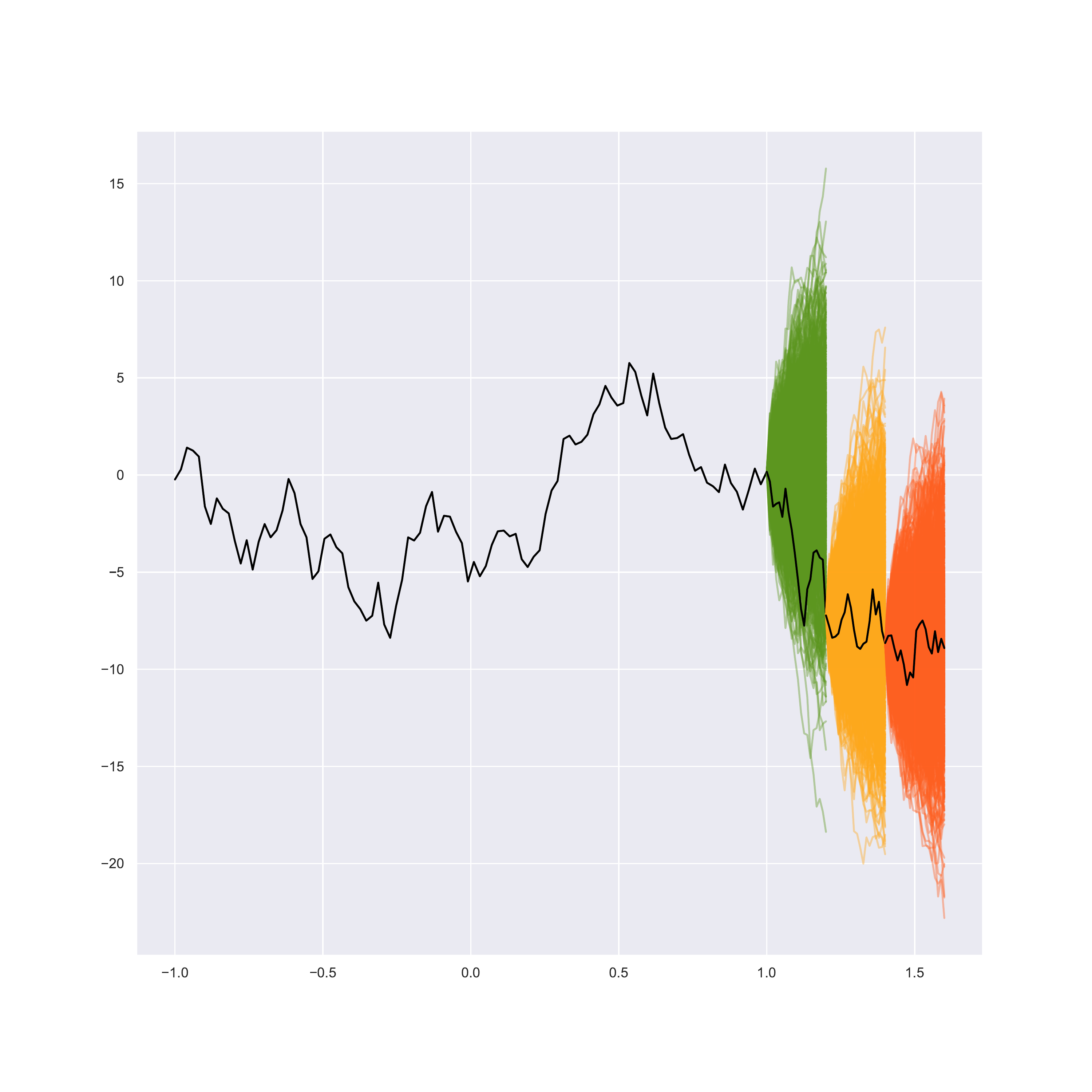}
    \caption{Illustration of a causal map which sends the sequence of realizations of a time-series to the law of its next few steps.  In this depiction, each of the colour represents the output of this causal map evaluated at different instances in time.}
    \label{fig_adapted_map_Example_visualization}
  \end{figure}

From the deep learning side of this interdisciplinary story, and to the best of authors' knowledge, the results presented here are the first approximation theoretic result advocating for the effectiveness of \textit{hypernetworks}.  Thus, we complement the extensive emerging empirical studies on hypernetworks (see \cite{ha2016hypernetworks,zhang2018graph,Oswald2020Continual}).

The hypernetwork defining the GHT weaves together instances of our main model for the static case, which we call \textit{geometric transformers (GTs)}. Even in the static case, we obtain a novel universal approximation theorem which can be summarized as
\[
\mbox{\textit{GTs can approximate any H\"{o}lder function $f:\xxx\rightarrow \yyy$,}}
\]
where $\yyy$ is as above, and $\xxx$ is a subspace of the Euclidean space $\rr^d$.  We emphasize that the output spaces covered by our results are much broader than the theoretically backed neural network architectures available in the literature.  The closest principled sequential deep learning results to ours are the theory of reservoir computers \cite{JPandLyydmila2019,Lukas2020a}, especially  echo-state networks \cite{GRIGORYEVA2018495,GONON202110}, as well as single-layer recurrent neural networks \cite{HUTTERRecepHelmut2021_Ent_Optimal_RNNs_LDSs}, each of which is known to be able to approximate specific classes of generalized discrete-time dynamical systems evolving in Euclidean spaces.

To fully appreciate geometric hypertransformers, let us overview the current state of the art in theoretically-founded modelling of stochastic processes.  
When approximating jump diffusion processes, \cite{gonon2021deepNeuralSDE} showed that solutions to stochastic differential equations whose coefficients are parameterized by neural networks, called \textit{neural SDEs} \cite{ChetRubanovaBettencourtDuvenaud_2018_NeuralODEs_OG,Neural_Jump_SDE_2019,Neural_Jump_SDE_2019,Lyons_2021_NeuralSDEsGANs,Lyons_2021_Neural_SDE}, can approximate any \textit{single} expected path-functional of a Markovian SDE with uniformly bounded Lipschitz coefficients.  Nevertheless, the approximation of the law of any regular SDE remains an open problem.   Alternatively, reservoir computers of \cite{JaegerHass_2004_ESN_OG,MaassNatschlagerMarkram_2002_LSM_OG} have been demonstrated to approximate time-invariant random dynamical systems with possibly infinitely long memory but \textit{fading memory} \cite{Lukas2020,JPandLyydmila2019} and in \cite{CuchieroGononGrigoryevaOrtegaTeichmann_2021_DiscTimeSigReservoir_UAT} using \textit{random signature-based approaches} drawing from rough path theory \cite{LyonsRoughPath_1994}.  

Circling back to the origins of recurrent models, we arrive at the extreme case of deterministic, stochastic processes, by which we mean discrete-time dynamical systems with an arbitrarily long memory. These are typically approximated by using RNNs with the LSTM architecture of \cite{HochreiterSchmidhuber_1997_LSTM_OG}.  Classical RNNs have long been known to have the capacity to approximate any \textit{computable function} \cite{siegelmann1995computational}, and preliminary results indicating that they are universal approximators of any regular dynamical systems on $\rr^d$ were suggested by \cite{schafer2006recurrent_Original_RNNfinite_timehorizon_finite_memory_dynamical_systems_Euclidean__NonRigerous__but_first}.  Nevertheless, it has only recently been shown in \cite{HUTTERRecepHelmut2021_Ent_Optimal_RNNs_LDSs} that classical RNNs, which aren't reservoir computers, can approximate any time-inhomogeneous linear dynamical system on $\rr^d$ with rapidly decaying memory.

Recently, various deep learning models mapping into infinite dimensional Banach spaces have been proposed such as the DeepONets of \cite{LuJinPanZhangKarniadakis_DeepONetNature_2021,LiuYandChenZhaoLiao_DeepOReLUNets_2022New}, the Fourier Neural Operators of \cite{KovachkiLanthalierMishra_FouerierNeuralOperator_OG_2021JMLR}, and the generalized feedforward model of \cite{benth2021neural} generalized to Fr\'{e}chet spaces (in the special case where the space of processes is a Martingale-Hardy space or a similar linear space of semi-Martingales, see \citep[Part IV]{CohenElliot_StochCalAppls_2015}).  Our models can cover approximation of functions taking values in Banach spaces, and, more generally in Fr\'{e}chet spaces.  Most notably our results are not limited to linear spaces and they cover a broad range of non-vectorial metric spaces: Wasserstein spaces, adapted Wasserstein spaces, many spaces arising from information geometry, and many others.  

It is worth noting that simultaneous vertical and horizontal connections have very recently and successfully been proposed in the few-shot image classification literature in the hypertransformer model of \cite{HypertransformerCNN_2021}.  In analogy with Figure~\ref{fig_metric_hypertransformer_TIKZ}, the vertical networks are the standard transformers of \cite{vaswani2017attention} and the horizontal networks are convolutional neural networks \cite{CNNs_LeCun_1089}.

\subsection*{Deep Learning in Mathematical Finance}
To fully gauge the potential impact of our work in mathematical finance, we briefly highlight some of the recent advances made possible by the deep learning methods in finance, in relation to sequential learning.  Recently, many novel neural network-based frameworks have been proposed for market simulation and data-driven model selection.  The ability to accurately model time-series data is critical for numerous applications in the finance industry, and neural networks offer different ways to tackle these problems as alternatives to classical approaches.
In particular, synthetic generation of time-series can be performed in a completely data-based fashion without imposing assumptions on the underlying stochastic dynamics.  The generated data can be used, for example, to facilitate training and validation of models.  Time-series generation has been successfully approached via autoencoders as e.g. in \cite{buehler2020data,wiese2021multi}, and via adversarial generation, as in \cite{kosh2019gen} with classical GANs, in \cite{AK22} with newly developed transport-theoretic techniques, and in \cite{ni2021sig,ni2020conditional} in conjunction with signature method.
Other than direct path generation, neural networks have been employed as function approximators for drift and diffusions of the modelled SDE system, for robust and data-driven model selection mechanisms, as e.g. done in \cite{arribas2020sig,CKT20,gier20,wiese2020quant,cohen2021arbitrage,kidger2021neural,van2021monte}.
\subsection*{Organization of Paper}
Our paper is organized as follows.  Section~\ref{s_Preliminaries} introduces the relevant mathematical and machine learning tools to formulate our geometric deep learning problem.  Section~\ref{s_Static_Case} covers our results in the static case, where we introduce a broad class of non-Euclidean metric spaces $\yyy$ as well as our geometric deep learning model, before showing that our model can approximate any H\"{o}lder function defined on a compact subset $\xxx$ and taking values in $\yyy$.   
Section~\ref{s_Dynamic_Case} then addresses the dynamic case by first introducing causal maps between the discrete-time path spaces $\xxx^{\zz}$ and $\yyy^{\zz}$, and then describing how the recurrent extension of our static model can approximate any causal map whose memory asymptotically vanishes.  Section~\ref{s_Examples} provides examples of spaces and causal maps covered by our framework, in particular relating to stochastic analysis and mathematical finance.  
Section~\ref{s_Proofs} contains all the paper's proofs, auxiliary metric-theoretic definitions, and needed technical analytic lemmas.

\section{Preliminaries}
\label{s_Preliminaries}
\subsection{Background}
\label{s_Preliminaries__ss_Background}
We recognize the interdisciplinary nature of this paper, and therefore, this section covers the necessary background in the areas of optimal transport, deep learning, and metric geometry.  The reader familiar with any of these areas is encouraged to skip the corresponding introductory section. 

\subsubsection{Optimal Transport}
\label{s_Preliminaries__ss_AdaptedOT}
A central feature of optimal transport is its ability to lift the geometry of a base space onto the set of probabilities.
We refer to \cite{OldandNew2010Villani} for a detailed introduction to optimal transport.
Typically for the field,  we write, when $(\xxx, d_\xxx)$ is a metric Polish space, $\mathcal P(\xxx)$ for the set of all Borel probability measures on $\xxx$ which is then equipped with the topology of weak convergence of measures.
When $1 \le p < \infty$, we denote by $\mathcal P_p(\xxx)$ the subset of probabilities that finitely integrate $x \mapsto d_\xxx(x,x_0)^p$ for some (and thus for any) $x_0 \in \xxx$.
Similarly, we equip $\mathcal P_p(\xxx)$ with the Wasserstein $p$-distance $\mathcal W_p$, that is, for $\mu,\nu \in \mathcal P_p(\xxx)$, the metric defined by
\[
    \mathcal W_p(\mu,\nu) 
     \eqdef  
    \inf_{\pi \in \operatorname{Cpl}(\mu,\nu)} 
    \left\{
        \int 
            d_\xxx(x,y)^p
            \,
            \pi(dx,dy)
    \right\}^{1 / p},
\]
where $\operatorname{Cpl}(\mu,\nu)  \eqdef  \{ \pi \in \mathcal P(\xxx \times \xxx) \colon \pi \text{ has first marginal } \mu, \text{ second marginal }\nu \}$.

In the context of mathematical finance, the \emph{adapted Wasserstein distance}, that is a variation of the Wasserstein distance, can be used to obtain sharp quantitative results for sequential decision making problems and in robust finance, see \cite{BackhoffBartlBeiglbrockEder_AWinMF_2020, WassersteinSpaceofStochasticProcesses} among others.
For $\mu,\nu\in\mathcal{P}_p(\mathbb{R}^{d T})$, the adapted Wasserstein-$p$ distance is defined by
\[
\mathcal{AW}_p(\mu,\nu) \eqdef \inf_{\pi\in\text{Cpl}_{\text{bc}}(\mu,\nu)}\left\{\int \sum_{t=1}^T|x_t-y_t|^p \pi(dx,dy)\right\}^{1/p},
\]
where $\text{Cpl}_{\text{bc}}(\mu,\nu)$ is the set of probability measures $\pi \in \operatorname{Cpl}(\mu,\nu)$ satisfying the bi-causality constraint: for all $t=1,\dots,T$,
\[
\pi(dy_1,\dots,dy_t|x_{1:T})=\pi(dy_1,\dots,dy_t|x_{1:t})\quad \text{and}\quad \pi(dx_1,\dots,dx_t|y_{1:T})=\pi(dx_1,\dots,dx_t|y_{1:t}),
\]
where we use $x_{1:T}$ to denote the vector $(x_1, x_2, \ldots, x_T)$.
We note that if $T=1$, $(\mathcal{P}_p(\mathbb{R}^{dT}),\mathcal{AW}_p)$ coincides precisely with the Wasserstein space $(\mathcal P_p(\rrd), \mathcal W_p)$.
\subsubsection{Feedforward Neural Networks with Parametric Activation Functions}
\label{s_Preliminaries__ss_ffNNs}
In this section we recall the definition of  \textit{feedforward neural networks}.  Introduced in the landmark paper \cite{NeuralNetworksOG_1944} as a model for the cognition, a feedforward neural network is a function between Euclidean spaces which processes an input by iteratively applying affine maps with a fixed component-wise non-linear function $\sigma:\rr\rightarrow \rr$ called an \textit{activation function}.  

Feedforward neural networks (often abbreviated simply as neural networks) are often interpreted as \emph{black boxes} and frequently visualized pictorially as in Figure~\ref{fig_ffNN}, where each green circle represents an application of the activation function $\sigma:\rr\rightarrow \rr$ sometimes called a \emph{neuron}, the sequentially black lines called \emph{connections} represent univariate affine transformations of the preceding neuron's output until the network's outputs are produced.  Note that the analogy with neurocomputing arises when considering $\sigma(x)=I_{\{x\geq 0\}}$ where the state $1$ represents the ``activation/firing'' of a neuron upon receiving an input signal.
\begin{figure}[ht]
\centering
\centering
\includegraphics[width=.4\linewidth]{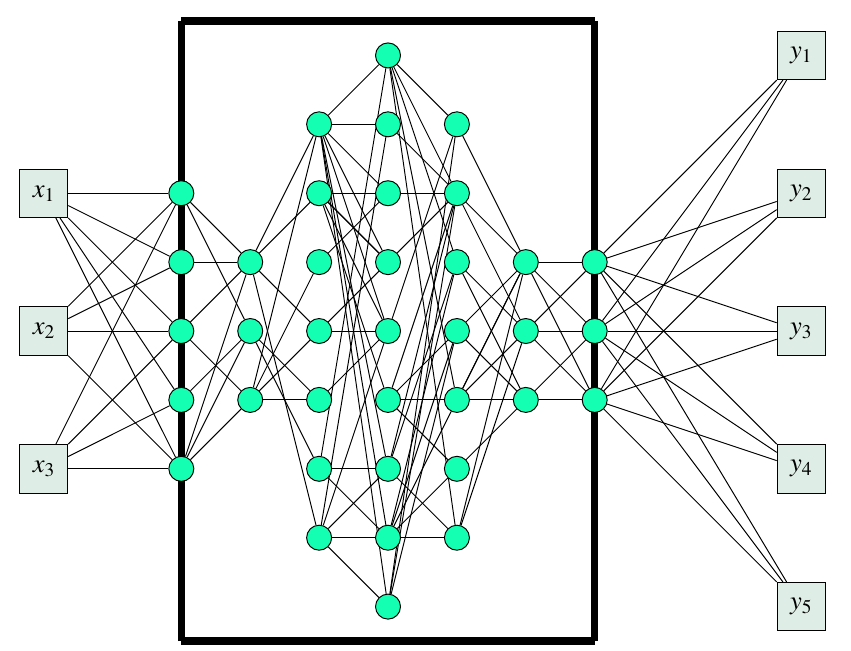}
\caption{A feedforward neural networks is a \textbf{black box} deep learning model which, given enough {\color{InputColor}{neurons}}, can be trained to approximate any function which continuously maps inputs in $\rr^3$ to outputs in $\rr^5$, uniformly on compact sets.}
\label{fig_ffNN}
\end{figure}

Recently, there has been significant theoretical work on activation functions which can adapt to their inputs, meaning that $\sigma$ is not a single function but rather a family of functions \citep{Swishramachandran2018searching,jiao2021deep,Yarotsky2020NEuripsPhaseDiagram,beknazaryan2021neural,yarotsky2021elementary}.  The advantage here is that an aptly chosen family of activation functions can ``collaborate'' to achieve superior approximation rates or help the original network exhibit desirable properties, such as being able to easily implement the identity and thus overcome overfitting \cite{hardt2016identity} and allow to represent the same function with fewer parameters \cite{FlorianHighDimensional2021}.  
Accordingly, our neural network models will depend on activation functions.
\begin{definition}[Trainable Activation Functions]
\label{defn_TrainableActivationFunction}
A \emph{trainable activation function} is any map $\sigma:\rr^2\times \rr\ni (\alpha,x)\mapsto \sigma_{\alpha}(x)\in \rr$.
\end{definition}
We will consider neural network models built using either of three different types of activation's functions, ``singular and trainable'', ``smooth non-polynomial and trainable'', and classical continuous activation functions  which are non-trainable and exhibit some basic regularity.  Our quantitative approximation results depend on the choice of activation function.  

Our interest in trainable activation functions, with discontinuities, stems from and adds to the emerging body of literature on the improved expressibility \cite{Yarotsky2021SuperPressiveICML,ShenYangZhang2021,park2021minimum,Shen2021ThreeLayersSuperExpressive,2022ZamanlooyPCNNs}.  Briefly, such activation functions allow us to build neural networks achieving exponential approximation rates.  Thus, they provide an upper-bound for the trainable \textit{singular} activation approximation rate which can be achieved by a practically deployed deep learning model.

In what follows, for any $x\in\R$, we denote $\lfloor x \rfloor \eqdef \max\{n\in\Z : n\leq x\}$.
\begin{definition}[Trainable Activation Function: Singular-ReLU Type]
\label{defn_TrainableActivation_Singular}
A trainable activation function $\sigma$ is of \emph{ReLU+Step type} if
\[
    \sigma_{\alpha}:
\rr
    \ni x
\mapsto 
        \alpha_1
            \max\{
                x
                    ,
                \alpha_2 x
                \}
            +
        (1-\alpha_1)
\lfloor x \rfloor\in\R.
\]
\end{definition}
For most $\alpha\in \rr^2$, the singular-ReLU type activation function $\sigma_{\alpha}$ is discontinuous.  We juxtapose the rates we derive with these singular types of trainable activation functions, with neural networks built using non-singular (i.e.\ smooth) and trainable activation functions; popular are the Swish activation function \citep{Swishramachandran2018searching}, the analytic and periodic activation functions used in SIRENs \cite{SIRENEs2000Neurips} and studied in \cite{SIEGEL2020313}, and GeLU \cite{hendrycks2016gaussian}, and the classical sigmoid activation function.  
By the quantitative results of \cite{KidgerLyons2020}, we know that polynomial activation functions do generate universal approximators in the deep regime, unlike the shallow regime of \cite{pinkus_1999}.  However, the quantitative results of \cite{kratsios2021universal} show that such activation functions are strictly less expressive than their non-polynomial counterparts; thus we do not consider such activation functions as a distinguished special case.  
\begin{definition}[Trainable Activation Function: Smooth-ReLU-Type]
\label{defn_TrainableActivation_Smooth}
A trainable activation function $\sigma$ is of \emph{smooth non-polynomial type} if there is a non-polynomial $\sigma^{\star}\in C^{\infty}(\rr)$, for which
\[
    \sigma_{\alpha}:
\rr
    \ni x
\mapsto 
        \alpha_1
            \max\{
                x
                    ,
                \alpha_2 x
                \}
            +
        (1-\alpha_1)
    \sigma^{\star}(x)\in\R
.
\]
\end{definition}
For completeness, we also consider the general case consisting of activation functions which do not fall into either of these two classes.  The most common example of such an activation function is the ReLU unit of \cite{fukushima1982neocognitron_ReLUOG}. We only require the regularity condition of \cite{KidgerLyons2020}.
\begin{definition}[{Classical Activation Function}]
\label{defn_TrainableActivation_ClassicalNonTrainable}
Let $\sigma^{\star}\in C(\rr)$ be non-affine and such that there is some $x\in \rr$ at which $\sigma$ is differentiable and has non-zero derivative.  Then, $\sigma$ is a \emph{classical regular activation function} if, for every $\alpha \in \rr^2$, $\sigma_{\alpha}=\sigma^{\star}$.
\end{definition}

In what follows we will often be applying our trainable activation functions component-wise. For positive integers $n,m$, we denote the set of $n\times m$ matrices by $\rr^{n\times m}$.  More precisely, we mean the following operation defined for any $N\in \nn_+$, $\bar{\alpha} \in \rr^{N \times 2}$ with $i^{th}$ row denoted as $\bar{\alpha}_i$, and $x\in \rr^{N}$, by
\[
\sigma_{\bar{\alpha}}\bullet x
 \eqdef 
\left(
\sigma_{\bar{\alpha}_i}(x_i)
\right)_{i=1}^N
.
\]
We may now formally define feedforward neural networks with trainable activation functions. 

\textbf{Feedforward Neural Networks with Trainable Activation Functions:}
Fix $J\in\nn_+$ and a multi-index $[d] \eqdef 
{(d_0,\dots,d_{J+1})}
$, with $d_0,\dots,d_{J+1}\in\nn_+$,
and let 
$P([d])
=
\sum_{j=0}^{J}d_{j+1}(d_{j}+3) 
-
2d_{J+1}
$.
We identify any vector $\theta \in \rr^{P([d])}$ with
\begin{equation}
\label{eq_hypernetwork_association}
\begin{aligned}
    \theta & \leftrightarrow
\left(
(A^{(j)},b^{(j)},\bar{\alpha}^{(j)})_{j=0}^{J-1}
,(A,c)
\right)
,
\\
(A^{(j)},b^{(j)},\bar{\alpha}^{(j)}) & \in 
\rr^{d_{j+1}\times d_{j}}\times \rr^{d_{j+1}} \times \rr^{d_{j+1}\times 2}
,\,{A \in \rr^{d_{J+1} \times\, d_J},\,}
c \in \rr^{d_{J+1} }
.
\end{aligned}
\end{equation}
With the identification in~\eqref{eq_hypernetwork_association}, and similarly to \cite{gribonval2021approximation}, we recursively define the representation function of a $[d]$-dimensional deep feedforward network by
\begin{equation}
    \begin{aligned}
   \rr^{P([d])}\times \rr^{d_0} \ni (\theta,x)
    & \mapsto 
    \hat{f}_{\theta}(x) 
         \eqdef  
    {A\,}
    x^{(J)}
   +c,
    \\
    x^{(j+1)} & \eqdef  
        \sigma_{\bar{\alpha}^{(j)}}\bullet( A^{(j)}
        x^{(j)}
            +
        b^{(j)})\qquad 
            \mbox{for }  
        j=0,\dots,J-1,
    \\
    x^{(0)} & \eqdef  x.
    \end{aligned}
    \label{eq_definition_ffNNrepresentation_function}
\end{equation}

We denote by $\NN[[d]]$ the family of $[d]$-dimensional deep feedforward networks $\{\hat{f}_{\theta}\}_{\theta \in \rr^{P([d])}}$ described by \eqref{eq_definition_ffNNrepresentation_function}.  The subset of $\NN[[d]]$ consisting of networks $\hat{f}_{\theta}$  with each $\bar\alpha^{(j)}_i=(1,0)$  in~\eqref{eq_definition_ffNNrepresentation_function} is denoted by $\NN[[d]][ReLU]$ and consists of the familiar deep ReLU networks.  
The value $\max_{0\le j\le J+1}\,d_{j}$ is $\hat{f}$'s \textit{width}, $J$ is called $\hat{f}$'s \textit{depth}\footnote{Some authors, e.g.\ \cite{BartlettHarveyLiaw_JMLR_2019_VCPseudodimension_DNNs}, refer to $J$ as the number of \textit{hidden layers} in $\hat{f}$.}, and $J+1$ is the number of $\hat{f}$'s \textit{layers}.

\subsubsection{Metric Capacity}
\label{s_Preliminaries__ss_CompactSetsGeometry}
Many quantitative universal approximation results for classical feedforward networks  exhibit the same approximation rate for any compact subset of the input space of the same \textit{diameter}, defined for any $K\subseteq \xxx$ by $\operatorname{diam}(K) \eqdef \sup_{x,y\in K}\, d(x,y)$; see \cite{Yarotsky2020NEuripsPhaseDiagram}.  However, one would expect that for ``simpler compacts'' approximation requires less complicated networks.  Indeed, this is a feature of our main quantitative approximation result, which also relates the approximation quality to the complexity of the compact set on which the approximation holds.  This complexity  is expressed in two ways, first in terms of the regularity of each path $\xb \in \xxx^{\zz}$, and second in terms of the size of the compact set in which each $x_t$ is to be approximated.

For the latter, we use a refinement of the usual notion of dimension, for smooth manifolds, which can distinguish between compact subsets of different intermediate ``fractal dimensions''.  More specifically, we consider the \textit{covering dimension} of a metric spaces, in the sense of \citep[Section 10]{HeinonenClassicBook_AnalysisinMet2001}, which we control by the metric space's \textit{metric capacity}, as defined in \citep[Definition 1.6]{BRUE2021LipExtRP}.  Following \cite{BRUE2021LipExtRP}, we define the metric capacity of a \textit{subset} $K\subseteq \rr^d$ as the map $\kappa_K:(0,1]\rightarrow \nn\cup \{\infty\}$ which sends any $\delta>0$ to
\begin{equation}\label{eq.kk}
\kappa_K(\delta)
     \eqdef  
            \sup
        \left\{
            k\in \nn 
                :
            \exists x_0,\dots,x_k,\ \exists r>0\ \text{s.t.}\
            \sqcup_{i=1}^k \Ball_{K}(x_i,\delta r)
                \subset
            \Ball_{K}(x_0,r)
        \right\},
\end{equation}
where, as usual, for any $x \in K$ and any $r>0$, we set  $\Ball_K(x,r) \eqdef  \{
z\in K:\, \|z-x\|<r
\}$, $\| \cdot \|$ the Euclidean norm, and $\sqcup$ denotes the union of disjoint sets.  We make things concrete with the following two examples.    

 In what follows, we say that a function $f:\nn\rightarrow \rr$ \textit{asymptotically grows at a linear rate}, written $f(n)\in \Theta(n)$, if there are constants $c,C>0$ and $N\in \nn$ such that $cn\leq f(n)\leq Cn$, for every $n\geq N$.
\begin{example}[Euclidean Spaces]
\label{ex_Euclidean}%
For the $d$-dimensional Euclidean space, the discussions on \citep[page 82]{HeinonenClassicBook_AnalysisinMet2001} and \citep[Proposition 1.7]{BRUE2021LipExtRP} imply that
$\log_2(\kappa_{{\rr^d}}(5^{-1}))
\in \Theta(d)
$. 
\end{example}
Similarly, when $K$ is a Riemannian submanifold of $\rr^n$, the metric capacity is proportional to the intrinsic dimension of $K$, as defined in differential geometry.  Thus, our approximation rates automatically adapt to paths which lie in low-dimensional submanifolds of our input space.
\begin{example}[Riemannian Submanifolds]%
\label{prop_compact_Riemannian_example}
Together, \citep[Proposition 2.7]{UrsLandSchlichermaier2005AssouadDimensionandLipschitzExtensions2005} and \citep[Proposition 1.7]{BRUE2021LipExtRP} imply that, if $K$ is a $d$-dimensional compact Riemannian submanifold of $\rr^n$, then
$
\log_2(\kappa_K(5^{-1})) 
\in \Theta(d)
.
$
\end{example}
Note that any subset $K\subseteq \rr^d$ has finite metric capacity%
\footnote{%
By \citep[Theorem 12.1]{HeinonenClassicBook_AnalysisinMet2001} every subset of $\rr^d$ has the doubling property (see Section~\ref{s_Proofs} for the definition and \citep[Section 10.13]{HeinonenClassicBook_AnalysisinMet2001} for details), and following the discussion on \citep[page 3]{BRUE2021LipExtRP} a metric space is doubling if and only if its metric capacity is finite for all $\delta \in (0,1]$. 
}%
, meaning that $\kappa_K(t)<\infty$ for all $t\in (0,1]$.
\section{Static Case: Universal Approximation into QAS Spaces}
\label{s_Static_Case}
We begin by presenting our results in the static case, where we find that a continuous function between suitable metric spaces can always be approximated by our geometric transformer model (Theorem~\ref{thrm_main_StaticCase}).  
Our driving example is the approximation, in the (adapted) Wasserstein sense, of the transition kernels of higher order Markov processes; see Examples~\ref{ex_adapted_wasserstein},~\ref{ex_SDEs},~\ref{ex_SDEs_are_FC_adaptedMaps}, and~\ref{ex_SDEs_AW_simple}.
\subsection{The Metric Geometry of {\texorpdfstring{$\yyy$}{Y}}}
\label{s_Static_Case__ss_YMetGeo}
One cannot expect to have a universal approximation result for any metric space due to topological obstructions; see 
\citep[Theorem 7]{kratsiosuniversal2021}.  However, there is a rich class of output metric spaces $(\yyy,d_{\yyy})$ (abbreviated by $\yyy$), encompassing most spaces encountered in stochastic analysis, and specified by exactly two simple conditions for which the approximation results presented in this paper apply.  

The first condition on $\yyy$, generalizes the existence of a \textit{geodesic bicombing}, as studied by \cite{lang2001bilipschitz,UrsDescombes,basso2018fixed,miesch2018cartan,basso2019conical}, and is analogous in spirit to the idea of a \textit{simplicial topological space} from homotopy theory \citep[Chapter 83.2]{stacksproject} and from fuzzy set theory \citep{BarrClassicFuzzySetTopos1986}, and the peaked partitions of unity in metric space theory \cite{semadeni2006schauder}.  Essentially, we require that any \emph{inscribed simplex} in $\yyy$ formed by joining any number of points in $\yyy$ (illustrated by Figure~\ref{fig_generalized_simplices_Y}) is nothing else but a deformation of the standard Euclidean simplex with the same number of vertices (as illustrated by Figure~\ref{fig_simplex}).  In particular, if $\yyy$ is a geodesic space, then the edges of the inscribed simplex in Figure~\ref{fig_generalized_simplices_Y} can be geodesics.  However, this is by no means a requirement.   
\begin{figure}[H]
\centering
     \begin{subfigure}[b]{0.375\textwidth}
  \centering
    \scalebox{0.5}{\includegraphics[width=1.25\linewidth]{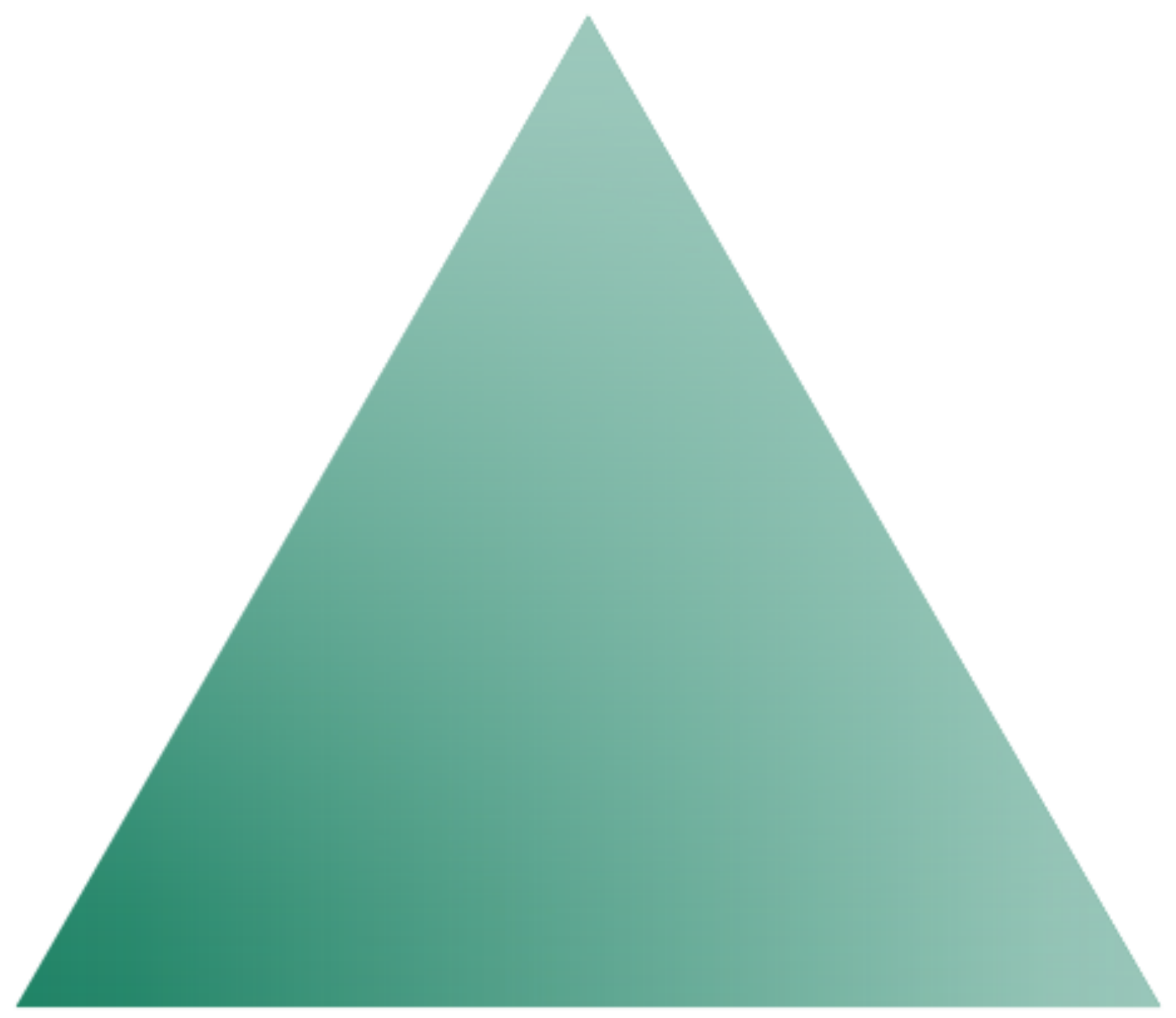}}
  \caption{{Inputs: $N$-Simplex}}
  \label{fig_simplex}
    \end{subfigure}
    \hfill
     \begin{subfigure}[b]{0.6\textwidth}
    \centering
    \scalebox{0.5}{\includegraphics[width=1.25\linewidth]{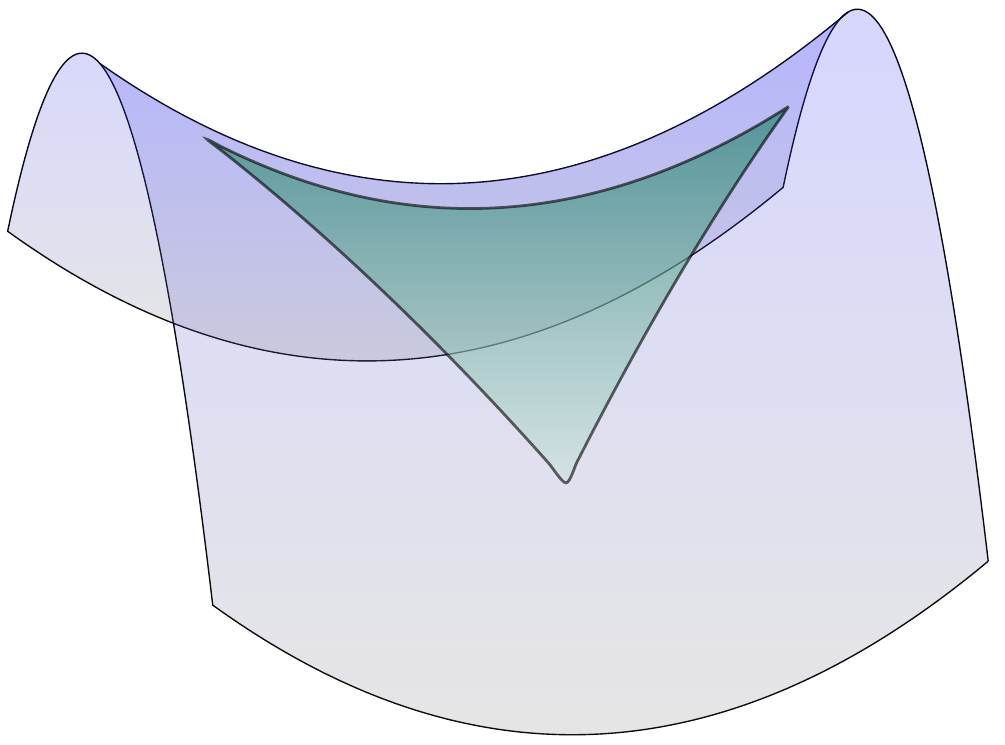}}
    \caption{Outputs: Inscribed $N$-simplex in $\yyy$.}
  \label{fig_generalized_simplices_Y}
  \end{subfigure}
\caption{The mixing function $\eta$ inscribes any $N$-simplex into the approximately simplicial space $\yyy$, sending any ``weight'' $w$ in the $N$-simplex $\Delta_N$ to the point $\eta(w,\mathcal{Y})$ in $\yyy$ whose vertices are the set of $N$ points $\mathcal{Y}$ in $\yyy$.}
\label{fig_inscriptionOfSimplexIntoQAS_Space}
\end{figure}
We denote the $N$-simplex by $\Delta_N \eqdef \left\{w\in [0,1]^N:\,\sum_{n=1}^N w_n=1 \right\}$, and we set $\widehat\yyy \eqdef \bigcup_{N\in\mathbb{N}_+}\left(\Delta_N\times \yyy^N\right)$.
\begin{definition}[Approximately Simplicial]
\label{ass_approximately_simplicial}
A metric space $(\yyy,d_{\yyy})$ is said to be \emph{approximately simplicial} if there is
a function $\eta:\widehat\yyy\to\yyy$, called a mixing function, and constants $C_\eta\geq 1$ and $p\in\mathbb{N}_+$, such that
for every $N\in \mathbb{N}_+$, $w=(w_1,\dots,w_N)\in\Delta_N$ and $\mathcal{Y}=(y_1,\dots,y_N) \in \yyy^N$, one has
\[
    d_\yyy \left( 
    \eta\left( w, \mathcal{Y} \right), y_i \right) 
    \le 
C_\eta
\left(
    \sum_{j = 1}^N d_\yyy(y_i,y_j)^p w_j\right)^{1 / p}.
    \]
Note that, in particular, the function $\eta$ satisfies $\eta(e_i,\mathcal Y)=y_i$, for all $i=1,\dots,N$, where $\{e_i\}_{i=1}^N$ denotes the standard basis of $\rr^N$.
\end{definition}

Only assuming that $\yyy$ is approximately simplicial is not enough, since $\yyy$ can be ``too large'' and universal approximation of any continuous function to $\yyy$ by a ``reasonable'' model may be impossible.  For example, if $\yyy$ is non-separable, such as the family of c\`{a}dl\`{a}g paths maps from $[0,1]$ to $\rr$ with the uniform topology, then no sequence of models depending continuously on a finite number of real parameters can even approximate constant functions therein.

Our second condition on the output space's geometry requires that we can parameterize a dense subset of $\yyy$.
Moreover, we organized these parameterized subsets into a nested sequence of $\epsilon$-nets which can be implemented by any idealized computer, capable of processing finite-dimensional real-vectors.  In order for these $\epsilon$-nets to be implementable, we require that each $\epsilon$-net is explicitly parameterized by some Euclidean space. 
The sequence of the parameterizations is called a \textit{quantization of $\yyy$}; with the name drawing from the fact that our notion of quantization generalizes the quantization of probability measures (see \cite{graf2007foundations,Pages2020_Quantization}).

\begin{definition}[Quantizability]
\label{def_disc_modulus}
Let $(\yyy,d_{\yyy})$ be a metric space.  If there is a sequence $Q \eqdef  (Q_q)_{q\in \nn_+}$ of functions $Q_q:\rr^{D_q}\rightarrow \yyy$, with $D_q\in \nn_+$, such that
\begin{enumerate}
    \item
    for any $q\in \nn_+$ and $z\in \rr^{D_q}$, there exists $\tilde{z}\in \rr^{D_{q+1}}$ such that $Q_q(z)=Q_{q+1}(\tilde{z})$;
    \item for every $y\in \yyy$ and $\epsilon>0$, there exist $q\in \nn_+$ and $z\in \rr^{D_q}$ satisfying
\[
d_{\yyy}\left(
    y
        ,
    Q_q(z)
\right)
    <
\epsilon
,
\]
\end{enumerate}
then $(\yyy,d_{\yyy})$ is said to be \emph{quantizable}.  Furthermore, $Q$ is said to quantize $(\yyy,d_{\yyy})$.  
\end{definition}

We remark that it is advantageous not to require that the vectorial parameterization of these $\epsilon$-nets be continuous, especially when $\yyy$ is a discrete metric space
.  However, under the additional assumption that our parameterizations of these $\epsilon$-nets by Euclidean inputs is continuous, then the quantizability of $\yyy$ implies separability of its topology. 
Therefore, quantizability can be interpreted as a \textit{quantitative metric analogue} of separability, which itself is an otherwise \textit{qualitative topological property}.  

Since quantizability is a quantitative property of $\yyy$, then we measure the complexity required to parameterize an $\epsilon$-net collection in any compact subset of $\yyy$ by simply counting the minimum dimension of a Euclidean space required to parameterize that $\epsilon$-net.  The map sending an $\epsilon$ and a compact subset of $\yyy$ to this minimum dimension is called a ``quantization modulus''.  The role of the quantization modulus is therefore analogous to other moduli appearing in quantitative metric theories; examples include, moduli of continuity which quantifying metric distortion; moduli of smoothness in \citep{BesovExtension2016GeneralMetricSpacesHeikkinenIhnatsyevaTuominen,DraganovIvanov2014ModulusofSmoothness} quantifying higher-order distortions, or the modulus of padded decomposability quantifying the complexity of a random partition \citep{krauthgamer2004measuredNaor,krauthgamer2005measuredNaor}.  

\begin{definition}[Modulus of Quantizability]
\label{def_disc_modulus_quantification}
In the notation of Definition~\ref{def_disc_modulus}.  Let $(\yyy,d_{\yyy})$ be a metric space quantized by $Q = (Q_q)_{q\in \nn_+}$.   For any given compact $\kkk\subseteq \yyy$, the \emph{quantization modulus} of $Q$ on $\kkk$ is the map $\mathcal{Q}_{\kkk}:\RR_+\to\mathbb{N}$ that sends any $\epsilon>0$ to $\mathcal{Q}_{\kkk}(\epsilon) \eqdef  D_{q_{\mathcal{K},\epsilon}} $ where
\[
    q_{\mathcal{K},\epsilon}
    \eqdef  
    \inf\,
    \left\{
        q  \in \mathbb{N}_+
    :
    \,
    \forall y\in \kkk,\,
    \exists z\in \rr^{D_q}\,\, \text{s.t.}\
    d_{\yyy}(
    y
        ,
   Q_q(z)
    )<\epsilon
    \right\}.
\]
The map $Q$ is called \emph{regular} if $\mathcal{Q}_{\kkk}(\epsilon)$ is finite for all $\epsilon>0$ and all compact subsets $\kkk\subseteq \yyy$.
\end{definition}
As we will see in Section~\ref{s_Examples}, most metric spaces encountered when working with stochastic processes are both quantizable and approximately simplicial.  Since we only consider output spaces with precisely these two properties, we name metric spaces carrying this additional structure.  
\begin{definition}[QAS Space]
\label{defn_QAS_space}
A metric space $(\yyy,d_{\yyy})$ together with a function $\eta$ satisfying Definition~\ref{ass_approximately_simplicial} and a sequence $Q = (Q_q)_{q\in \nn_+}$ satisfying Definition~\ref{def_disc_modulus} will be called a \emph{Quantizable and Approximately Simplicial space} (QAS space).  We denote QAS spaces by the tuple $(\yyy,d_{\yyy},\eta,Q)$.
\end{definition}
The additional structure carried by QAS spaces allows us to approximately parameterized the inscribed simplices (as in Figure~\ref{fig_generalized_simplices_Y}) using Euclidean data projected onto the corresponding standard simplex (as in Figure~\ref{fig_simplex}).  This is because every QAS space naturally defines a neural network layer which approximately encodes $\yyy$'s geometry, by mixing the attention payed to $N$ particles in $\yyy$ (represented by the vertices of the inscribed simplex in Figure~\ref{fig_generalized_simplices_Y}) via  the mixing function $\eta$.  

\begin{definition}[Geometric Attention Mechanism]
\label{defn_attention}
Let $(\yyy,d_{\yyy},\eta,Q)$ be a QAS space.  Then the \emph{geometric attention mechanism} on $\yyy$  is the family of functions $(\operatorname{attention}_{N,q})_{N,q\in \nn_+}$ defined by
\[
\operatorname{attention}_{N,q}: 
\rr^N\times \rr^{N\times D_q}\ni (u,(z_n)_{n=1}^N)
    \rightarrow 
\eta\left(
    \Pi_{\Delta_N}(u)
        ,
    (Q_q(z_n))_{n=1}^N
\right)\in\yyy
,
\]
where $\Pi_{\Delta_N}$ is the projection of  $\rr^N$ onto the $N$-simplex $\Delta_N$. 
\end{definition}
The Wasserstein space $(\mathcal{P}_1(\rr^2),\mathcal{W}_1)$ can be made into a QAS space (see Examples~\ref{ex_wasserstein} and~\ref{ex_Wasserstein_barycenter} below for details); in which case, a variant of the probabilistic attention mechanism of \cite{AB_2021} is a special case of our geometric attention mechanism; illustrated visually in Figure~\ref{fig_KratsiosZamanlooyetal_example}%
\footnote{For the relationship between probabilistic attention and the (classical) attention mechanism of \cite{vaswani2017attention} see the discussion following \citep[Equation 4]{AB_2021}.}.  For instance, given
three points $z_1,z_2,z_3$ in $\rr^2$ which we wish to charge with mass, Figure~\ref{fig_KratsiosZamanlooyetal_example} illustrates the map 
\[
    \operatorname{attention}_{3,1}\left(
      u,(z_1,z_2,z_3)
    \right)
     \eqdef 
        \sum_{n=1}^3
w_{n} \delta_{z_n}
,
\]
which sends any vector $u\in \rr^3$ to the nearest normalized weight $w \eqdef \Pi_{\Delta_3}(u)$ in the $3$-simplex $\Delta_{3}$, as visualized by the $\bullet$.
Then it distributes the weights in $\bullet$ amongst the three points $z_1,z_2,z_3$ in $\rr^2$, with the amount of mass illustrated by the radius of the their respective violet bubbles.
\begin{figure}[H]%
\centering
\includegraphics[width=0.3\linewidth]{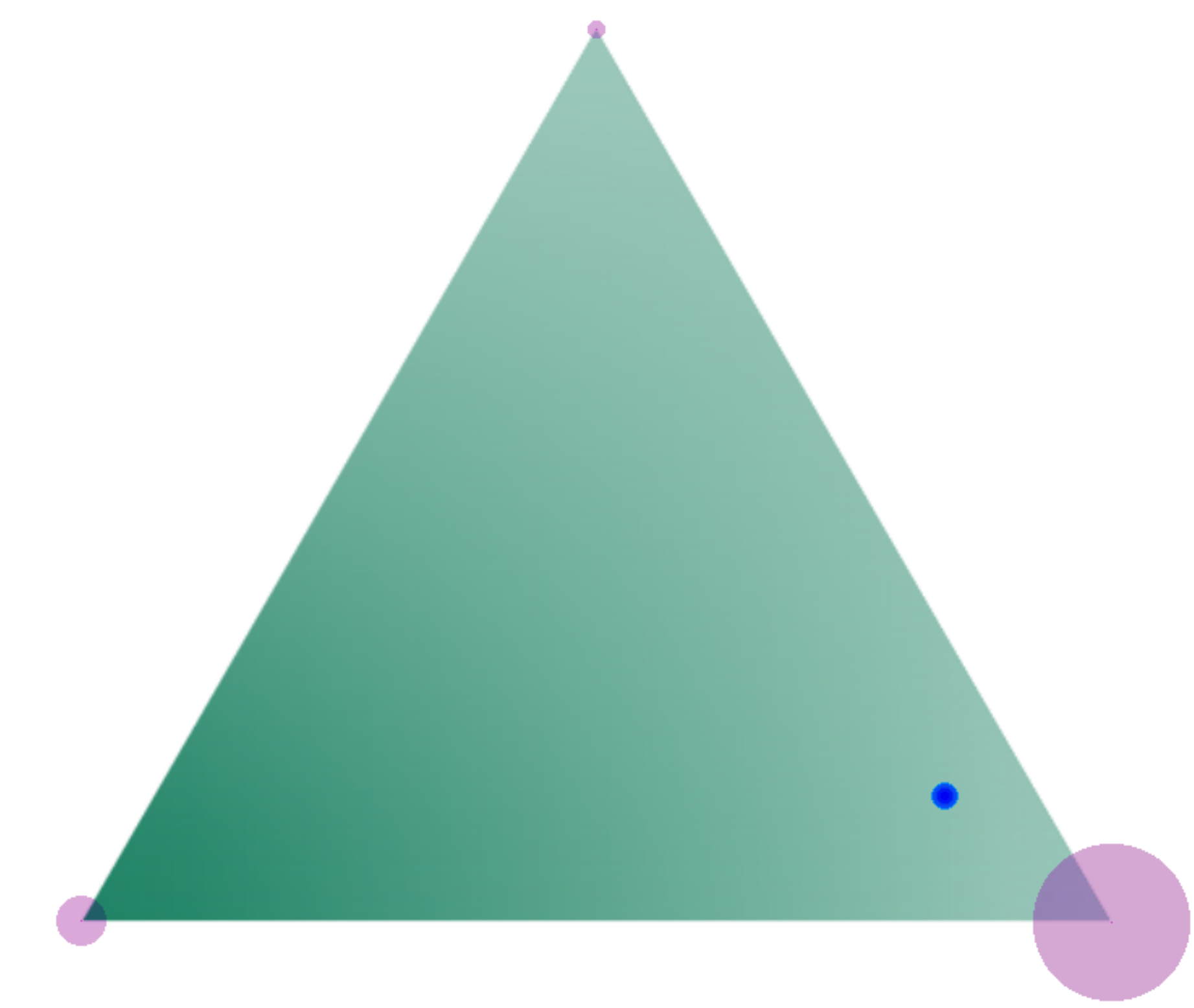}
\caption{Simplex inscribed in the Wasserstein Space $\mathcal{P}_1(\rr^2)$.  Here, the {\color{MidnightBlue}{blue dot}} represents the normalized weight {\color{MidnightBlue}{$w$}} in the {\color{AnnieGreen}{$3$-simplex $\Delta_3$}}, the vertices of the simplex represent the points ${\color{violet}{z_1,z_2,z_3}}\in \rr^2$ to be charged with mass, and the size of the {\color{violet}{violet bubble}} around each vertex represents the measure $\sum_{n=1}^3 {\color{MidnightBlue}{w_n}}{\color{violet}{\delta_{z_n}}}$ charging the points $z_1,z_2,z_3$.  }
\label{fig_KratsiosZamanlooyetal_example}
\end{figure}
More generally, given $N$ probability measures $y_1,\dots,y_N \in \mathcal{P}_1(\rr^d)$, one can find empirical probability measures
$\hat{y}_n \eqdef \frac1{D_q}\sum_{q=1}^{D_q} \delta_{z_{n,q}}$ with $\{z_{n,q}\}_{n,q=1}^{N,D_q}$ in $\rr^d$ which quantize $\{y_n\}_{n=1}^N$, i.e.\ $\mathcal{W}_1(y_n,\hat{y}_n) \approx 0$; see \citep[Corollary 3]{Chevallier_2018_UnifDecompositionProbMEasures_JAP}.  In this case, the geometric attention becomes
\begin{equation}
    \operatorname{attention}_{N,q}(u,z)
     \eqdef         \sum_{n=1}^N
        \frac{w_{n}}{D_q}
        \sum_{q=1}^{D_q}
                \delta_{z_{n,q}}
    =
         \sum_{n=1}^N
        w_{n}\hat{y}_n
    ,
\label{eq_attention_W1}
\end{equation}
where $w \eqdef \Pi_{\Delta_N}(u)$ and $z
             \eqdef 
        (z_{n,q})_{n,q=1}^{N,D_q}$.  
In the next example, we extend this framework to the case of distributions on path spaces, and replace the Wasserstein distance with its \textit{adapted counterpart}  \cite{Ruschendor_1985_AdaptedOG}, which respects the flow of information.
\begin{example}[Adapted Wasserstein Space with Convex Combinations]
\label{ex_adapted_wasserstein}
We consider $(\yyy,d_{\yyy}) \eqdef (\mathcal{P}_p([0,1]^{dT}),\mathcal{AW}_p)$.   In difference to the Wasserstein distance case, empirical distributions are not consistent estimators with respect to the adapted Wasserstein distance.
Thus, to quantize $(\mathcal{P}_p([0,1]^{dT}),\mathcal{AW}_p)$, we can not take functions as in \eqref{eq_attention_W1}.
Instead, \cite{BBBW2020estimating} suggests the use of an \emph{adapted empirical distribution}. Let $r=(T+1)^{-1}$ for $d=1$, and $r=(dT)^{-1}$ for $d\geq 2$.
For all $q\geq 1$, partition the cube $[0,1]^d$ into the disjoint union of $q^{rd}$ cubes with edges of length $q^{-r}$, and let $\varphi^q:[0,1]^d\to[0,1]^d$ map each small cube to its center.
By \cite[Theorem~1.3]{BBBW2020estimating}, the following family $(Q_q)_{q\in \nn_+}$ of functions quantizes $(\mathcal{P}_p([0,1]^{dT}),\mathcal{AW}_p)$:
\[
Q_q: \rr^{{dT}\times q}\ni 
    z=(z^1,\dots,z^q)
        \mapsto 
    \frac1{q}
    \sum_{s=1}^{q}
        \,\delta_{(\varphi^q(z^s_1),\dots,\varphi^q(z^s_T))}
\in \mathcal{P}_p([0,1]^{dT}).
\]
Thus,
$(\mathcal{P}_p([0,1]^{dT}),\mathcal{AW}_p,\eta,Q)$ 
is a QAS space associated to the geometric attention mechanism
\[
\operatorname{attention}_{N,q}\left(
        u
        ,
        (z^n)_{n=1}^N
        \right)
     \eqdef 
        \sum_{n=1}^N
        \frac{\Pi_{\Delta_N}(u)_n}
{q}
        \sum_{s=1}^{q}
                \,\delta_{(\varphi^q((z^n)^s_1),\dots,\varphi^q((z^n)^s_T))}
.
\]
\end{example}
Below, we introduce our static models, which we call \textit{geometric transformers}.  These work like the standard transformers of~\cite{vaswani2017attention}, by decomposing the approximation of a suitable function $f:\rr^d\rightarrow \yyy$ into two steps.  In the first step, a deep feedforward network learns how to encode any input $x\in \xxx\subseteq \rr^d$ into a vector in a deep feature space $\rr^N$.  
This step is illustrated in Figure~\ref{fig_natural_transformer} by the feedforward network sending the input $(x_1,x_2,x_3)$ to the deep features $(u_1,\dots,u_5)$.  In the second step, illustrated by the rectangular purple node in Figure~\ref{fig_natural_transformer},
the deep features in $\rr^N$ are decoded into $\yyy$-valued predictions (illustrated in Figure~\ref{fig_generalized_simplices_Y}).  

    The decoding step is implemented by our geometric reinterpretation of the attention mechanism of \cite{bahdanau2014neural}, designed for Natural Language Processing tasks, and of its probabilistic counterpart of \cite{kratsios2021universal, AB_2021}.  The critical difference between our \textit{geometric attention} and the above attention mechanisms is that the geometric attention is customized for $\yyy$'s geometry, whereas the others are suited to their respective spaces.

\begin{figure}[ht]
%
\centering
\includegraphics[width =1\linewidth]{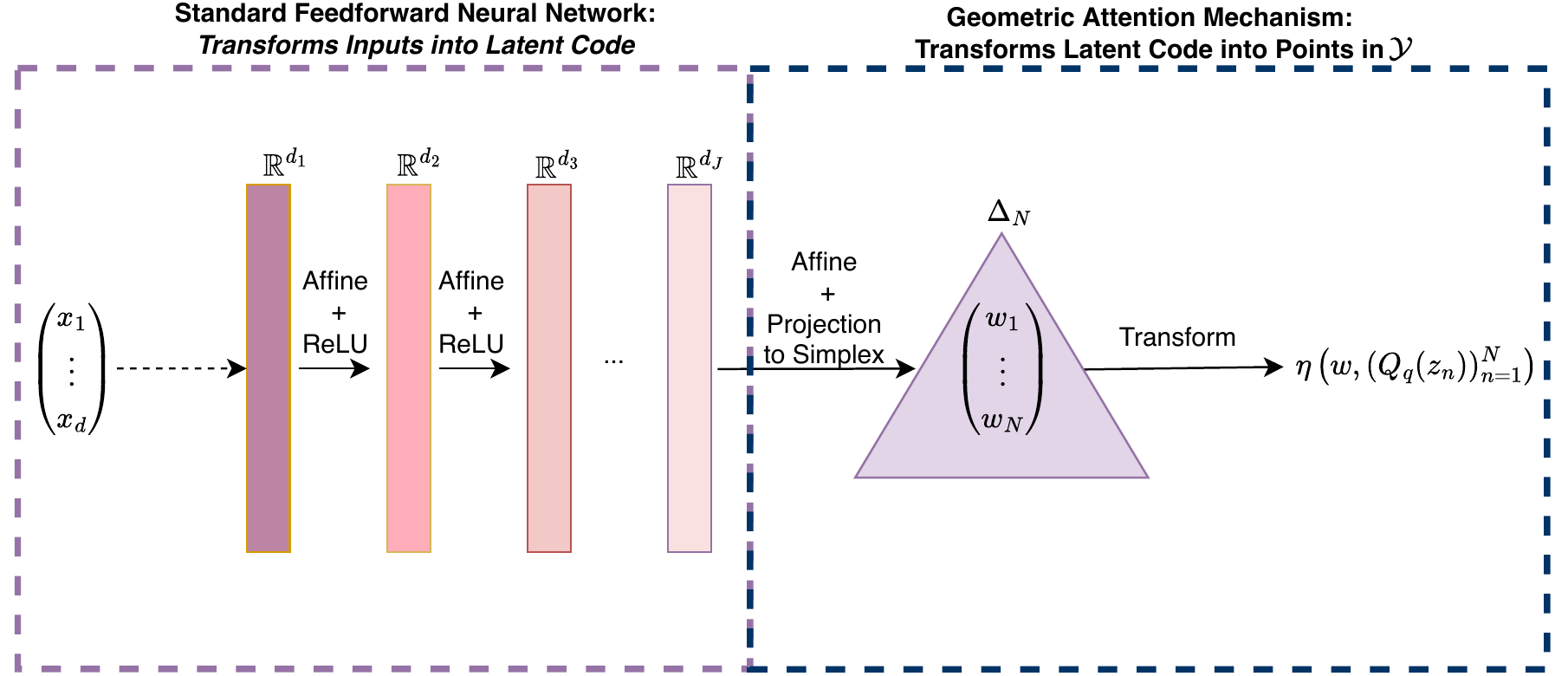}
\caption{An illustration of the geometric transformer architecture {(see Definition~\ref{def_natural_transformer})}.  Euclidean inputs are first mapped to latent vectorial outputs by a feedforward neural network, these latent vectors are then transformed into predictions on $\mathcal{Y}$ by the geometric attention mechanism.}
\label{fig_natural_transformer}
\end{figure}
\begin{definition}[Geometric Transformer]
\label{def_natural_transformer}
Let $(\yyy,d_{\yyy},\eta,Q)$ be a QAS space and $d\in \nn$.  Fix a trainable activation function $\sigma$, constants $N,q \in \nn_+$, and a multi-index $[d]$ with $d_0=d$ and $d_J=N$.  A \emph{geometric transformer} (GT) from $\rr^d$ to $\yyy$ is a function $\hat{\rho}:\rr^d\rightarrow \yyy$ with representation
\begin{equation}
    \hat{\rho} 
    = 
\operatorname{attention}_{N,q}\left(
    \hat{f}_{\theta}(\cdot)
        ,
    Y
\right),
\label{eq___def_natural_transformer}
\end{equation}
where $\hat{f}_\theta\in \NN$ and $Y\in \rr^{N\times D_q}$.  The set $\GT$ of geometric transformers from $\rr^d$ to $\yyy$ with activation function $\sigma$ of complexity $([d],N,q)$ consists of all $\hat\rho$ with representation~\eqref{eq___def_natural_transformer}.  
\end{definition}
In what follows, if already clear from the context, we will omit the specification ``from $\rr^d$ to $\yyy$ with activation function $\sigma$ of complexity $([d],N,q)$'' when talking of a geometric transformer.

\subsection{Static Case -- Universal Approximation into QAS Spaces}
\label{s_Static_Case__ss_Main_Results}
We now present our first \textit{universal approximation theorem} (Theorem~\ref{thrm_main_StaticCase}).  This result is both a refinement and a generalization of the classical universal approximation theorem, where we use geometric transformers in place of feedforward networks.  

Following \cite{KovachkiLanthalierMishra_FouerierNeuralOperator_OG_2021JMLR, 2021arXiv210606682L}, we decompose the total approximation error incurred by approximating a target function $f:\xxx\rightarrow \yyy$ by a geometric transformer into three components: $\epsilon_A,\epsilon_Q,\epsilon_{NN}>0$.  
The first term $\epsilon_A$ represents the \textit{intrinsic error} incurred by approximately representing $f$ by a map defined on a specific simplex inscribed in $\yyy$ (as in Figure~\ref{fig_generalized_simplices_Y}).
Since this inscribed simplex' endpoints may not be representable by Euclidean features, we need to perturb them using the quantization maps $Q_q$, and the \textit{quantization error} $\epsilon_Q$ captures this. Then, $\epsilon_{NN}$ represents the \textit{encoding error} originated from the approximation of $f$ by a feedforward network, which approximately encodes the elements of $\xxx$ into deep features which are then passed to the geometric attention layer to be decoded into predictions on $\yyy$.  
For $0<\alpha \le 1$, we denote the set of $\alpha$-H\"{o}lder functions from $\xxx$ to $\yyy$ by $C^{\alpha}(\xxx,\yyy)$.
\begin{theorem}[Metric Transformers are Universal Approximators of QAS-Space-Valued Functions]
\label{thrm_main_StaticCase}
Let $0<\alpha \leq 1$, $\xxx\subseteq \rr^d$ be compact, $(\yyy,d_{\yyy},\eta,Q)$ be a QAS space, and let $\sigma$ be a trainable activation function as in Definitions~\ref{defn_TrainableActivation_Singular},~\ref{defn_TrainableActivation_Smooth}, or~\ref{defn_TrainableActivation_ClassicalNonTrainable}. 
Then, for every $f\in C^{\alpha}(\xxx,\yyy)$, every ``intrinsic error'' $\epsilon_A>0$, ``quantization error'' $\epsilon_Q>0$, and ``encoding error'' $\epsilon_{\mathcal{NN}}>0$, there exist positive integers $N,q\in \nn_+$, a matrix $Y\in \rr^{N\times D_q}$, and a 
geometric transformer $\hat{\rho} \in \GT$ satisfying
\[
\sup_{x\in \xxx}
    \,
d_{\yyy}\left(
    f(x)
        ,
       \hat{\rho}(x)
\right)
    \leq 
\epsilon_A + \epsilon_Q + \epsilon_{\mathcal{NN}},
\]
with $\hat{\rho}$ as in \eqref{eq___def_natural_transformer}, and
where the number of parameters determining $\hat{\rho}$ are recorded in Table~\ref{tab_spacecomplexity_universalmetrictransformer}. 
\end{theorem}

Table~\ref{tab_extrapolation_function} reports the ``space complexity'' of the geometric transformer built in Theorem~\ref{thrm_main_StaticCase}, with each column describing a different aspect of the universal approximation capabilities of our neural network model.  Moving from right to left, the third column in Table~\ref{tab_spacecomplexity_universalmetrictransformer} confirms that ``narrow'' geometric transformers with classical (non-trainable) activation functions, such as ReLU, sigmoid, or the swish activation function of \cite{Swishramachandran2018searching}, have the capacity to approximate any function in $C^{\alpha}(\xxx,\yyy)$.  

The second column gives precise quantitative approximation rates for geometric transformers with trainable activation functions and smooth-ReLU-type.  The deep feedforward used in building these models thus has the capacity to precisely implement the identity and can be trained using stochastic gradient descent methods, of course, depending on $\yyy$.  For details see \cite{backhoff2022stochastic} for the case where $\yyy$ is the Wasserstein space, \cite{gallego2015convergenceBanachUniformlyConvex} when $\yyy$ is a suitable Banach space, and \cite{SGD_Riemannian_Bonnabel_2013,2018_MJ_SGD_Riemannian,Lucci_2021_Momentum_Riemannian} for the case where $\yyy$ is a complete Riemannian manifold with bounded sectional curvature.  

The first column of Table~\ref{tab_extrapolation_function} covers the efficiency of transformer networks if one relaxes the need to have continuous activation functions.  Though such neural network models cannot be trained using conventional stochastic gradient descent-type algorithms, they can be implemented using randomized approaches such as \textit{random neural network} (or \textit{extreme learning machine}); see \cite{louart2018random,Lukas2020,2022ZamanlooyPCNNs}. Not only is this column most pertinent to transformer networks trained with randomized methods, but it also highlights the potential of transformer approaches to geometric deep learning, applications to stochastic analysis, and mathematical finance.   

In what follows, for any $x\in \rr$, we denote $\lceil x \rceil \eqdef \min\{n\in\Z : n\geq x\}$ and $x_{++} \eqdef \max\{1,x\}$. 
\begin{table}[ht]
\centering
\caption{Upper bounds on the model complexity of the geometric transformer network\, $\hat{\rho} \eqdef \operatorname{attention}_{N,q}\left(\hat{f}_{\theta}(\cdot),Y\right)$ of Theorem~\ref{thrm_main_StaticCase}.
$N$ is the same for all activation functions.  
}
\label{tab_spacecomplexity_universalmetrictransformer}
\ra{1.3}
\begin{adjustbox}{width=\columnwidth,center}
\begin{tabular}{@{}llll@{}}
\toprule
\textbf{Activation} ($\sigma$) & \textbf{Singular}~(\ref{defn_TrainableActivation_Singular}) & \textbf{Smooth}~(\ref{defn_TrainableActivation_Smooth}) & \textbf{Classical}~(\ref{defn_TrainableActivation_ClassicalNonTrainable})
\\
\midrule
Depth ($J$) & 
$
    (N-1)(1+(2^6n D + 3))
    $
    &
    $
    \mathcal{O}\left(
    (N-1)\left(1+
              \tilde{\epsilon}^{-2n/\alpha}
              L_f^{2n/\alpha}
              (1+n/4)^{2n/\alpha}
            \right)
    \right)
    $
    &
    -
\\
\arrayrulecolor{lightgray}\hline
Width  
& 
    $
     n(N-2) 
        + 
    \max\{n,5W+13\}
    $ 
    &
    $
        n(N-1)+3
    $
    &
    n + N + 1
\\
\arrayrulecolor{lightgray}\hline
\# Parameters P([d])
& 
$
\left(
        \frac{11}{\, 4\,}n^2(N-1)^2 - 1
    \right) 
    (N-1)
    \max\{n+3,5W+16\}^2
    (2^6n D 
     +4
    )
$
&
$
    \mathcal{O}\left(
    \left( \frac{11}{ \, 4 \, } n^2(N-1)^2 - 1\right) 
    (N-1)
    (n + 6)^2
        \big(
            \tilde{\epsilon}^{-4n/\alpha}
            L_f^{4n/\alpha}
            (1+n/4)^{4n/\alpha}
            + 1
        \big)
    \right)
$
& 
$(N+n +1)^2\,(\operatorname{Depth}+1)$
\\
\midrule
Implicit Parameter (D) &
$
\epsilon_A =
\sqrt{
N
}
n^{\frac{\alpha}{2}} W^{-\sqrt{D}}
(W^{(1-\alpha) \sqrt{D}} +2)
$ & - & -\\
%
\midrule
$\ln(N)$ & 
\multicolumn{3}{c}{
$
\ln\left(\kappa_{\xxx}\left(5^{-1}\right)\right)
    {
    \left\lceil
    \alpha^{-1}
    \left(
        \log_2(\operatorname{diam}(\xxx))
            -
        \log_2\left(
            \epsilon_A/3L_f
        \right)
        +
        \log_2\left(   \left( 
                    C_{\eta}
                    2 c \lceil 1 / \alpha \rceil \cdot \log_2\left(\kappa_{\xxx}\left(5^{-1}\right)\right)
                \right)_{++}    \right)
    \right)
    \right\rceil
    }
$
}
\\
\arrayrulecolor{lightgray}
\midrule
q & 
\multicolumn{3}{c}{
$\mathcal \mathcal{Q}_{f(\xxx)}(\epsilon_Q)$
}
\\
\bottomrule
\\
\end{tabular}
\end{adjustbox}
\end{table}

\begin{remark}[The Width Parameter in Tables~\ref{tab_spacecomplexity_universalmetrictransformer} and~\ref{tab_spacecomplexity_universalFFNN}]
\label{remark_W_width_parameter}
The width parameter $W$ in Proposition~\ref{prop_universal_approximation_improved_rates} only concerns the case where the approximating feedforward networks are wide and utilizes a trainable activation function of singular-ReLU-type (first column of Table~\ref{tab_spacecomplexity_universalFFNN}).  For example, one can take $W=\lceil \epsilon^{-1}\rceil$.  
\end{remark}

A key step in the derivation of  Theorem~\ref{thrm_main_StaticCase}, is the following refinement of the central universal approximation theorem for deep feedforward networks with (possibly) trainable activation function.  
Briefly, the following result is a universal approximation theorem, which reflects not only the complexity of the target function being approximated and the size of the compact subset $K\subseteq \rr^n$ on which the approximation holds, but also $K$'s \textit{fractal dimension}.  This refines many universal approximation theorems in the literature.  For instance, it refines \cite{pmlr-v75-yarotsky18a} which concerns functions defined on the unit cube, it provides a quantitative version of \cite{kidger2021neural}, and it parallels the findings of \cite{SHAHAM2018537} beyond the case where $K$ is a differentiable sub-manifold of $\rr^n$; all while allowing for trainable activation functions.  Note that in the special case where $\rr^m=\rr$, $K=[0,1]^n$, and $\sigma$ is of singular-ReLU type, we recover \citep[Theorem 1]{Shen2021ThreeLayersSuperExpressive}.  
\begin{proposition}[{$\NN[\cdot]$-Networks are Efficient Universal Approximators}]
\label{prop_universal_approximation_improved_rates}
Let $n,m\in \nn_+$, $K\subseteq \rr^n$ be a compact set with at least two points, $0<\alpha\leq 1$, $f\in C^{\alpha}(K,\rr^m)$, and let $\sigma$ be an activation function as in Definitions~\ref{defn_TrainableActivation_Singular},~\ref{defn_TrainableActivation_Smooth} or~\ref{defn_TrainableActivation_ClassicalNonTrainable}.  For every approximation error $\epsilon>0$ and any width parameter $W\in \nn_+$, there is a feedforward neural network $\hat{f}_{\theta}$ satisfying
\[
\sup_{x \in K}\, 
    \|
        f(x) - \hat{f}_{\theta}(x)
    \|
\leq
C_K
\epsilon,
\]
where the constant $C_K>0$ encodes the ``complexity'' of the input space $K$ and is defined by
\begin{equation}\label{def:ck}
C_K \eqdef  
c \,
\sqrt{m}
\lceil \alpha^{-1} \rceil
\underbrace{
\log_2\left( \kappa_K(5^{-1}) \right)
}_{\mbox{Dimension of $K$}}
\underbrace{\operatorname{diam}(K)^\alpha}_{\mbox{Size of $K$}},
\end{equation}
where $c>0$ is an absolute constant\footnote{By absolute constant, we mean that $c$ is independent of $W$,$\epsilon$, $n$, $m$, $\alpha$, and of $K$.} and $\kappa_K$ is defined in \eqref{eq.kk}. 
Furthermore, the space complexity of $\hat{f}_{\theta}$ is recorded in Table~\ref{tab_spacecomplexity_universalFFNN}.  
\end{proposition}

\begin{table}[H]%
    \centering
    \caption{Upper bounds on the model complexity of the feedforward network $\hat{f}_{\theta}$ in Proposition~\ref{prop_universal_approximation_improved_rates}.}
    \label{tab_spacecomplexity_universalFFNN}
    \begin{adjustbox}{width=\columnwidth,center}
    \begin{tabular}{@{}llll@{}}
    \toprule
    \textbf{Activation} $\sigma$ & \textbf{Singular}~(\ref{defn_TrainableActivation_Singular})
    &
    \textbf{Smooth}~(\ref{defn_TrainableActivation_Smooth})
    & 
    \textbf{Classical}~(\ref{defn_TrainableActivation_ClassicalNonTrainable})
    \citep{kidger2021neural}
    \\
    \midrule
    Depth $(J)$ &  
    $
    m(1+(2^6n D + 3))
    $
    &
    $
    \mathcal{O}\left(
    m\left(1+
              \tilde{\epsilon}^{-2n/\alpha}
              L_f^{2n/\alpha}
              (1+n/4)^{2n/\alpha}
            \right)
    \right)
    $
    &
    -
    \\
    \arrayrulecolor{lightgray}\hline
    Width  & 
    $
     n(m-1) 
        + 
    \max\{n,5W+13\}
    $ 
    &
    $
        nm+3
    $
    &
    n + m + 2
    \\
    \hline
    $\#$ Parameters P([d])
    & 
    $
    \left(
            \frac{11}{\,4 \,}n^2m^2 - 1
        \right) 
        m
        \max\{n+3,5W+16\}^2
        (2^6n D 
         +4 
        )
    $
    &
    $
        \mathcal{O}\left(
        \left( \frac{11}{ \, 4 \, } n^2m^2 - 1\right) 
        m
        (n + 6)^2
            \big(
                \tilde{\epsilon}^{-4n/\alpha}
                L_f^{4n/\alpha}
                (1+n/4)^{4n/\alpha}
                + 1
            \big)
        \right)
    $
    & 
    $(n+m +2)^2\,(\operatorname{Depth}+1)$
    \\
    \hline
    Implicit Parameter $(D)$ & 
    $
\epsilon
    =
n^{\frac{\alpha}{2}} W^{-\sqrt{D}}
(W^{(1-\alpha) \sqrt{D}} +2)
$
&
-
&
-
    \\
    \bottomrule
    \\
    \end{tabular}
    \end{adjustbox}
\end{table}

\section{Dynamic Case -- Universal Approximation of Causal Maps}
\label{s_Dynamic_Case}
This section's main result (Theorem~\ref{thrm_main_DynamicCase}) states that any function between discrete-time path spaces which flow information forward can be approximated by the dynamic extension of our geometric transformer neural network architecture. 
\label{s_Dynamic_Case__ss_Preliminaries}
\subsection{Compact Subsets in Discrete-time Path Spaces}
\label{s_Dynamic_Case__ss_Preliminaries___sss_Compact_in_Pathspace}
We are interested in compact subsets $\kkk$ of the path space $\xxx^{\zz}$. In order to detail some different classes of paths which  are relevant to our analysis, we need to fix the rate at which time flows, by specifying a time-grid $\{t_n\}_{n\in \zz}\subset \rr$.  Any $t_n<0$ represents the past, $t_0=0$ is the present, and any $t_n$ are future times $t_n>0$.  

On the time-grid we only assume that it spans all time, and that it is not overly sparse nor overly clustered at any instance in time.  These requirements are formalized by the following assumption.  Given any $\xb\in \xxx^{\zz}$, we define $\Delta_n\xb \eqdef  x_{t_n}-x_{t_{n-1}}$, and set $\Delta t_n \eqdef  t_{n+1}-t_n$.  

\begin{assumption}[Non-Degenerate Time-Grid]
\label{assumption_of_regular_gridmesh}
The time-grid $\mathbb{T} \eqdef \{t_n\}_{n\in \zz}\subset \rr$ satisfies: $t_0=0$ and
\begin{enumerate}
    \item[(i)] \textbf{Spans all time:} $\inf_{n\in \zz} \, t_n=-\infty$ and $\sup_{n\in \zz}\, t_n =\infty$,
    \item[(ii)] \textbf{Not overly sparse nor overly clustered:} 
        \[
            0<
        \delta_-
         \eqdef 
        \inf_{n\in \zz}\,  \Delta t_n 
            \leq 
        \sup_{n \in \zz}\,  \Delta t_n 
         \eqdef  {\delta_+}
        <\infty
        .
        \]
\end{enumerate}
\end{assumption}
Our dynamic universal approximation theorem will hold on compact subsets $\kkk$ of the path-space $\xxx^{\zz}$, which, by Tychonoff's Product Theorem \citep[Theorem 37.3]{munkres2000topology}, are of the form $\kkk= \prod_{n\in \zz} \, K_n$, where each $K_n\subseteq \xxx$ is compact.  Our worst-case model complexity result concerns compact subsets of the path space at this level of generality.

Although we take inspiration for considering compact path spaces of this form from the machine learning for dynamical systems and reservoir computing literature \cite{JPandLyydmila2019}, we are equally motivated by stochastic analysis; specifically from the high-probability behaviour of paths of solutions of stochastic differential equations (SDEs).  The next result gives a precise statement motivating our first compact class of paths.  We begin by considering the link between general stochastic differential equations, and the following compact path space
\begin{equation}
\label{eq:Exponential_Growth_PathSpace}
        K_{\mathbb{C},C^{\star},\varepsilon}^{\operatorname{exp}}
     \eqdef  
        \left\{\xb\in \xxx^{\zz}:\,
                \|x_0\|\le C_0 
            \mbox{ and }
                    (\forall n\in \mathbb{N}_+)\,
                    \|x_{t_n}\|
                \le 
                    \frac{
                            C^{\star}
                        }{
                            \varepsilon^{1/2}
                        }
                    \,
                        C_n^{1/2}
                    \, 
                        e^{-n \,C_n \,\delta_- /2 }
        \right\}
,
\end{equation}
which in addition to the time-grid $\mathbb{T}$ is defined by the following hyperparameters: a sequence of positive constants $\mathbb{C} \eqdef (C_n)_{n\in \mathbb{Z}}$, $C^{\star}>0$, and a fixed $0< \varepsilon\le 1$.  
We illustrate this phenomenon for SDEs with deterministic initial condition and a broad range of SDEs with random initial condition.
\begin{proposition}[{Common Paths of SDEs Satisfy our Compactness Conditions}]
\label{prop:HighProbConstainement}
Fix a time grid $\mathbb{T} \eqdef \{t_n\}_{n\in \mathbb{Z}}$ satisfying Assumption~\ref{assumption_of_regular_gridmesh} and a sequence of positive constants $\mathbb{L} \eqdef (L_n)_{n\in \mathbb{Z}}$.  
Let $(\Omega,\mathcal{F},\mathbb{F} \eqdef (\mathcal{F}_t)_{t\ge 0},\mathbb{P})$ be any filtered probability space supporting an $n$-dimensional Brownian motion $(W_t)_{t\ge 0}$, and any pair of $\mathbb{F}$-progressively measurable processes $(\alpha_t)_{t\ge 0}$ and $(\beta_t)_{t\ge 0}$ from $[0,\infty)\times \Omega\times \mathbb{R}^d$ to $\mathbb{R}^d$ and to the set of $d\times n$-matrices respectively, satisfying
\begin{enumerate}
    \item[(i)] \textbf{Integrability:} For every $n\in \mathbb{N}_+$, $
            \mathbb{E}\left[
                 \int_0^{t_n}\|\alpha_{t_n}(0)\|^2
                +
                \|\beta_{t_n}(0)\|^2dt
            \right]
        <
            \infty$;
    \item[(ii)] \textbf{Local Lipschitz Regularity:} 
        For every integer $n\in \mathbb{N}$, every $t\in [0,t_n]$, and every pair of points $x,y\in \mathbb{R}^d$, it holds that  $\|\beta_t(x)-\beta_t(y)\|+\|\alpha_t(x)-\alpha_t(y)\| \le L_n \|x-y\|$;
    \item[(iii)] \textbf{Initial Condition}: Assume either that $X_0$ is a constant $x\in \mathbb{R}^d$ or $X_0$ is a $\mathcal{F}_0$-measurable sub-Gaussian random vector and is independent of $W_0$.
\end{enumerate}
Let $(X_t)_{t\ge 0}$ be a solution to the SDE
\[
X_t  = X_0  + \int_0^t\, \alpha_s(X_s)ds  + \int_0^t\,\beta_s(X_s)\, dW_s, \qquad t\geq 0.
\]
Then, there exists a sequence of positive constants $(C_n)_{n\in \mathbb{Z}}$, depending only on the time-grid $\mathbb{T}$ and on the local Lipschitz constants $\mathbb{L}$, such that, for every $0<\varepsilon \le 1$, the discrete-time stochastic process $(X_{t_n})_{n \in \mathbb{N}}$ belongs to the compact subset $K_{\mathbb{C},C_{(\delta_-,\mathbb{L})},\varepsilon}^{\operatorname{exp}}\subseteq(\mathbb{R}^d)^{\mathbb{N}}$ with high probability:
\[
\mathbb{P}\Big(
(X_{t_n})_{n\in \mathbb{Z}}\, 
\in \,
K_{\mathbb{C},C_{(\delta_-,\mathbb{L})},\varepsilon}^{\operatorname{exp}}
\Big)
\gtrsim
(1-\varepsilon),
\]
where $C_{(\delta_-,\mathbb{L})}>0$ depends only on $\delta_-$ and on $\mathbb{L}$, and $\gtrsim$ hides a constant depending only on $X_0$'s law.
\end{proposition}
\begin{remark}
\label{remark:constantinSDEApplication}
The constant suppressed by $\gtrsim$ in Proposition~\ref{prop:HighProbConstainement} is $1$ if $X_0$ is deterministic, and when $X_0$ is sub-Gaussian, it is positive for large enough $L_0$.
\end{remark}

Proposition~\ref{prop:HighProbConstainement} guarantees that, with high probability, the paths of an SDE with a random sub-Gaussian initial state belongs to a path space of the form~\eqref{eq:Exponential_Growth_PathSpace}.  However, one can easily construct other types of stochastic processes whose typical paths lie in more general compact subsets of the product space $(\mathbb{R}^d)^{\mathbb{Z}}$.  Similarly to the weighted approximation literature \cite{prolla1971bishop,schmocker2022universal}, a broad family of path-spaces can be defined by a compact subset $K\subset \mathbb{R}^d$ and a monotone increasing map $w:\mathbb{Z}\rightarrow [0,\infty)$, which in analogy with the reservoir computing \cite{OrtegraGrigoryeva,JPandLyydmila2019} and approximation theory \cite{prolla1971bishop,schmocker2022universal} literature, we call a \textit{weighting function}.  Together,  the pair $(K,w)$ define the following path-space comprised of all paths in $(\mathbb{R}^d)^{\mathbb{Z}}$ that are $w(i)$-close to the compact set $K$ at time $t_i\in \mathbb{Z}$:
\[
        K^{w} 
    \eqdef  
        \big\{
        \xb\in (\rr^d)^{\mathbb{Z}}:\, 
        \forall i\,
        \exists y \in K \mbox{ s.t.\ }
        \|x_{t_i}-y\|\le w(|i|)
        \big\}
.
\]

More broadly, regular classes of paths will yield smaller networks.  In particular, we consider three classes of paths, $K^{\mathbb{Z}}$, $K_{C,p}^{\infty}$, and $K_{C,p}^{\alpha}$, induced by a fixed compact $K\subseteq \mathbb{R}^d$, and defined as follows.  
The first notable case considers uniformly bounded paths for all time.  This case is typical within the reservoir computing literature, see e.g. \cite{OrtegraGrigoryeva}, and is formalized by setting $\kkk=K^{\zz}$.  

The next case consists of paths that pass through $K$ at the present time, and whose time-increments uniformly control the $p$-variation.  For any $C,p>0$, we define
\[
K_{C,p}^{\infty}
     \eqdef  
\left\{\xb\in (\rr^d)^{\zz}:\,
x_0\in K  \mbox{ and } (\forall n\in \zz)\,
\|\Delta_n \xb\|^p \leq C 
    |\Delta t_n|
\right\}.
\]

Our last distinguished compact class of paths in $(\rr^d)^{\zz}$ mimics the approximation spaces of \citet{DeVoreLorentz1993ClassicConstructiveApproxBook}, recently studied in the neural network context in \cite{gribonval2021approximation}, and describes paths whose $p$-variation rapidly converges to $0$.  This set consists of paths that pass through $K$ at the present time, and whose weighted $p$-variation converges when re-weighted by a factor of $|n|_{++}^{\alpha}$.  
Formally, for any fixed $C>0$, $p\geq 1$, and any $0<\alpha<1-p$, we define
\[
K_{C,p}^{\alpha}
     \eqdef  
\left\{\xb\in (\rr^d)^{\zz}:\,
x_0\in K  \mbox{ and }
\sum_{n\in \zz}
\,
\frac{
\|\Delta_n \xb\|^p}
{
    |\Delta t_n||n|_{++}^{\alpha}
} \leq C
\right\}.
\]
Morally, $|n|_{++}$ is $|n|$ with the added technical point being that we avoid division by $0$ when $n=0$.

Next, we describe the functions between suitable discrete-time path spaces.  
\subsection{Causal Map of Approximable Complexity}
\label{s_Preliminaries__ss_Causal_Systems}
We build on the ideas of the reservoir computing literature \citep{boyd1985fading,jaeger2001echo,ManjunathJaeger_2013_ESP,jaeger2001echo,JPandLyydmila2019}, and on that of non-anticipative functionals of \cite{MR2586744,MR3059194}.  In our setting this translates to maps $F:\xxx^{\zz}\rightarrow \yyy^{\zz}$ between discrete-time path spaces which are \emph{causal} in the following sense. 

\begin{definition}[Causal Map]\label{defn_system}
Given any two metric spaces $\xxx,\yyy$, a map $F \colon \mathcal \xxx^\zz \to \mathcal \yyy^\zz$ is called a \emph{causal map} if, for every $t \in \zz$ and $x,x' \in \mathcal \xxx^\zz$ with $x_s = x_s'$ for $s \leq t$, we have $F(x)_s = F(x')_s$.
\end{definition}

Let us consider our first, and possibly most familiar, broad class of stochastic processes.  This example frames solutions to discrete-time SDEs in the language of causal maps taking values in paths in Wasserstein spaces.  

The crucial point here is that the framework of causal maps encompasses all classical discretized diffusion processes, such as any standard neural SDE model \cite{CKT20,gier20, Lyons_2021_NeuralSDEsGANs}.  We employ independent non-Gaussian noise in the SDEs in our example to illustrate that causal maps can comfortably describe much more complicated structures than what is describable with any diffusion model.  Similarly, we highlight that the drift and diffusion coefficients of the SDE are of much lower regularity than what can be handled with the classical theory of diffusion (without resorting to stochastic differential inclusions \cite{SDIs_book_2013}).  
\begin{example}[Discrete-time SDEs (with Non-Gaussian Noise)]
\label{ex_SDEs}
    Let $p\geq 1$ and $(\yyy,d_{\yyy})=(\mathcal{P}_p(\rr^d),\mathcal{W}_p)$.  Let $(W_n)_{n \in \mathbb N}$ be a sequence of independent standard Gaussians, and let $(Z_n)_{n\in \zz}$ be a sequence of i.i.d. random vectors in $\rr^d$, independent of $(W_n)_{n\in\nn}$, with $Z_0\sim \nu \in \mathcal{P}_p(\rr^d)$.  For simplicity, we assume that $\delta \eqdef \delta_+=\delta_-$, where $\delta_+$ and $\delta_-$ are as in Assumption~\ref{assumption_of_regular_gridmesh}. 
    Let the H\"older coefficients of the drift coefficients $\mu(t_n, \cdot) \in C^\alpha(\rrd,\rrd)$ and diffusion coefficients $\sigma(t_n, \cdot) \in C^\alpha(\rrd,\mathbb R^{d \times d})$ be uniformly bounded in $(t_n)_{n \in \mathbb N}$, and define at time $t_n$, for initial condition $X_{t_n} = x_{t_n} \in \rrd$, $X_{t_{n+1}}$ by
    \[
        X_{t_{n+1}} = x_{t_n}
        + \delta%
        \mu(t_n, x_{t_n}) 
        + \sqrt{\delta}\sigma(t_n, x_{t_n}) W_{n}
        + Z_{n}
        .
    \]
    This discrete-time SDE induces the  causal map
    \[
        F(\xb)_{t_n}
         \eqdef 
        \Law( X_{t_{n+1}} | X_{t_{n}} = x_{t_{n}} ), 
    \]
    where $\Law( Y | X = x)$ denotes the (a.s.\ well-defined) conditional law of a random variable $Y$ given $X=x$.
\end{example}
\begin{remark}[Higher Order Markovian SDEs]
\label{remark_extended_state_space}
By extending the state space, Example \ref{ex_SDEs} also covers higher order Markovian SDEs, i.e.\ when drift and diffusion coefficients depend on finitely many past states.  In order to ease reading, the example was presented in the Markovian case.
\end{remark}

In principle, a causal map's memory and internal structure may be infinitely complicated since one can easily construct such maps which depend on the arbitrarily distant past.  It would be surprising if these pathological causal maps can be approximated in any reasonable variant of the ``uniform on compact sets'' sense. Therefore, in analogy with \cite{boyd1985fading,Lukas2020}, we also exclude such maps.  Conversely, one would expect that any interesting universal approximation theorem for causal maps must encompass all causal maps that depend only on a finite (but potentially long) memory and process any path-segment using finite (but potentially large) number of Euclidean features.  
\begin{definition}[Causal Map of Finite Complexity]
\label{def_AdaptedMap_FiniteComplexity}
Let $L,m\in \nn_+$, $\alpha \in (0,1]$, $f\in C^{\alpha}(\rr\times \xxx^m,\rr^{L})$ (called the \emph{encoding map}), and $\rho\in C^{\alpha}(\rrflex{L},\yyy)$ (called the \emph{decoding map}).   
We associate to $f$ and $\rho$ the system
$F^{f,\rho}:{\xxx}^{\zz}\ni \xb\mapsto (\rho(f(
    t_n,
    x_{t_{n-m}:t_{n}}))
    )_{n\in \zz}\in \yyy^{\zz}$, 
    and call it a \emph{causal map of finite complexity}. $F^{f,\rho}$ is said to be \emph{time-homogeneous} if $f$ has no explicit dependence on the first (time) component.
\end{definition}
By using the parabolic PDE representation of an SDE, offered by the Feynman-Kac formula, \cite{ParabolicPDE_2020_Thesis_Arnulf, Vonwustemberger_2022_Thesis_Arnulf} have shown that feedforward neural networks can efficiently approximate most SDE's associated PDE.  Likewise, regular path-functionals of a jump-diffusion process with Lipschitz coefficients can efficiently be approximated by neural SDEs \cite{gonon2021deepNeuralSDE}.  In a similar spirit, we find that the causal maps of Example~\ref{ex_SDEs} are not only of finite complexity but, a fortiori, they admit a simple representation specified by only a small number of parameters.  

\begin{example}[Discrete-time SDEs (with Non-Gaussian Noise)]
\label{ex_SDEs_are_FC_adaptedMaps}
    Continuing Example~\ref{ex_SDEs}, with the simplified assumption that $\delta=1$, we show that the causal map $F$ is of finite complexity, since it can be expressed as $F(\xb)_{t_n}=\rho(f(t_n,x_{t_n}))$, where the encoding and decoding maps are defined by
    \[
    f(t,x)  \eqdef  
    \left(
        x+ \mu(t,x)
    ,
        \sigma(t,x)
    \right)\quad
    \mbox{ and }\quad
    \rho(\mu,\sigma)%
         \eqdef 
    N_d(\mu,\sigma \sigma^{\top}) \star \nu,
    \]
    where $\star$ denotes the convolution, and  $N_d(\mu,\sigma \sigma^\top)$ a $d$-dimensional normal distribution with mean $\mu$ and covariance matrix $\sigma \sigma^\top$.
    Clearly, $f(t_n, \cdot) \in C^\alpha(\rrd, \rrd \times \mathbb R^{d \times d})$, and moreover we have
    \begin{align*} 
        \mathcal W_p(\rho(\mu_1,\sigma_1), \rho(\mu_2,\sigma_2)) 
        &\le 
        \mathcal W_p(N_d(\mu_1,\sigma_1 \sigma_1^\top), N_d(\mu_2, \sigma_2 \sigma_2^\top))
        \\
        &\le
        \mathbb E 
        \left[ 
            \|\mu_1 - \mu_2 + (\sigma_1 - \sigma_2) \cdot W_n\|^p
        \right]^{1 / p}
        \\
        &\le
        \|\mu_1 - \mu_2\|
        +
        \|\sigma_1 - \sigma_2\|
        \,
        \mathbb E
        \left[
            \|W_n\|^p
        \right]^{1 / p}
    ,
    \end{align*}
so that $\rho $ belongs to $C^1(\rrd \times \mathbb R^{d \times d}, (\mathcal P_p(\rrd), \mathcal W_p))$.
\end{example}
In general, causal maps of finite complexity extend Example~\ref{ex_SDEs} in a number of directions.  For instance, they can describe stochastic processes which become Markovian in an extended state-space, encoding finitely many previous states realized by the process.  Furthermore, the map $f$ in Example~\ref{ex_SDEs_are_FC_adaptedMaps} is a very particular case of an encoding map, and it \doesnothaveto{} be an affine function of the process' current state, drift, and volatility.  In general, $f$ can be replaced by any H\"{o}lder function of time and current state, as in the theory of (non-linear) random dynamical systems (see \citep[Section 4.5]{FreidlinWentzell_1984}). 
Moreover, the decoding map \doesnotneedto{} take values in the subspace of the Wasserstein space $(\pp_1(\rr^d),\mathcal{W}_1)$ of $d$-dimensional Gaussian measures convoluted with $\nu$, representing the possible process' next step.  Rather, $\rho$ can be any H\"{o}lder map into $(\pp_1(\rr^d),\mathcal{W}_1)$, or more broadly, $\rho$ can map into the \textit{adapted Wasserstein space} $(\pp_1(\rr^{dN_F}),\mathcal{AW}_1)$ of \cite{Ruschendor_1985_AdaptedOG} which robustly describes the process' conditional law on the next $N_F$ future steps.  Note that these two generalizations coincide when $N_F=1$.  This is detailed in Example~\ref{ex_SDEs_AW_simple} in Section~\ref{s_Examples__ss_Examples_CausalSystems_wt_AC} below.  Finally, $f$  \doesnothaveto{} have closed-form expression depending only on a finite segment of the process' history nor relating that data to a finite number of latent parameters decoded by $\rho$.  Analogously, $\rho$ can be highly complicated and \doesnotneedto{} be expressible in closed-form as a decoding map depending on finitely many parameters.  

These considerations naturally lead to our reinterpretation of 
the \textit{fading memory property}, which was first formalized by \cite{boyd1985fading} but whose origins date back to ideas of Volterra and Wiener \cite{Wiener_1958_FMP}.  We note that, since its introduction, the fading memory property has been a central tool for deriving approximation theorems for dynamical systems between Euclidean spaces (see \cite{lukovsevivcius2009reservoir,JPandLyydmila2019,Lukas2020,Manjunath_2020_RMS}) and is closely linked to the Echo State Property \cite{jaeger2001echo} that is key to reservoir computing \cite{ManjunathJaeger_2013_ESP,GONON202110,GRIGORYEVA2018495}. 
\begin{definition}[Approximable Complexity]
\label{defn_compression_memory_property}
A causal map $F:\xxx^{\zz}\rightarrow {\yyy}^{\zz}$ is of \emph{approximable complexity} (AC Map) if there exist functions $m,L:(0,\infty)\rightarrow 
\nn_+$ and $c_{AC}:\zz\times (0,\infty)\rightarrow [1,\infty)$ such that,
for each $\epsilon>0$ and each compact $\kkk \subseteq \xxx^{\zz}$, there is $\alpha\in (0,1]$ and there exist 
an approximate encoding map $f_{\epsilon}\in C^{\alpha}(\rr\times \xxx^{m(\epsilon)},\rr^{L(\epsilon)})$
and an approximate decoding map $\rho_{\epsilon}\in C^{\alpha}(\rr^{L(\epsilon)},\yyy)$ such that the associated causal map of finite complexity satisfies
\begin{equation}
\sup_{n \in \zz}
\sup_{\xb\in \kkk}
\,
\frac{d_{\yyy}(F^{f_{\epsilon},\rho_{\epsilon}}(\xb)_{t_n},F(\xb)_{t_n})}{
        c_{AC}(n,\epsilon)
    }<\epsilon
\label{eq_approximation_sense}
.
\end{equation}
An AC map is said to be \emph{time-homogeneous} if there is a family $(F^{f_{\epsilon},\rho_{\epsilon}})_{\epsilon>0}$ of time-homogeneous systems of finite complexity satisfying~\eqref{eq_approximation_sense}.  
\end{definition}
Intuitively, an AC map $F$ is very close to some causal map of finite complexity $F^{f_{\epsilon},\rho_{\epsilon}}$ on some (possibly large) time window around the current time, after which the two may begin to drift apart.  Figure~\ref{fig:compressionrate} illustrates a typical output of $F(\xb)$ (in violet) and $F^{f_{\epsilon},\rho_{\epsilon}}$ (in orange) evaluated on some path $\xb \in \xxx^{\zz}$. 
The rate at which this drifting apart occurs is expressed by the compression rate $c_{AC}$.  The value of $c_{AC}$ is illustrated by the width of the turquoise region in Figure~\ref{fig:compressionrate}. 

\begin{figure}[H]%
    \centering
    \includegraphics[width=.5\linewidth]{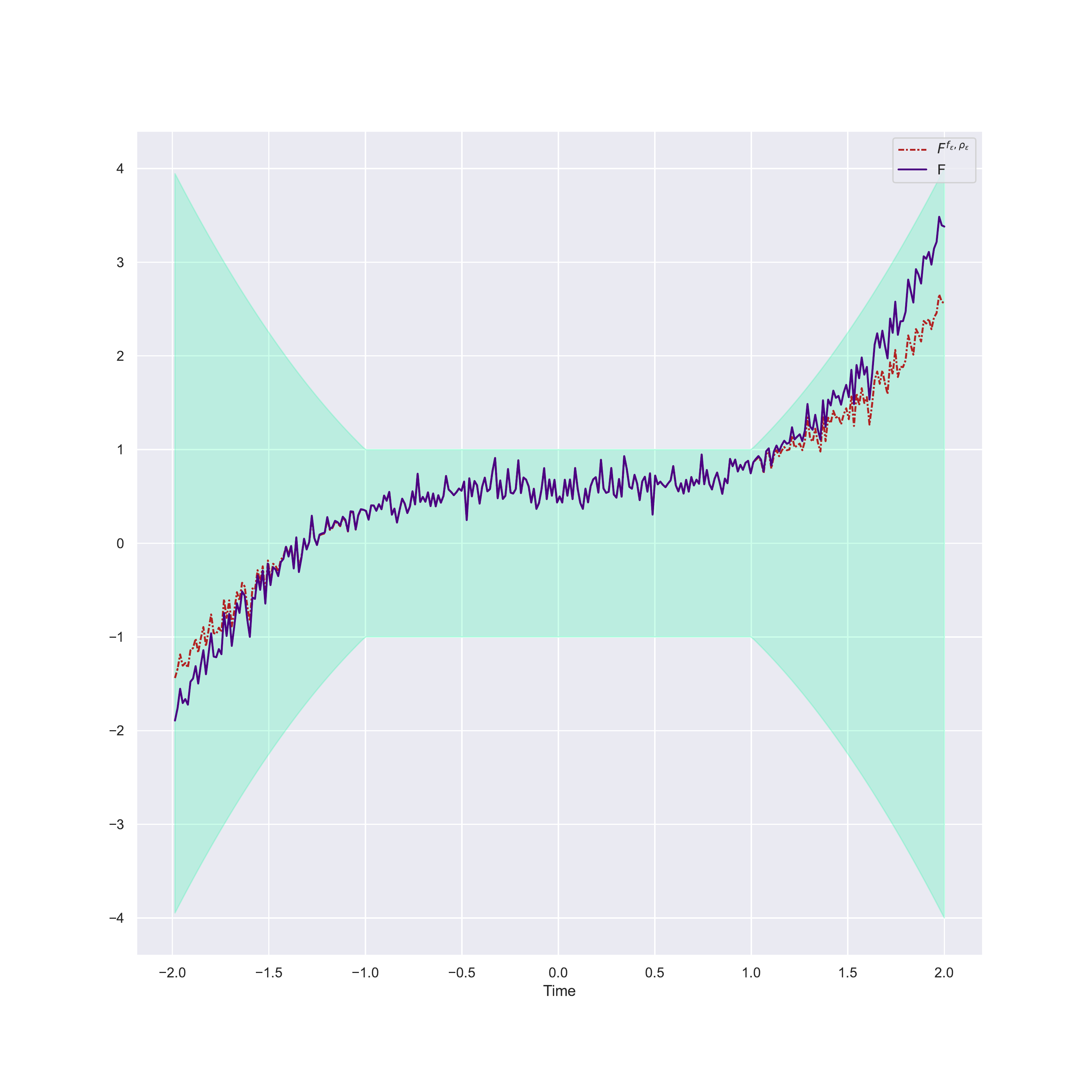}
    \caption{An AC map $F$'s compression rate $c_{AC}$, in Definition~\ref{defn_compression_memory_property}, quantifies the rate at which a finitely-parameterized approximation of $F$ decays.}
    \label{fig:compressionrate}
\end{figure}
 
In the case of $\yyy=\rr$, for any $n\in \zz$ and any $\xb \in \xxx^\zz$ the condition in \eqref{eq_approximation_sense} is equivalent to
\[
F(\xb)_{t_n}
    \in 
\left[
F^{f_{\epsilon},\rho_{\epsilon}}(\xb)_{t_n} - c_{AC}(n,\epsilon) \epsilon
,
F^{f_{\epsilon},\rho_{\epsilon}}(\xb)_{t_n} + c_{AC}(n,\epsilon) \epsilon
\right]
.
\]
Therefore, the outputs of the (possibly intractable) causal map $F$ each belongs to a region defined by the (tractable) causal map of finite complexity $F^{f_{\epsilon},\rho_{\epsilon}}$.  
An example of an AC map, but which is not of finite complexity, is given in Section~\ref{s_Examples}.  
\subsection{The Geometric Hypertransformer (GHT) Model}
\label{s_Dynamic_Case__ss_Preliminaries___sss_Our_Model}
We now introduce our geometric deep learning model, which we will use to approximate AC maps.  
The idea is to approximate AC maps by sequences of geometric transformers, which we sew together using a small auxiliary hypernetwork.  Just as in Figure~\ref{fig_metric_hypertransformer_TIKZ}, the role of this hypernetwork is to sequentially generate the next geometric transformer (given the current geometric transformer's parameters).  From the lens of stochastic analysis, the hypernetwork encoding the dynamic version of our geometric transformer is akin to a Feller process' infinitesimal generator \citep[Theorem II.3.8]{EngelNagel2000OneParameterSemigroups}.  
\begin{definition}[Geometric Hypertransformer]
\label{defn_hypernetwork}
Let $(\yyy,d_{\yyy},\eta,Q)$ be a QAS space, and fix positive integers $m,L,d$, a multi-index $[d]$ where $d_0=(m+1)d$ and $d_J=L$, and an
activation function $\sigma$.  
Let $\hat{\rho}:\rr^L\rightarrow \yyy$ be a geometric transformer, $h\in \mathcal{NN}^{ReLU}_{\cdot}$ a network mapping $\rr^{P([d])}$ to itself, $\theta \in \rr^{P([d])}$, and let $N\in \nn_+$.  The \emph{geometric hypertransformer} (GHT) generated by $(\hat{\rho},h,\theta,N)$ is the 
causal map $F^{(\hat{\rho},h,\theta,N)}:(\rr^d)^{\zz}\rightarrow \yyy^{\zz}$ defined by
\[
F^{(\hat{\rho},h,\theta,N)}(\xb)_{t_n}
     \eqdef 
\hat{\rho}\circ \hat{f}_{\theta_n}(x_{t_{n-m}:t_n}),\qquad n\in\zz,
\]
where $\hat{f}_{\theta_n}\in \NN[[d]]$ for all $n \in \zz$, and the parameters $(\theta_n)_{n\in \zz}$ are defined recursively by
\begin{equation}\label{eq:thetarec}
\theta_n \eqdef \theta \; \mbox{ for } n\leq -N
    \quad\mbox{ and }\quad
\theta_{n+1}
     \eqdef 
\begin{cases}
    h(\theta_{n} ) & : -N< n<N\\
    \theta_n &: n\geq N.
\end{cases}
\end{equation}
\end{definition}
Transformer networks typically carry an encoder-decoder structure, which means that each $\hat{f}$ can be thought of as the composition of two neural networks.  The role of the first network is to encode the incoming input information into some deep latent features designed to optimize the prediction of the second, decoder, network, whose role is to generate the predictions at each time-step.  In our context of the AC map $F$, the role of each of these networks becomes explicit.  Namely, the encoder network's role will be to approximate the $f_{\epsilon}$ given by its approximable complexity, and the role of the decoder network is then to approximate the measure-valued map $\rho_{\epsilon}$, also given by $F$'s approximable complexity.  

GHTs can approximate AC maps on an arbitrarily long but finite-time horizon without suffering from performance degradation.  However, this is not the case when approximating an AC map across an infinite-time horizon, as in this case the GHT's approximation quality will eventually begin to degrade past a prespecified moment in time.  Given a discrete path space $\kkk\subseteq \xxx^{\mathbb{Z}}$, the rate at which the performance of a GHT $\hat{F}:{\xxx}^{\mathbb{Z}}\rightarrow {\yyy}^{\mathbb{Z}}$ degrades beyond a finite-time horizon $N_T\in \mathbb{N}_+$, with hyperparameter $\lambda>0$, is quantified by the map $c^{\hat{F}}_{\kkk,N_T,\lambda}:\mathbb{Z}\rightarrow [1,\infty)$ defined on any $n\in \mathbb{N}_+$ by    
\begin{equation}
\label{eq:self_compression_function}
        c^{\hat{F}}_{\kkk,N_T,\lambda}(n)
    \eqdef
        \sup_{x\in \kkk}\,
            \max\big\{1,
                \lambda\,
                d_{\yyy}\big(
                    \hat{F}(x)_{t_n}
                ,
                    \hat{F}(x)_{t_{\operatorname{sgn}(n)\,\cdot N_T}}
                \big)
            \big\}
    ,
\end{equation}
where $\operatorname{sgn}(n)\eqdef -1$ if $n$ is a negative integer and $\operatorname{sgn}(n)=1$ otherwise.
Outside the time-window $\{-N_T,\dots,N_T\}$, the map in~\eqref{eq:self_compression_function} plays an analogous re-normalizing role to an AC maps' compression rate $c_{AC}$.  

\subsection{Main Result -- GHTs are Universal Causal Maps} 
\label{s_Dynamic_Case__ss_Main_Results}
In what follows, we will use the notation
\begin{equation}\label{def:GHT}
N_T \eqdef 
\min\big\{
    \min\{
    n\in \nn_+:\,
        t_n \geq T 
    \}
,
    \left|\max\{
    n\in \nn_-:\,
        t_n \leq -T 
    \}
\right|
\big\}
.
\end{equation}
Moreover, in Theorem~\ref{thrm_main_DynamicCase}, for a fixed compact set $\kkk\subseteq \xxx^{\zz}$,
$\xxx\subseteq\rr^d$, and an AC map $F:\xxx^{\zz}\rightarrow \yyy^\zz$, we will denote by
$L_{\alpha,\rho_{\epsilon}}$ and $L_{\alpha,f_{\epsilon}}$ the \Holderseminorm{} of $\rho_{\epsilon}$ and $f_{\epsilon}$ in~\eqref{eq_approximation_sense}, and set
$K_n \eqdef \operatorname{pj}_n(\kkk)$, where $\operatorname{pj}_n$ is the projection into the $n$-th coordinate, $\operatorname{pj}_n:(\rr^d)^{\zz}\ni (x_{t_u})_{u\in \zz}\mapsto x_{t_n}\in \rr^d$. 

\begin{theorem}[Adapted Universal Approximation via GHTs]
\label{thrm_main_DynamicCase}
Let $\xxx$ be a subset of $\rr^d$ and $\kkk\subseteq \xxx^{\zz}$ a compact subset.
Fix any AC map $F:\xxx^{\zz}\rightarrow \yyy^\zz$, and a ``time span'' $T>0$, with $T=t_{\bar{n}}$ for some $\bar{n}\in \nn_+$.    
Then, for every $\epsilon>0$, there is a ``compression rate'' $c_{\eps}:\zz \rightarrow (0,\infty)$, with $c_{\eps}(n)=1$ if $|n|\leq N_T$ and otherwise recorded in Table~\ref{tab_extrapolation_function} (depending on $\kkk$), such that the following holds:
there are a multi-index $[d]$, a geometric transformer $\hat{\rho} \in \GT[\cdot,N,q]
$, a hypernetwork $h\in \mathcal{NN}^{ReLU}_{\cdot}$ mapping $\rr^{P([d])}$ to itself, $\theta \in \rr^{P([d])}$, and $N\in \nn_+$, such that the GHT $F^{(\hat{\rho},h,\theta,N)}$
generated by $(\hat{\rho},h,\theta,N)$ satisfies
\[
\sup_{n\in \zz}\sup_{\xb\in\kkk}\,
\frac{
  d_{\yyy}\left(
F(\xb)_{t_n}, F^{(\hat{\rho},h,\theta,N)}(\xb)_{t_n}
\right)
}{
        \max\left\{c_{AC}(n,\epsilon)
    ,
        c_{\eps}(n)
    ,
        c^{F^{(\hat{\rho},h,\theta,N)}}_{\kkk,N_T,8/\eps }(n)
    \right\}
} 
    < \epsilon
.
\]
Moreover, $F^{(\hat{\rho},h,\theta,N)}$'s model complexity can be estimated by:
\begin{enumerate}
    \item \textbf{Encoder Complexity:} As in Table~\ref{tab_spacecomplexity_universalFFNN}, 
     for $\epsilon=\min_{n=-N_T,\dots,N_T}\;\frac{1}{C_{K_{t_n-dm(\epsilon/4):t_n}}}
           \left(\frac{\epsilon}{8 L_{\alpha,\rho_{\epsilon/4}}}\right)^{1/\alpha}$, where $C_{K_{
                t_n-dm(\epsilon/4):t_n
            }}$ is the constant defined in \eqref{def:ck}, and $\alpha$ is the regularity of $f_{\epsilon/4}$ as in~\eqref{eq_approximation_sense}.
    \item \textbf{Decoder Complexity:} As in Table~\ref{tab_spacecomplexity_universalmetrictransformer}
    \item \textbf{Hypernetwork Complexity:} $[d]=[P,M,M,P]$, where $P$ is the number of trainable parameters of the encoder network, as described in Table~\ref{tab_spacecomplexity_universalFFNN}, and $M$ is given by:
\[
M =
\min
    \left\{
    \tilde{M}\in \nn:\, 
        2\left\lfloor
        \frac{\tilde{M}}{2}
        \right\rfloor
        \left\lfloor
        \frac{\tilde{M}}{4P}
        \right\rfloor
            \geq 
        N_T
    \right\}
.
\]
\end{enumerate}
\end{theorem}

\begin{remark}
\label{remark_independance_of_extrapolation_rates_of_F}
From Table~\ref{tab_extrapolation_function}, one can see that only in the case of general paths $\kkk$ there is explicit dependence on $F$, whereas the other cases depend only on its regularity.  
\end{remark}

\begin{remark}
When applying Theorem~\ref{thrm_main_DynamicCase} to the sample paths of the solution to an SDE, as in Proposition~\ref{prop:HighProbConstainement}, the statement should be interpreted as a high probability guarantee, depending on the parameter $0<\varepsilon\le 1$ from Proposition~\ref{prop:HighProbConstainement}.
\end{remark}

The new compression rate defined in Theorem~\ref{thrm_main_DynamicCase} by $\max\left\{c_{AC}(\cdot,\eps),c_\eps(\cdot),c^{F^{(\hat{\rho},h,\theta,N)}}_{\kkk,N_T,8/\eps }(\cdot)\right\}$ estimates the width of the turquoise region in Figure~\ref{fig:compressionrate} when approximating $F$ on an infinite-time horizon.  We emphasize that, in the interval $[-T,T]$, $c_\eps=1$ and $c^{F^{(\hat{\rho},h,\theta,N)}}_{\kkk,N_T,8/\eps }=1$ and therefore the compression rate $\max\{c_{AC}(\cdot,\eps),c_{\eps}(\cdot),c^{F^{(\hat{\rho},h,\theta,N)}}_{\kkk,N_T,8/\eps }(\cdot)\}$  coincides with $c_{AC}(\cdot,\eps)$ therein. %

We also point out that, if $F$ is a causal map of finite complexity, then $c_{AC}(n,\eps)=1$ for all $n\in \zz$. Therefore, in this case $\max\{c_{AC}(\cdot,\eps),c_{\eps}(\cdot),c^{F^{(\hat{\rho},h,\theta,N)}}_{\kkk,N_T,8/\eps }(\cdot)\}=1$ in the time-interval $[-T,T]$.  Naturally, $c$ depends on the regularity of the paths on which the causal map $F$ is being approximated, as reflected by our rates recorded in Table~\ref{tab_extrapolation_function}.

\begin{table}[H]%
    \centering
    \caption{Compression rate required outside the prescribed time-horizon, in the universal approximation theorem of Theorem~\ref{thrm_main_DynamicCase}, as a function of the path space.}
    \label{tab_extrapolation_function}
    \begin{adjustbox}{width=\columnwidth,center}
    \begin{tabular}{@{}ll@{}}
    \toprule
    \textbf{Compact Path-space} & \textbf{Compression Rate} $c_{\eps}(n)$ for $|n|>N_T$\\
    \arrayrulecolor{lightgray}
    \midrule
    $K^w$ & 
            $
                4\epsilon^{-1}
                \,
                \omega_{\rho_{\epsilon/4}}\circ \omega_{f_{\epsilon/4}}\biggl(
                        (|n|-N_T)\delta_{+}
                    + 
                      dm(\epsilon/4)\, 
                        \Big(
                            \operatorname{diam}(K) 
                        +
                            w(|n|) 
                        +
                            w(N_T)
                        \Big)
                \biggr)
            $
    \\
    \arrayrulecolor{lightgray}\hline
        $K_{\mathbb{C},C^{\star},\varepsilon}^{\operatorname{exp}}$
    &
            $
                4\epsilon^{-1}
                \,
                \omega_{\rho_{\epsilon/4}}\circ \omega_{f_{\epsilon/4}}\biggl(
                        (|n|-N_T)\delta_{+}
                    + 
                      dm(\epsilon/4)\, 
                        \Big(
                            \operatorname{diam}(K) 
                        +
                            w_{\boldsymbol{C},C^{\star}}(n)
                        +
                            w_{\boldsymbol{C},C^{\star}}(N_T)
                        \Big)
                \biggr)
            $
    \\
    \arrayrulecolor{lightgray}\hline
    $K_{C,p}^{\infty}$     & 
    $
    4\epsilon^{-1}
    L_{\alpha,\rho_{\epsilon/4}}
    L_{\alpha,f_{\epsilon/4}}^{\alpha}
    \left(
    (|n|-N_T)
    \delta_+
    \left[
    1
    +
        (dm(\epsilon/4) +1) C^{\frac1{p}} {\delta_+}^{\frac{1-p}{p}}
    \right]
    \right)^{\alpha^2}
    $
    \\
    \arrayrulecolor{lightgray}\hline
    $K_{C,p}^{\alpha}
    $ &   
    $
    4\epsilon^{-1}
    L_{\alpha,\rho_{\epsilon/4}}
    L_{\alpha,f_{\epsilon/4}}^{\alpha}
    \left(
    (|n|-N_T)
    \delta_+  + (dm(\epsilon/4)+1)C^{\frac{1}{p}}
    {\delta_+}^{\frac{1}{p}}
    \left(
        1+2
        \zeta\left(
                \frac{\alpha}{p-1}
            \right)
    \right)^{\frac{p-1}{p}}\right)^{\alpha^2}
    $
\\
\arrayrulecolor{lightgray}\hline
$K^{\zz}$ &  
    $
    4\epsilon^{-1}
    L_{\alpha,\rho_{\epsilon/4}}
    L_{\alpha,f_{\epsilon/4}}^{\alpha}
    \left(
        (|n|-N_T){\delta_+}
            +
         (dm(\epsilon/4)+1)\operatorname{diam}(K)
    \right)^{\alpha^2}
    $
    \\
\arrayrulecolor{lightgray}\hline
``Worst-Case'' Arbitrary $\kkk$ & 
$
\max_{\xb\in \kkk,\, k\leq |n|}\,
\,
4\epsilon^{-1}
\max\{1,
d_{\yyy}(F^{\rho_\epsilon,f_\epsilon}(\xb)_{t_{k}},F^{\rho_\epsilon,f_\epsilon}(\xb)_{ t_n})
\}
$
\\
\bottomrule
    \end{tabular}
    \end{adjustbox}
    \caption*{
    Here, $K\subseteq \rr^d$ is compact, $C>0$, $\alpha <1-p$, $p\geq 1$ when $\kkk=K_{C,p}^{\alpha}$, and $p>0$ when $\kkk=K_{C,p}^{\infty}$,
    $L_{\alpha,\rho_{\epsilon}}$, and $L_{\alpha,f_{\epsilon}}$ are the \Holderseminorm{s} of $\rho_{\epsilon}$ and $f_{\epsilon}$ in~\eqref{eq_approximation_sense}, and $\zeta$ is the Riemann zeta function.   
    When $\kkk=K_{\boldsymbol{C},C^{\star},\varepsilon}^{\exp}$, $w_{\boldsymbol{C},C^{\star}}(n)\eqdef \max\{C_0,\varepsilon^{-1/2}\,C^{\star}C_n^{1/2}e^{-nC_n\delta_-/2} \}$, for each $n\in \mathbb{N}_+$, the positive constants $\boldsymbol{C}=(C_n)_{n=0}^{\infty}$ and $C^{\star}$ are as in Proposition~\ref{prop:HighProbConstainement} and $\varepsilon>0$ can be taken to be the approximation error from Theorem~\ref{thrm_main_DynamicCase}.
    }
\end{table}

Furthermore, in the case where $F$ is time-homogeneous and $\kkk=K^{\zz}$, we recover the following analogue of the reservoir-computing-type results of \cite{OrtegraGrigoryeva,Lukas2020}.  
\begin{corollary}[Approximation of Time-Homogeneous Causal Maps]
\label{cor_reservoir_analogue___time_homogeneous_case}
Assume the setting of Theorem~\ref{thrm_main_DynamicCase} and suppose further that $F$ is time-homogeneous and $\kkk=K^{\zz}$ for some compact set $K\subseteq \RR^d$.  
Then, the compression rate $c_{\eps}$ is $1$ for $|n|\leq N_T$, while for $|n|>N_T$ it is
\[
c_{\eps}(n)  = 
    4\epsilon^{-1}
    L_{\alpha,\rho_{\epsilon/4}}
    L_{\alpha,f_{\epsilon/4}}^{\alpha}
    \left(
         (dm(\epsilon/4)+1)\operatorname{diam}(K)
    \right)^{\alpha^2}.
\]
\end{corollary}

\section{Examples of Geometric Attention Mechanisms}
\label{s_Examples}
\subsection{Examples of QAS Spaces}
\label{s_Examples__ss_Examples}
We illustrate the scope of our causal geometric deep learning framework with various examples of the spaces covered by our theory.  In each case, we work out exactly what the geometric attention mechanism in \textit{closed-form} is and, therefore, the form of the geometric transformer in that context. 
We begin with the infinite-dimensional linear case.
Then, we progress to non-linear examples arising from optimal transport theory, and finally we transition to the finite dimensional but non-Euclidean setting within the context of information geometry.  

\subsubsection{Linear Spaces}
\label{s_Preliminaries__ss_Examples___sss_Linear}
\begin{figure}[ht]
    \centering
    \includegraphics[width=1\textwidth]{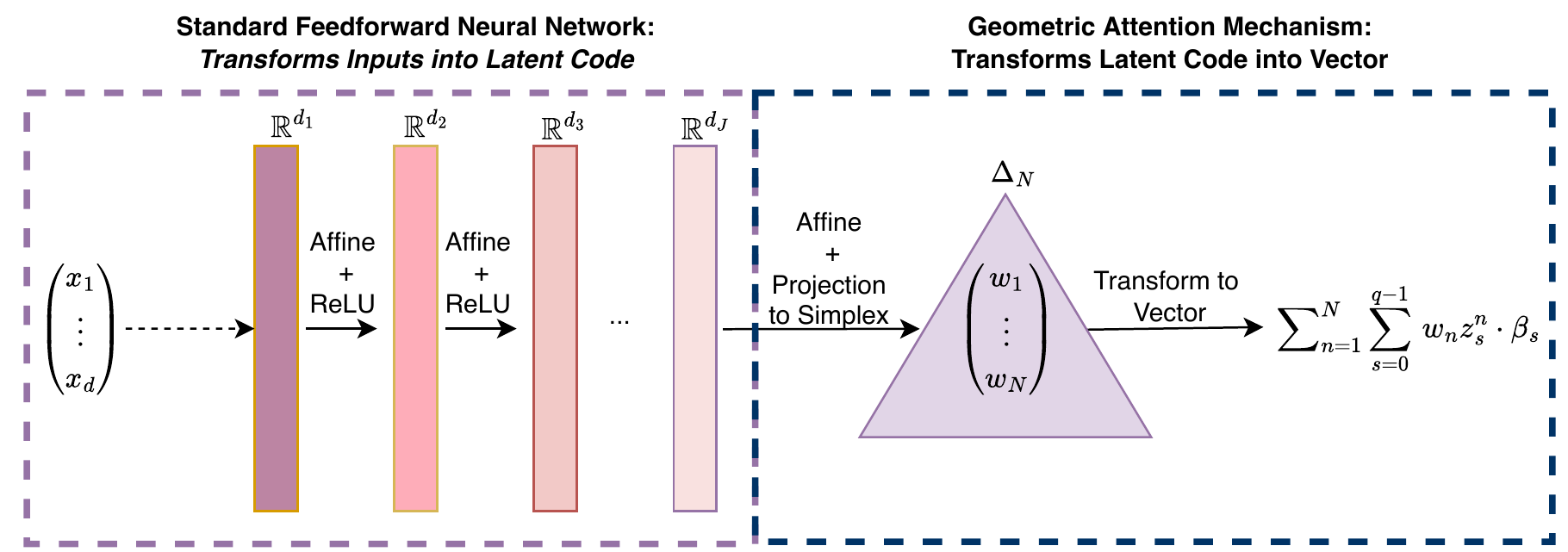}
    \caption{The \textit{geometric transformer} of Example~\ref{ex_Frechet} maps inputs in $\mathbb{R}^d$ to outputs in an \textit{infinite-dimensional Fr\'{e}chet space with a Schauder basis}.  The outputs are convex combinations of $N$ vectors, in the Fr\'{e}chet space, each of which can be exactly implemented as a linear combination of the first $q$ basis vectors $\{\beta_s\}_{s=0}^{q-1}$.}
    \label{fig:Transformer_Rd_Frechet}
 \end{figure}

Our framework easily encompasses a broad range of topological vector spaces relevant throughout much of applied probability theory and, in particular, to mathematical finance.  These include all Euclidean spaces and any ``well-behaved'' Banach space, such as all $L^p$-spaces on $\sigma$-finite measure spaces where $1\leq p<\infty$.

\begin{example}[Fr\'{e}chet Spaces]
\label{ex_Frechet}
Let $(\yyy,d_{\yyy})$ be a Fr\'{e}chet space.  For any $N\in \nn_+$, we may define the mixing function $\eta$ to be the map sending any $w\in \Delta_N$ and any $(y_n)_{n=1}^N\in \yyy^N$ to
\[
\eta(w,(y_n)_{n=1}^N) \eqdef \sum_{n=1}^N w_n y_n
.
\]
Let $\dim(\yyy)$ denote dimension of $\yyy$ as a vector space.  
If $\yyy$ admits a Schauder basis\footnote{
A Schauder basis for a Fr\'{e}chet space $(\yyy,d_{\yyy})$ is a linearly independent set  $(\beta_s)_{s=0}^{q^{\star}-1}\subseteq\yyy$ for which, given any $y\in \yyy$ there exists a unique sequence $(z_s^y)_{s=0}^{q^{\star}-1}$ in $\rr$ satisfying $\lim\limits_{q \uparrow q^\star}\,
d_{\yyy}\left(y,\sum_{s=0}^{q-1} z_s^y{\beta_s}\right)=0$.} %
$(\beta_s)_{s=0}^{q^\star-1}$, where $q^{\star} \eqdef \min\{\dim(\yyy),\#\nn \}$, then $(\yyy,d_{\yyy})$ is quantized by the functions $(Q_q)_{q \in \nn}$ defined via
\[
Q_q:\rr^{q^{\star}} \ni 
    z 
        \mapsto 
    \sum_{s=0}^{\min\{q,q^{\star}\}-1}
    \, 
    z_s 
        \cdot 
    \beta_s
\in \yyy,
\]
where $z_s$ denotes the $s$-th component of $z$.
Therefore, Fr\'{e}chet spaces with Schauder basis can be endowed with a QAS space structure whose associated geometric attention is
\[
\operatorname{attention}_{N,q}(
        u
            ,
        (z^n)_{n=1}^N
    )
     \eqdef 
\operatorname{\sum}_{n=1}^N
\sum_{s=0}^{\min\{q,q^{\star}\}-1}
    \Pi_{\Delta_N}(u)_n
    z_s^n
        \cdot 
    \beta_s
.
\]
\end{example}
In particular, Example~\ref{ex_Frechet} shows that our static universal approximation result for our architecture can approximate more general output spaces than the DeepONets of \cite{LuJinPanZhangKarniadakis_DeepONetNature_2021,LiuYandChenZhaoLiao_DeepOReLUNets_2022New}.  Further, our results yield quantitative counterparts to the Banach space-valued results of \cite{benth2021neural}'s qualitative universal approximation results for their feedforward-type architecture.  

Next, we explicate Example~\ref{ex_Frechet} with a Fr\'{e}chet space central to classical mathematical finance.  Namely, we use it to design a universal deep neural approach to \textit{term structure modeling} which is compatible with the Heath-Jarrow-Morton (HJM) framework of \cite{heath1992bond}.  We note that, machine learning models have recently found successful applications in HJM-type frameworks since one can approximately impose no-arbitrage restrictions into the learned model \cite{gambara2020consistent,kratsios2020deep}.  We also note that a feedforward counterpart to our transformer approach for forward-rate curve modeling in the Forward rate curve space of \cite{Damir_Thesis_2001} has recently been considered in \cite{benth2022pricing}.    

\begin{figure}[ht]
    \centering
    \includegraphics[width=1\textwidth]{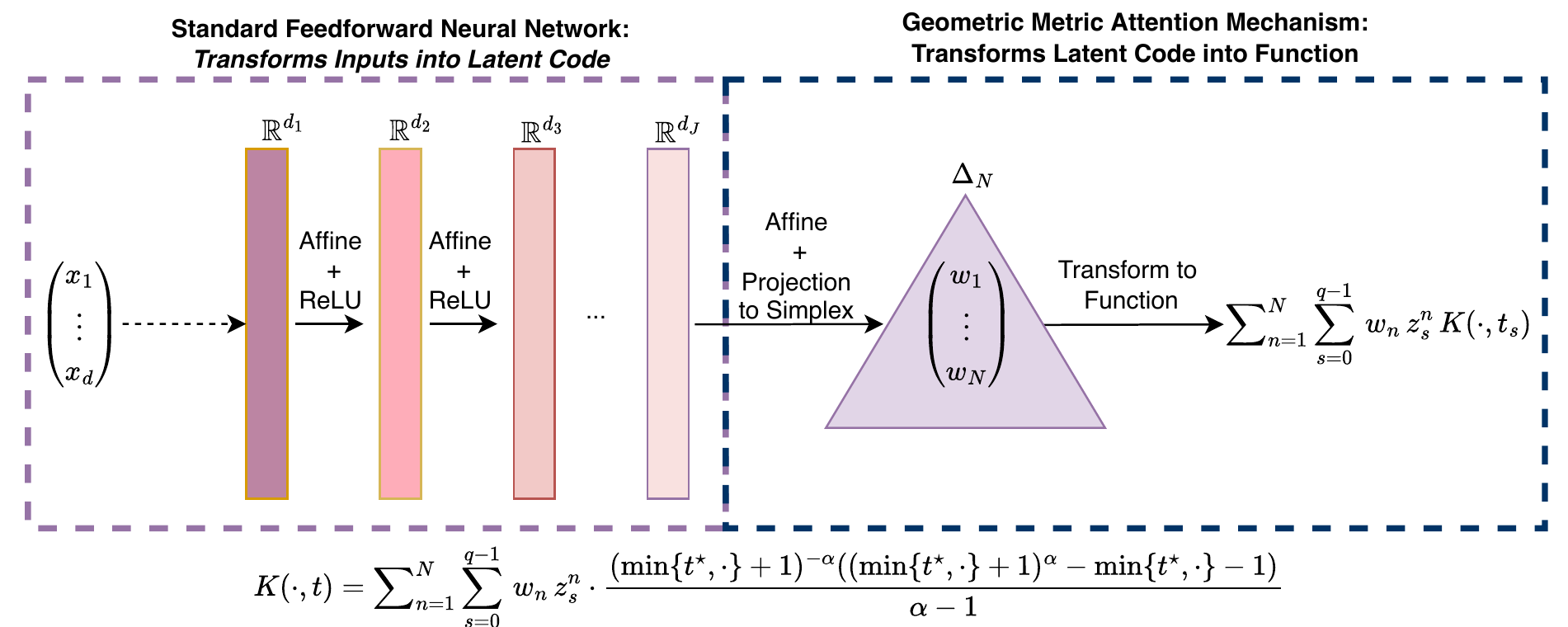}
    \caption{The \textit{geometric transformer} of Example~\ref{ex_Filipovic_Space} maps inputs in $\mathbb{R}^d$ to \textit{forward rate curves in the Reproducing Kernel Hilbert space $H_{\alpha}$ of \cite{Damir_Thesis_2001}} which can be expressed as convex combinations of $H_{\alpha}$'s kernel function $K(\cdot,\cdot)$ with second component evaluated at specified times $t_0,\dots,t_{q-1}\ge 0$.}
    \label{fig:Transformer_Rd_Filipovic_Space}
 \end{figure}

\begin{example}[Transformers in a Space of Forward Rate Curves]
\label{ex_Filipovic_Space}
In \citep[Section 5.1]{Damir_Thesis_2001}, the author introduces a class of Hilbert spaces of functions on $[0,\infty)$ for modeling the term structure of interest rates, which are both economically meaningful (see \citep[Sections 4.2-4.3]{Damir_Thesis_2001}) and convenient to analyse within the HJM framework.   For instance, if $\alpha >3$, the author considers the space $\tilde{H}_{\alpha}$ of all \textit{absolutely continuous functions} $y:[0,\infty)\rightarrow \rr$
(representing yields curves) for which the norm $\|y\| \eqdef \langle y,y\rangle^{1/2}_{\alpha}$ is finite, where
\[
\langle y,\tilde{y}\rangle_{\alpha}
     \eqdef 
|y(0)\tilde{y}(0)|
    +
\int_0^{\infty} |y'(t)\tilde{y}(t)|\ |1+t|^{\alpha} dt,
\]
where $y'$ is a weak derivative of $y$ on $(0,\infty)$ and $\tilde{y}$ is some absolutely continuous function on $[0,\infty)$.  For simplicity of exposition, let us consider the subset $H_{\alpha}\subset \tilde{H}_{\alpha}$ consisting of all functions satisfying the boundary condition $y(0)=0$.  Then, \citep[Example 1.3]{SaitohSawano_2016_DetailedRKHSBook} shows that this is also a Reproducing Kernel Hilbert Space (RKHS) with reproducing kernel $K$ defined on any $(t,\tilde{t})\in [0,\infty)^2$ by
\[
K(t,\tilde{t})
     \eqdef 
\int_0^{\min\{t,\tilde{t}\}}
    \frac{1}{(1+t)^{\alpha}}
    dt
=
    \frac{(\min\{t,t'\}+1)^{-\alpha}((\min\{t,t'\}+1)^{\alpha} -\min\{t,t'\} -1)}{\alpha -1}
.
\]
Applying \citep[Moore-Aronzajins Theorem]{aronszajn1950theory}, we conclude that the span of $\{K(t,\cdot):\,t\in [0,\infty)\}$ is dense in $H_{\alpha}$.  Now, since $[0,\infty)$ is separable and $K(\cdot,\cdot)$ is continuous, then there is a countable subset $\{t^{\star}_n\}_{n\in \nn}\subseteq [0,\infty)$ for which
\[
y_n 
     \eqdef  
K(t_n^{\star},\cdot)
\]
is a Schauder basis of $H_{\alpha}$.  Thus, 
the geometric attention of Example~\ref{ex_Frechet} simplifies to
\[
\operatorname{attention}_{N,q}(u,(z^n)_{n=1}^N)
     \eqdef 
\operatorname{\sum}_{n=1}^N
\sum_{s=0}^{q-1}
    \Pi_{\Delta_N}(u)_n
    z_s^n
        \cdot 
    \frac{
        (\min\{t_n^{\star},\cdot\}+1)^{-\alpha}((\min\{t_n^{\star},\cdot\}+1)^{\alpha} -\min\{t_n^{\star},\cdot\} -1)
    }{
        \alpha -1
    }
,
\]
for some $\{t^{\star}_n\}_{n=0}^{\infty}$ in $[0,\infty)$.  
\end{example}

\begin{remark}[Other HJM-Type Frameworks via Example~\ref{ex_Filipovic_Space}]
\label{remark_General}
As an interesting future research direction, one can likely modify Example~\ref{ex_Filipovic_Space} to suit other HJM-type of models used in equity or credit markets \cite{carmona2007hjm} or stocks \cite{KallsenKruner_2015_HJMVol}.
\end{remark}
\begin{remark}[Example~\ref{ex_Filipovic_Space} Translates to Separable RKHSs]
\label{remark_RKHS}
Since RKHSs are a central object in machine learning, it is worth emphasizing that the analysis carried out in Example~\ref{ex_Filipovic_Space} translates, nearly identically, to geometric transformers mapping in any other separable RKHS.  The geometric transformer in that context is the the same as in Figure~\ref{fig:Transformer_Rd_Filipovic_Space} with only the kernel function swapped for the kernel of the new RKHs.
\end{remark}

\subsubsection{Optimal Transport Spaces} 
\label{s_Preliminaries__ss_Examples___sss_Wasserstein}
Our results also apply to many spaces from optimal transport theory, adapted optimal transport, and consequentially robust finance.  

\begin{figure}[H]
    \centering
    \includegraphics[width=1\textwidth]{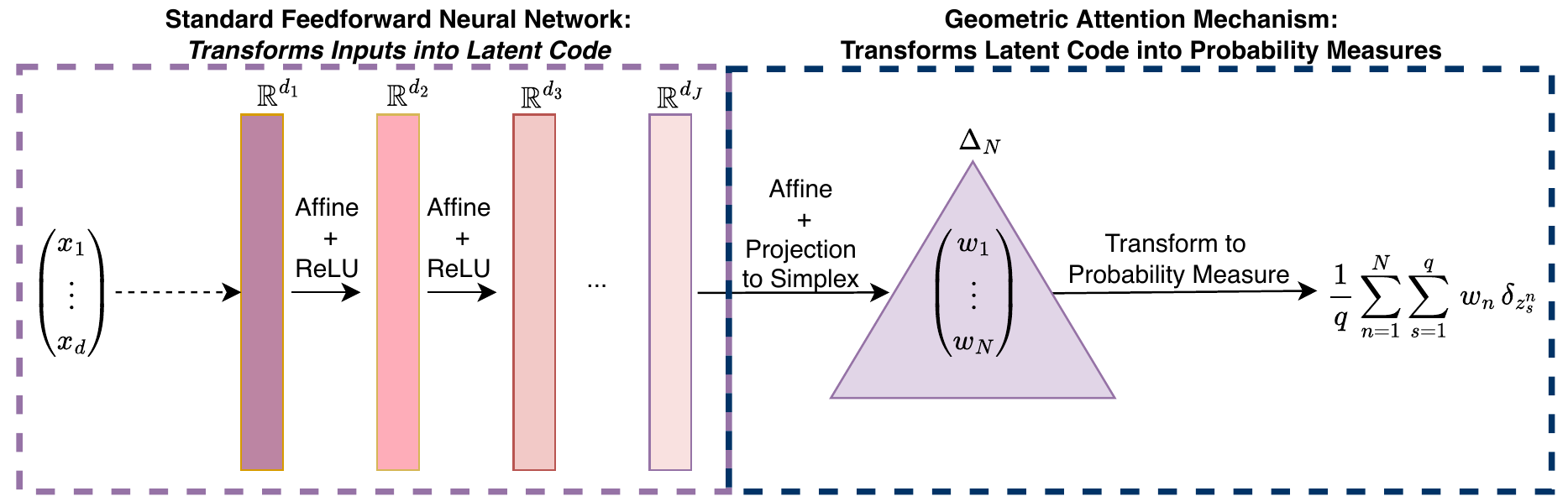}
    \caption{The \textit{geometric transformer} of Example~\ref{ex_wasserstein} transforms inputs in $\mathbb{R}^d$ to outputs in the Wasserstein space $(\mathcal{P}_q(\mathbb{R}^d),\mathcal{W}_p)$.  Every input in $\mathbb{R}^d$ is mapped to a weight $w$ in the $N$-simplex which is used to create a convex combination of $N$ empirical measures on $\mathbb{R}^d$, each of which charges at-most $q$ points with mass.}
    \label{fig:Transformer_Rd_wasserstein}
 \end{figure}

\begin{example}[Wasserstein Space with Convex Combinations]
\label{ex_wasserstein}
Fix $p\geq 1$, let $Z\subseteq \rr^d$ be closed, and either suppose that $Z=\rr^d$ and $q> p$, or that $Z$ is bounded and $q\geq p$. 
Let $(\yyy,d_{\yyy}) \eqdef (\mathcal{P}_q(Z),\mathcal{W}_p)$.  Define the mixing function $\eta$ to be the map that, for each $N\in \nn_+$,  sends any $w\in \Delta_N$ and any $(y_n)_{n=1}^N\in \yyy^N$ to
\begin{equation}
    \eta(w,(y_n)_{n=1}^N)
    	 \eqdef 
    \sum_{n=1}^N w_n y_n
\label{eq_ex_wasserstein__co_convex_combinations}
    .
\end{equation}
Under our assumptions on $Z,p,$ and $q$, we may apply \citep[Theorem 2 or Corollary 3]{Chevallier_2018_UnifDecompositionProbMEasures_JAP} (respectively) to conclude that $(\mathcal{P}_q(Z),\mathcal{W}_p)$ is quantized by the following family $(Q_q)_{q\in \nn_+}$ of functions:
\begin{equation}\label{eq:emp_meas}
Q_q: \rr^{d\times q}\ni 
    z=(z_0,\dots,z_{q-1})
        \mapsto 
    \frac1{q}
    \sum_{s=0}^{q-1}
        \,\delta_{z_s}
\in \mathcal{P}_q(Z).
\end{equation}
Therefore, $(\mathcal{P}_q(Z),\mathcal{W}_p)$ can be endowed with a QAS space structure whose geometric attention mechanism, coincides with the probabilistic attention mechanism of \cite{kratsios2021universal,AB_2021}, and is given by
\[
\operatorname{attention}_{N,q}(u,(z^n)_{n=1}^N)
     \eqdef 
    \frac1{q}
    \sum_{n=1}^N 
    \sum_{s=1}^q
    \Pi_{\Delta_N}(u)_n
        \delta_{z_s^n}
.
\]
\end{example}

\begin{figure}[ht]
    \centering
    \includegraphics[width=1\textwidth]{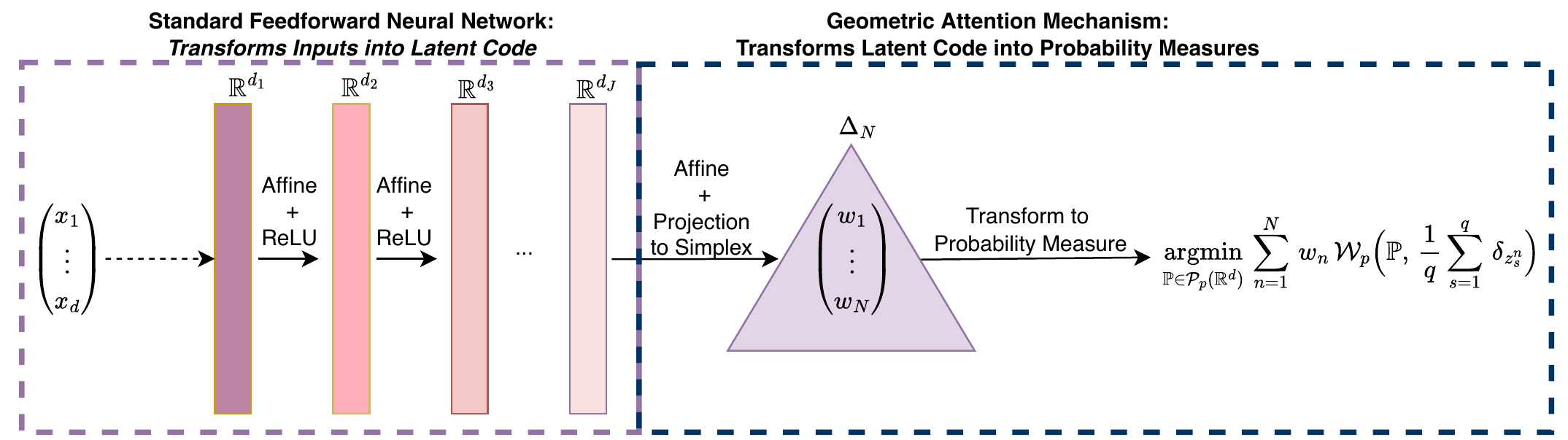}
    \caption{The \textit{geometric transformer} of Example~\ref{ex_Wasserstein_barycenter} transforms inputs in $\mathbb{R}^d$ to outputs in the Wasserstein space $(\mathcal{P}_p(\mathbb{R}^d),\mathcal{W}_p)$.  Every input in $\mathbb{R}^d$ is first mapped to a weight $w$ in the $N$-simplex.  The weight $w$ together with $N$ different empirical measures $\sum_{s=1}^q\,\delta_{z_s^1}$, $\dots$, $\sum_{s=1}^q\,\delta_{z_s^N}$ on $\mathbb{R}^d$ each of which charges at-most $q$ points with mass, defines a Wasserstein barycenter problem where the relative importance of each empirical measure is weighted according to $w$.  A probability measure optimizing this Wasserstein barycenter problem is output by the GT.}
    \label{fig:Transformer_Rd_wasserstein_McCann}
\end{figure}

Still considering the Wasserstein space, for the next example, we show how the mixing function $\eta$ can be chosen to be a \textit{non-linear averaging of distributions} unlike the linear mixing functions used to illustrate out theory thus far.  We consider the notion of Wasserstein barycenters, as introduced in \cite{AguehCarlier_WassersteinBarycenters}, which generalize McCann’s interpolation problem originally defined only for two probability measures. This has a wide range of applications, for example when one wants to average features defined as distributions (as in computer vision). We refer to \cite{cuturi2014fast} and \cite{claici2018stochastic} for algorithms to efficiently compute Wasserstein barycenters and to \cite{HeinemannMunkZemel_SIAM_SCIMODs_Guarantees_2022} for related statistical learning guarantees.
\begin{example}[Wasserstein barycenters]
\label{ex_Wasserstein_barycenter}
Let $(\yyy, d_\yyy) = (\mathcal P_p(\rrd), \mathcal W_p)$ be the set of measures with finite $p$-th moment on $\rrd$ equipped with the Wasserstein-$p$-distance.
For each $N \in \N$ and fixed $(y_n)_{n = 1}^N \in \mathcal P_p(\rrd)^N$, consider
\begin{align*}
    S(y,w)  \eqdef & \sum_{n = 1}^N w_n \mathcal W_p(y,y_n)^p, \\
    V(w)  \eqdef & \inf_{ y \in \mathcal P_p(\rrd)} S(y,w),
\end{align*}
and define $\eta(w,(y_n)_{n = 1}^N)$ as a measurable selection of optimizers of $V$.
To see the existence of such a selection, first note that $V$ is continuous.
Consequently,
\[
    B  \eqdef 
    \left\{
    (y,w) \in \mathcal P_p(\rrd) \times \Delta_N \colon
    S(y,w) = V(w)
    \right\}
\]
is a bounded subset of $\mathcal P_p(\rrd) \times \Delta_N$, and therefore relatively compact in $\mathcal P(\rrd) \times \Delta_N$.
By \cite[Remark 6.12]{OldandNew2010Villani}, $S$ is lower semicontinuous on $\mathcal P(\rrd) \times \Delta_N$, thus, $B$ is even compact in $\mathcal P(\rrd) \times \Delta_N$.
Therefore, we find, for $z \in \mathcal P_p(\rrd)$ and $r > 0$, that
\[
    \left\{ w \in \Delta_N \colon \exists (y,w) \in B \text{ s.t.\ }
    y \in
    \overline{\Ball_{(\mathcal P_p(\rrd), \mathcal W_p)}(z,r)}
    \right\}
\]
is a closed subset of $\Delta_N$.
Hence, \cite[Theorem 6.9.3]{Bogachev} to produces the desired measurable selection.

We remark that, for $p = 2$, by \cite{AguehCarlier_WassersteinBarycenters} the set of optimal couplings is a singleton if there is a $1 \le k \le N$ with $y_k$ absolutely continuous w.r.t.\ the Lebesgue measure.
Hence, in this case the (unique) selection is even continuous.
In any case, $\eta$ satisfies the following inequalities for all $i=1,...,N$:
\[
    \mathcal W_p\left(\eta(w,(y_n)_{n = 1}^N), y_i\right) 
    \le
    \sum_{j = 1}^N w_j \left( \mathcal W_p(y_j, y_i) + \mathcal W_p\left(\eta(w,(y_n)_{n = 1}^N), y_j\right)\right)
    \le 
    2 
    \left( 
        \sum_{j = 1}^N w_j \mathcal W_p(y_i, y_j)^p
    \right)^{1 / p}.
\]
\end{example}
Just as Example~\ref{ex_adapted_wasserstein} provided an adapted counterpart to Example~\ref{ex_wasserstein}, so does the following construction provide an adapted counterpart to Example~\ref{ex_Wasserstein_barycenter}.  The corresponding illustration is akin to Figure~\ref{fig:Transformer_Rd_wasserstein_McCann}, mutatis mundanis.

\begin{example}[Adapted Wasserstein barycenters]
\label{ex_adapted_wasserstein_barycenters}
    Even though $(\mathcal P_p(\mathbb R^{d T}), \mathcal{AW}_p)$ on its own is not a geodesic space, its completion $(\mathcal{FP}_p, \mathcal{AW}_p)$ is geodesically complete; see \cite{WassersteinSpaceofStochasticProcesses}, which also gives a probabilistic interpretation of the aforementioned space as a Wasserstein space of stochastic processes.
    Moreover, $(\mathcal{FP}_p, \mathcal{AW}_p)$ provides a way of taking barycenters of stochastic processes that take into account path properties as well as the arrow of time encoded in the underlying filtrations.
    Similar to Example \ref{ex_Wasserstein_barycenter}, for fixed $N \in \mathbb N$ and $(y_n)_{n = 1}^N$, we define the mixing function $\eta$ as a measurable selection of $\mathcal{AW}_p$-barycenters:
    \[
        \eta(w,(y_n)_{n = 1}^N)
        \in \arg \min
        \left\{
            \sum_{n = 1}^N w_n \mathcal{AW}_p(y, y_n)^p
            \colon
            y \in \yyy
        \right\}.
    \]
    The existence of $\eta$ can be shown along the lines of Example \ref{ex_Wasserstein_barycenter} due to the characterization of relative compact sets given in \cite[Theorem 1.7]{WassersteinSpaceofStochasticProcesses}.
\end{example}

\begin{remark}[Mixture Density Networks are Geometric Transformers]
\label{remark_guassian_mixtures}
One can modify the quantization function from the previous section to output Gaussian mixtures with parameterized non-degenerate variance.  Therefore, the mixture density network of \cite{bishop1994mixture} is a particular case of our geometric transformer framework. 
\end{remark}

\subsubsection{Finitely Parameterized Families of Probability Distributions from Information Geometry} 
\label{s_Preliminaries__ss_Examples___sss_Information Geometry}

\begin{figure}[ht]
    \centering
    \includegraphics[width=0.95\textwidth]{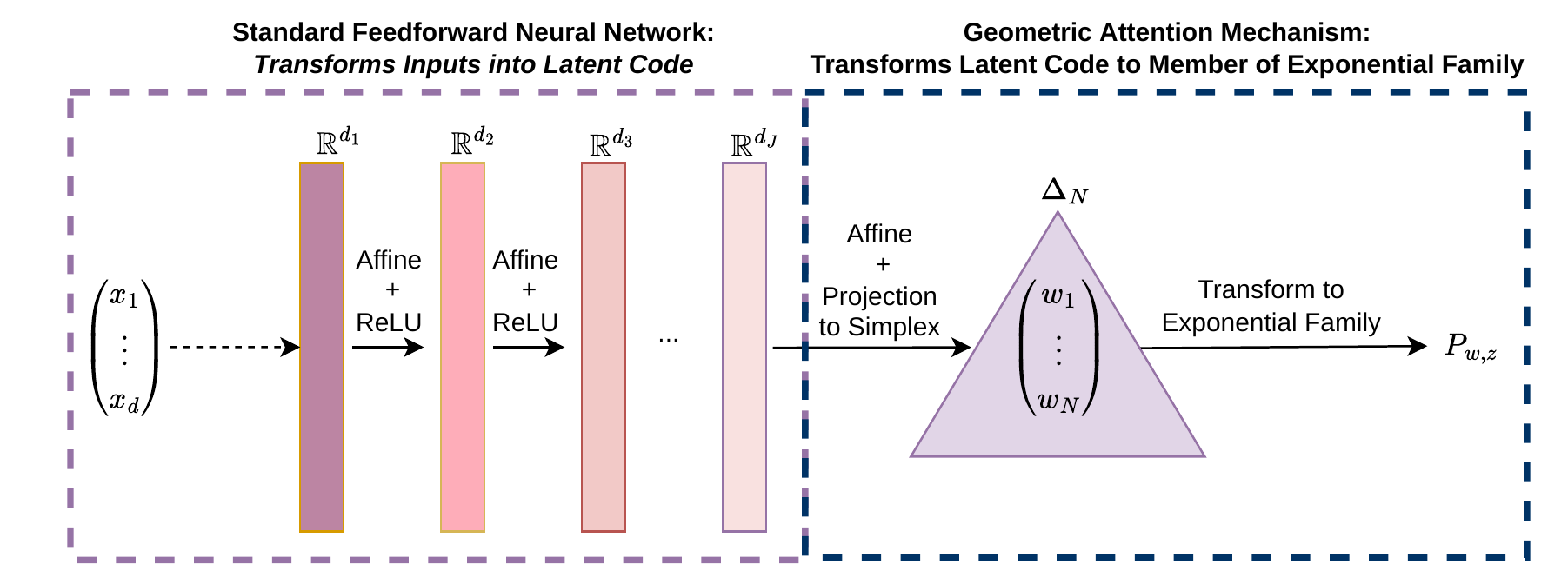}
    \\
    $\frac{P_{w,z}}{dx} \propto \exp\Big(-\sum_{n=1}^N \sum_{k=1}^d \,w_n\, P_{[0,1]^d}(z^n)_k\,F_k(\cdot)\Big)$
    \caption{The \textit{geometric transformer} of Example~\ref{ex_information_geometry} transforms inputs in $\mathbb{R}^d$ to the unique parameter associated to a probability measure in the natural parameter space $\Theta_{F_{\cdot}}$ associated to the statistical manifold $\mathcal{P}_{F_{\cdot}}$.  The model outputs the unique probability measure in $\mathcal{P}_{F_{\cdot}}$ corresponding to this parameter in $\Theta_{F_{\cdot}}$.}
    \label{fig:Transformer_Exponential}
 \end{figure}

At times, more can be assumed of the probability distributions describing the target process.  In particular, we are interested in the case where the process' law is characterized by a sufficient statistic and the conditional law of that process belongs to an exponential family of probability measures.  Typical examples arise when working with continuous-time and finite-state Markov chains (see \cite{Jacobsen_1982StatsofCountingProcesses,KuchlerSorensen_ExponentialFamiliesProcessesBook_1997}).  

This is not only of practical interest, but it provides a tangible class of non-linear examples, stemming in classical information geometry \citep{amari2016information}, where the mixing function $\eta$ is not a convex combination and where the maps $(Q_q)_{q\in \nn_+}$ are all identical due to $\yyy$'s finite dimensionality.  

\begin{example}[Finite-Dimensional Exponential Families of Stochastic Processes]
\label{ex_information_geometry}
Let $F_1,\dots,F_d:(\rr^n)^T\rightarrow \rr$ be linearly independent continuous path functionals and let %
$\Theta_{F_{\cdot}}$ be the open convex subset of $\rr^d$ defined by
\[
\Theta_{F_{\cdot}} 
     \eqdef 
\left\{
\theta \in \rr^d
    :
    \int_{x\in \rr^{nT}}
        \exp\left(
        -\sum_{k=1}^d \theta_k \, F_k(x)
        \right)
    dx
    <\infty
\right\}
.
\]
The set $\Theta_{F_{\cdot}}$, is known as the \textit{natural parameter space} of the statistical manifold
\[
\mathcal{P}_{F_{\cdot}}
     \eqdef 
\left\{
    y_{\theta} 
        \in 
    \mathcal{P}\left((\rr^n)^T\right)
:\,
(\exists \theta \in \Theta_{F_{\cdot}})\,
    \frac{dy_{\theta}}{
    dx%
    }
        \propto
    \exp\left(
        - \sum_{k=1}^d
            \theta_k
            F_k(\cdot)
    \right)
    \mbox{ and }
    \mathbb{E}_{X\sim y_{\theta}}\left[\|X\|\right]<\infty
\right\}
,
\]
where the identification of $\Theta_{F_{\cdot}}$ with $\mathcal{P}_{F_{\cdot}}$ is given by the correspondence $\Theta_{F_{\cdot}} \ni \theta \leftrightarrow y_{\theta} \in \pp_{F_{\cdot}}$; see \citep[Chapter 2.1]{amari2016information}.  Following \cite{Cencov1970s_reprintenglish1982,amari2016information}, the family of Fisher information matrices $(I(\theta))_{\theta \in \Theta_{F_{\cdot}}}$, defined at each $\theta\in \Theta_{F_{\cdot}}$ by $I(\theta) \eqdef 
\left(
    \mathbb{E}_{y_{\theta}}\left[
        \frac{\partial p_{\theta}}{\partial \theta_i}
        \frac{\partial p_{\theta}}{\partial \theta_j} 
    \right]
\right)_{i,j=1}^d
,
$
where $p_{\theta} \eqdef \frac{dy_{\theta}}{
dx%
}$, defines a Riemannian metric on $\Theta_{F_{\cdot}}$, called the \textit{Fisher-Rao metric}.
Appealing to the identification $\varphi:\Theta_{F_{\cdot}} \ni \theta \rightarrow y_{\theta} \in \pp_{F_{\cdot}}$, the Fisher-Rao metric induces an intrinsic metric on $\pp_{F_{\cdot}}$ via
\begin{equation}
\label{eq:FisherRaoDistance}
    d_I\left(y_{\theta^1},y_{\theta^2}\right)
     \eqdef 
    \inf_{\gamma}\, \int_0^1 
        \sqrt{
            \dot{\gamma}(t)^{\top} 
                I(\theta)
            \dot{\gamma}(t)
        }
    dt,
\end{equation}
where the infimum runs over all piecewise continuous $\gamma:[0,1]\rightarrow \Theta_{F_{\cdot}}$ with $\gamma(0)=\theta^1$ and $\gamma(1)=\theta^2$.  

For simplicity, as in \cite{FigalliNNs2021_Nature}, 
let us assume that $[0,1]^d\subset \Theta_{F_{\cdot}}$ and let us set $(\yyy,d_{\yyy}) \eqdef (\varphi([0,1]^d),d_I)$. Since $\varphi$ is a chart, it is smooth and therefore Lipschitz when restricted to $[0,1]^d$.  Therefore, we may define $\eta$ such that, for every $N\in \nn_+$, $w\in \Delta_N$ and $(y_{\theta^n})_{n=1}^N\in \yyy^N$,
\[
\frac{d \, \eta(w,(y_{\theta^n})_{n=1}^N)}{
dx
}
    \propto
\exp\left(
-\sum_{n=1}^N 
        \sum_{k=1}^d
            w_n
            \theta^n_k
            \,F_k(\cdot)
\right)
.
\]
Since $\Theta_{F_{\cdot}}$ is a $d$-dimensional Riemannian manifold, then, similarly to Example~\ref{ex_Frechet}, we may quantize it by the functions $(Q_q)_{q \in \nn}$ defined via
\[
\frac{dQ_q(z)}{
dx%
}
    \propto
    \exp\left(
        -\sum_{k=1}^d 
        \,[\Pi_{[0,1]^d}(z)]_k
        F_k(\cdot)
    \right),
\]
where $\Pi_{[0,1]^d}$ is the metric projection of $\rr^d$ onto the cube $[0,1]^d$, i.e.
$\rr^d \ni u \mapsto (\min\{\max\{0,u_i\},1\})_{i=1}^d \in [0,1]^d$.
Thus,
$(\mathcal{P}_{F_{\cdot}},d_I)$ can be endowed with QAS space structure.  For $q\geq d$, its associated geometric attention mechanism is
\[
\frac{d \, 
\operatorname{attention}_{N,q}(u,(z^n)_{n=1}^N)
}{
dx%
}
    \propto
\exp\left(
    -\sum_{n=1}^N 
        \sum_{k=1}^d
            [\Pi_{\Delta_N}(u)]_n
            \,[\Pi_{[0,1]^d}(z^n)]_k
            \,F_k(\cdot)
\right)
.
\]
\end{example}

We conclude our illustration of QAS spaces with an example relevant to Markovian SDEs and Gaussian processes.  

\begin{figure}[H]
    \centering
    \includegraphics[width=1\textwidth]{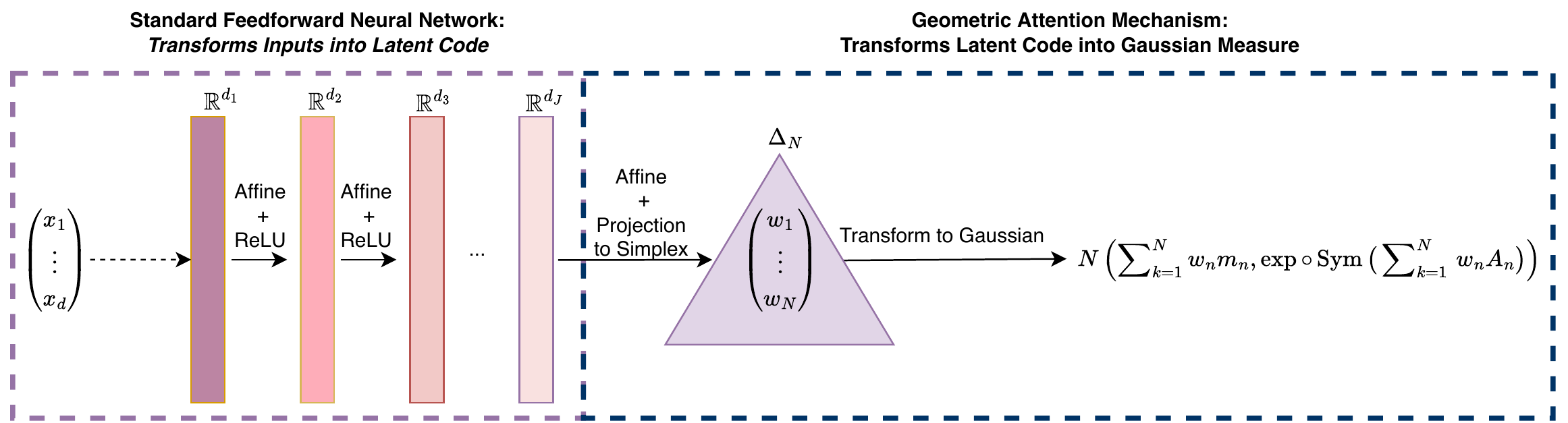}
    \caption{The \textit{geometric transformer} of Example~\ref{ex_Gaussian_measures} maps inputs in $\mathbb{R}^d$ to \textit{non-degenerate Gaussian} probability measures on $\mathbb{R}^d$.  The intrinsic distance function quantifying the dissimilarity between any such Gaussian measures is~\eqref{eq_RiemannianDistanceFunction}.}
    \label{fig:Transformer_Rd_Gaussians}
 \end{figure}

\begin{example}[Non-Degenerate Gaussian Measures]
\label{ex_Gaussian_measures}
Let $\yyy$ be the set of probability measures $N(m,\Sigma)$ on $\rr^d$ which are absolutely continuous with respect to the $d$-dimensional Lebesgue measure and whose Radon-Nikodym derivative can be written as
\[
\frac{dN(m,\Sigma)}{dx}
    =
(2\pi)^{-\frac{d}{2}}\det(\Sigma)^{-\frac{1}{2}} \, e^{ -\frac{1}{2}({x} - m)^{{{\top}}} \Sigma^{-1}({x} - m) }
,
\]
where ${\Sigma}$ is a $d\times d$ positive definite matrix and $m\in \rr^d$.  Accordingly, as in the information geometry literature \cite{GeometryOfMultivariateNormal_Lie_Canada,Nielson2020Flashcards,Wasserstein2Gaussian}, we identify non-degenerate Gaussian distributions in $\yyy$ with pairs $({\Sigma},{\mu})$ parameterizing them.  Instead of equipping $\yyy$ with a Riemannian structure, we may simply equip it with the following complete metric, arising from the non-positive curvature geometry of the set of $d\times d$-symmetric positive-definite matrices (see \cite{meyer2011regression,helgason1979differential}):
\begin{equation}
\label{eq_RiemannianDistanceFunction}
d_{\yyy}\left(
    N(m_1,\Sigma_1)
        ,
    N(m_2,\Sigma_2)
\right)
 \eqdef 
\sqrt{
    \left\|m_1-m_2\right\|^2
        +
    \left\|\sqrt{{\Sigma}_1}\log\left( {\sqrt{{\Sigma}_1}}^{-1} \sqrt{{\Sigma}_2} {\sqrt{{\Sigma}_1}}^{-1} \right)\sqrt{{\Sigma}_1}\right\|_F^2
}
,
\end{equation}
where $\sqrt{\cdot}$ denotes the square-root of a symmetric positive-definite matrix, $\log$ denotes the matrix logarithm, and $\|\cdot\|_F$ denotes the Frobenius norm.
Following \cite{kratsiosuniversal2021}, for each $q \in \nn$ we set
\[
Q_q:\rr^{d}\times \rr^{d(d+1)/2}\ni (m,A) \mapsto 
N(m,\operatorname{exp}\circ\operatorname{Sym}(A)),
\]
where $\operatorname{exp}$ is the exponential matrix and for any $A=(A_{1,1},\dots,A_{1,d},\dots,A_{d,d})\in \rr^{d(d+1)/2}$ the matrix $\operatorname{Sym}(A)$ is defined by $\operatorname{Sym}(A)_{i,j} \eqdef \operatorname{Sym}(A)_{j,i} \eqdef A_{i,j}$, $i,j=1,\dots,d$.
Thus, $(\yyy,d_{\yyy})$ can be equipped with a QAS space structure.  Its associated geometric attention mechanism is defined, the same way for each $q\in \nn_+$, by sending any $(u,(m_n,A_n)_{n=1}^N)$ in $\rr^{N}\times \rr^{N\times d + d(d+1)/2}$ to
\[
    \operatorname{attention}_{N,q}(u,(m_n,A_n)_{n=1}^N)
         \eqdef 
    N\left(
        \operatorname{\sum}_{k=1}^N
            \Pi_{\Delta_N}(u)_n
            m_n
    ,
        \exp\circ \operatorname{Sym}
        \left(\operatorname{\sum}_{k=1}^N
            \Pi_{\Delta_N}(u)_n
            A_n
        \right)
    \right)
.
\]
\end{example}
\subsection{Examples of Causal Maps of Approximable Complexity}
\label{s_Examples__ss_Examples_CausalSystems_wt_AC}
In this section, we present two examples of AC maps, which are not of finite complexity.  The first illustrates how an AC map can fail to be of finite complexity while exhibiting a limited memory, because the decoding map $\rho$ depends on arbitrarily many parameters.  The second example demonstrates an AC map where the converse is true.  Namely, it has infinite memory but its decoding maps depend on a fixed and finite number of latent parameters.

In the following example, we modify the setting of Example~\ref{ex_SDEs}, to show how an SDE's multi-step evolution defines an AC map into the adapted Wasserstein space. To keep technical details at a minimum, we consider a $1$-dimensional SDE.
\begin{example}[Discrete-time SDEs (Adapted Wasserstein space)]
\label{ex_SDEs_AW_simple}
Let $(\yyy, d_\yyy) = (\mathcal P_1(\rr^2), \mathcal{AW}_1)$.
Here, we assume that the H\"older coefficients of $\mu(t_n,\cdot) \in C^\alpha(\rr, \rr)$ and $\sigma(t_n,\cdot) \in C^\alpha(\rr,\rr)$ are uniformly bounded in $(t_n)_{n \in \mathbb{Z}}$, with constants $L_\mu$ and $L_\sigma$, respectively.
    For simplicity, we assume $\delta_-$ and $\delta_+$ in Assumption \ref{assumption_of_regular_gridmesh} to be  $\delta_- = \delta_+ = 1$.
    For initial condition $X_{t_n} = x_{t_n} \in \rr$, we define $X_{t_{n+1}}$ by
    \[
        X_{t_{n+1}} = x_{t_n} + \mu(t_n, x_{t_n}) + \sigma(t_n,x_{t_n}) W_n,
    \]  
    where $(W_n)_{n \in \mathbb Z}$ is a sequence of independent standard Gaussians.
    We consider the causal map that is, for $n \in \zz$ and $\xb \in \rr^\zz$, given by
    \[
        F(\xb)_{t_n}  \eqdef  
            \Law(X_{t_{n+1}}, X_{t_{n + 2}} | X_{t_{n}} = x_{t_n}),
    \]
    and proceed to show that it has approximable complexity.
    For this reason, let $\epsilon > 0 $ and $\mathcal K \subseteq \rr^\zz$ be compact.
    We define, for $n \in \nn$,
    \[
 c_{AC}(n,\epsilon) \eqdef \sup_{\xb \in \mathcal K} |\sigma(t_n,x_{t_n})|^ \alpha.
    \]
    For $m \in \nn$ sufficiently large, we choose an increasing sequence of quantiles $(q_k)_{k = 1}^m$ of a standard Gaussian $\gamma$, satisfying
    \[
        \int_\rr \sum_{k = 0}^m \mathbbm{1}_{I_k}(x) |x - q_k|^\alpha \, \gamma(dx) \le \frac{\epsilon}{L_\mu + L_\sigma},
    \]
    where we set $q_0 \eqdef  q_1, q_{m + 1}  \eqdef  \infty$, $I_k  \eqdef  (q_k,q_{k+1})$, $k = 1, \ldots, m$ and $I_0  \eqdef  (-\infty, q_1)$.

    To ease notation, we write
    \begin{gather*}
        q_k^n(x)  \eqdef  x + \mu(t_n,x) + \sigma(t_n, x) q_k,
        \\
        \mu^n(x)  \eqdef  \mu(t_n,x), \quad \mu^n_k(x)  \eqdef  \mu(t_n, q_k^{n-1}(x)),
        \\
        \sigma^n(x)  \eqdef  \sigma(t_n,x), \quad \sigma_k^n(x)  \eqdef  \sigma(t_n, q_k^{n-1}(x)).
    \end{gather*}
    Then, the encoding and decoding maps are given by
    \begin{align*}
        f_\epsilon(t_n , x)  \eqdef & 
        ( x + \mu^n(x), 
        \sigma^n(x), 
        (\mu^{n+1}_k(x) )_{k = 1}^m, 
        (\sigma^{n+1}_k(x) )_{k = 1}^m ),
        \\
        \rho_\epsilon( (\mu_k)_{k = 0}^m, (\sigma_k)_{k = 0}^m)
         \eqdef &
        N(\mu_0, \sigma_0^2)(dx_1) \sum_{k = 0}^m \mathbbm{1}_{\mu_0 + \sigma_0 I_k}(x_1) N(\mu_{k \vee 1}, \sigma_{k \vee 1}^2)(dx_2).
    \end{align*}
    By construction, we have that $f_\epsilon(t_n, \cdot) \in C^\alpha(\rr,\rr^{m + 1} \times \rr^{m + 1})$.
    Next we show that $\rho_\epsilon$ belongs to $C^1(\rr^{m + 1} \times \rr^{m + 1}, (\mathcal P_1(\rr^2),\mathcal{AW}_1))$. 
    We may estimate
    \begin{align*}
        \mathcal{AW}_1(\rho_\epsilon((\mu_k)_{k = 0}^m, (\sigma_k)_{k = 0}^m),
        \rho_\epsilon((\hat \mu_k)_{k = 0}^m, (\hat \sigma_k)_{k = 0}^m))
        &\le 
        \sum_{k = 0}^m \mathcal W_1(N(\mu_k,\sigma_k^2),N(\hat \mu_k,\hat \sigma_k^2))
        \\
        &\le
        \sum_{k = 0}^m |\mu_k - \hat \mu_k| + |\sigma_k - \hat \sigma_k| \, \mathbb E \left[ | W_n| \right].
    \end{align*}
    Finally, we verify that the causal map of finite complexity associated to $(f_\epsilon, \rho_\epsilon)$ satisfies \eqref{eq_approximation_sense}.
    For this reason, let $\gamma^n_{x_{t_n}}$ be the normal distribution with mean $\mu^n(x_{t_n})$ and variance $\sigma^n(x_{t_n})^2$.
    Thus, we get for $\xb \in \mathcal K$
    \allowdisplaybreaks
    \begin{align*}
        \mathcal{AW}_1
    &(F(\xb)_{t_n}, \rho_\epsilon(f_\epsilon(t_n, 
        x_{t_n})))
        \\
    &\le 
        \int_\rr \sum_{k = 0}^m \mathbbm{1}_{\mu^n(
        x_{t_n}
        ) + \sigma^n(
        x_{t_n}
       ) I_k}(x)
        \mathcal W_1\left(N(\mu^{n+1}(x), \sigma^{n+1}(x)), N(\mu_k^{n+1}(x), \sigma_k^{n+1}(x))\right) \, \gamma^n_{
        x_{t_n}
        }(dx)
        \\
    &\le 
        \int_\rr
        \sum_{k = 0}^m \mathbbm{1}_{I_k}(x)
        \Big(
        | \mu^{n + 1}(\mu^n(
        x_{t_n}
        ) + \sigma^n(
        x_{t_n}
        ) x) - \mu^{n+1}(\mu^n(
        x_{t_n}
        ) + \sigma^n(
        x_{t_n}
        ) q_k)|
        \\
    & \hspace{1cm}+
        | \sigma^{n + 1}(\mu^n(
        x_{t_n}
        ) + \sigma^n(
        x_{t_n}
        ) x) - \sigma^{n+1}(\mu^n(
        x_{t_n}
        ) + \sigma^n(
        x_{t_n}
        ) q_k)|
        \, \mathbb E [ | W_n | ] \Big) \, \gamma(dx)
        \\
    & \hspace{3cm} \le 
        \int_\rr
        \sum_{k = 0}^m \mathbbm{1}_{I_k}(x)
        |\sigma^n(
        x_{t_n}
        )|^\alpha | x - q_k|^\alpha (L_\mu + L_\sigma) \, \gamma(dx)
        \le 
        |\sigma^n(
        x_{t_n}
        )|^\alpha \epsilon.
        \end{align*}
\end{example}

Our next example of an AC map draws from the dynamical systems literature, such as \cite{HUTTERRecepHelmut2021_Ent_Optimal_RNNs_LDSs}.  
In particular, the example shows how the approximation of stochastic processes can equally be framed as a causal map into an $L^1$-space rather than a Wasserstein space.  This illustrates the ``modularity'' of our framework as well as the ability to model stochastic processes with genuinely infinite memory.  

\begin{example}[Infinite-Memory SDEs Inducing AC Maps]
\label{ex_SDEs_W1_infinite_memory}
Let $\alpha \in (0,1$], $(\Omega,\mathcal{F},\pp,\mathbb{F} \eqdef (\mathcal{F}_{t_n})_{n\in \nn})$ be a filtered probability space, and $(W_{t_n})_{n\in \nn}$ be a sequence of i.i.d.\ standard $d$-dimensional Gaussian random variables.
For simplicity, we assume that $\delta_-=\delta_+=1$ where, $\delta_+$ and $\delta_-$ are as in Assumption \ref{assumption_of_regular_gridmesh}.  Let $M\in C^{\alpha}([0,\infty)\times \rr^d,[0,1]^d)$, $\Sigma \in C^{\alpha}([0,\infty)\times \rr^d,[0,1]^{d\times d})$, and let $(k_n)_{n\in \zz_-}$ be an absolutely summable real-valued sequence.
Define the map
\[
\mu:
[0,\infty)\times (\rr^d)^{\zz_-}\ni 
\left(t,(z_n)_{n\in \zz_-}\right)
    \mapsto 
\sum_{-\infty<n\leq 0}
    k_n M(t,z_n)
.
\]
Then the causal map $F$ given by
\[
F(\xb)_{t_n}%
     \eqdef 
        N
        \left(
        x_{t_n}
            +
        \mu(t_n,(x_{t_s})_{s\leq n})
                ,
            \Sigma(t_n,x_{t_n})
        \right)
    \in \mathcal{P}_1(\rr^d),
\]
for $\xb \in (\rr^d)^\zz$ and $n \in \zz$, has ``infinite memory''.
This causal map has approximable complexity since the encoding and decoding maps of Definition~\ref{defn_compression_memory_property} can be taken to be
\[
f_{\epsilon}(t,(x_{-m(\epsilon)},\dots,x_0)) 
     \eqdef  
\Big(
        x_0
            + 
        \sum_{-m(\epsilon)\leq n\leq 0} \,k_n M(t, x_n)
    ,
        \Sigma(t,x_0)
\Big)
    \mbox{ and }
\rho_{\epsilon}(\mu,\Sigma)
         \eqdef 
    N(\mu ,\Sigma\Sigma^{\top})
    ,
\]
for $m(\epsilon)\in \nn$ large enough.
\end{example}

\section{Discussion}
\label{s_Discussion}
In the current manuscript, we considered the case where $\mathscr{X}$ is a compact subset of a Euclidean case.  Indeed, one can consider more general ``finite dimensional'' input spaces by precomposing our geometric transformer and geometric hypertransformer, architectures with \textit{feature maps}, which map inputs from more general metric spaces into Euclidean space.  Moreover, by \cite{kratsios2021universal} any such continuous feature map is ```suitable'', meaning that it preserves a model's universal approximation property upon pre-composition, if and only if it is injective.  

The situation complicates whenever $\mathscr{X}$ is no longer assumed to be a compact subset of Euclidean space, and instead, it is assumed to be a compact subset of an infinite-dimensional Banach space $B$.  This is because our proof techniques rely on the extendability of H\"{o}lder functions defined on $\mathscr{X}$ into a QAS space $(\yyy,d_{\yyy},\eta,Q)$ to H\"{o}lder functions defined on all of $B$ and that the H\"{o}lder constant of this extended function can be controlled in terms of the original map.  This is an active area of contemporary analytic research which is subtle, even in seemingly familiar cases, such as when $\mathscr{X}$ and $\mathscr{Y}$ are both Lebesgue spaces on $\mathbb{R}$ (see \cite{naor2001phase}).  

In the dynamic case, if we replace the uniform topology on the Fr\'{e}chet space $ (\mathbb{R}^d)^{\mathbb{Z}}$ with a weighted TVS topology in the sense of \cite{prolla1971bishop,schmocker2022universal}, then it may be possible to derive qualitative universal approximation theorems which are valid on a broader family of compact subsets of the path-space $(\rr^d)^{\mathbb{Z}}$ with this alternative topology. 
We would like to explore this direction and the possible connection to the compression function $c_{AC}$ in Definition~\ref{def_AdaptedMap_FiniteComplexity} in future work. 


\section{Proofs}
\label{s_Proofs}
This final section contains proofs of the paper's main results. We start by recalling concepts and introducing notations that we will use throughout the section.

Let $(\mathcal X,d)$ be a metric space.  We will often make use of a homeomorphism of a metric space to itself, which ``makes any $\alpha$-H\"{o}lder function Lipschitz'' and, following \citep[Section 2.6]{Weaver2018LipschitzAlgebras}, is defined as follows. For any $0<\alpha\leq 1$ we define the \textit{$\alpha$-snowflake} $(\mathcal X,d^{\alpha})$ of $(\mathcal X,d)$ to be the metric space whose underlying set is $\mathcal X$ equipped with the metric $d^{\alpha}$ given by 
\[
    d^{\alpha}(x_1,x_2) \eqdef  (d(x_1,x_2))^{\alpha}, \qquad x_1, x_2 \in \mathcal X.
\]
The metric ball in $(\mathcal X,d)$ of radius $r > 0$ at $x \in \mathcal X$ is denoted by $\Ball_{(\mathcal X,d)}(x,r)  \eqdef  \{ z \in \mathcal X \colon d(x,z) < r \}$.
The metric space $(\mathcal X,d)$ is called \textit{doubling}, if there is $C\in \nn_+$ for which every metric ball in $(\mathcal X,d)$ can be covered by at most $C$ metric balls of half its radius.
The smallest such constant is called $(\mathcal X,d)$'s \textit{doubling number}, and is here denoted by $C_{(\mathcal X,d)}$, see \citep[Section 10.13]{HeinonenClassicBook_AnalysisinMet2001} for further details.
We denote the Euclidean distance on $\rr^n$ by $d_n$, and we write $d_n^\alpha$ for the metric of the $\alpha$-snowflake $(\R^n, d_n^\alpha)$ of $(\R^n, d_n)$.

We also recall that, for any $\alpha>0$, the $\alpha$-H\"{o}lder norm of a function $g:[0,1]^n\to\rr^m$ is defined as
\[
\|g\|_{\alpha}  \eqdef  \sup_{x\in [0,1]^n} \|g(x)\| + \sup_{x_1,x_2\in [0,1]^n,\, x_1\neq x_2} \frac{\|g(x_1) - g(x_2)\|}{\|x_1-x_2\|^\alpha}.
\]
When $\|g\|_{\alpha}<\infty$, we write $g\in C^{\alpha}([0,1]^n,\rr^m)$.
\subsection{Proofs for the Static Case}
\label{s_Proofs__ss_StaticCase}

\begin{lemma}[Doubling Number of Snowflakes and Covering Number]
    \label{lem:DoublingOfSnowFlakesAndCoveringNumber}
    Let $(\mathcal X, d)$ be a doubling metric space with doubling number $C_{(\mathcal X, d)}$.
    Then, for $\alpha \in (0,1]$, the doubling number $C_{(\mathcal X, d^\alpha)}$ of its $\alpha$-snowflake $(\mathcal X, d^\alpha)$ is bounded as
    \begin{equation}
        \label{eq:DoublingConstantKSnowflakeToNormal}
        C_{(\mathcal X, d^\alpha)} \le C_{(\mathcal X, d)}^{ \lceil 1 / \alpha \rceil}.
    \end{equation}
    Moreover, if $\mathcal X$ has finite diameter, 
    then, for any $\delta > 0$, there exist $N  \eqdef  C_{(\xxx,d_{\xxx})}^{\lceil \log_2(\operatorname{diam}(\mathcal X)) - \log_2(\delta) \rceil}$ balls of radius $\delta$ covering $\mathcal X$, i.e., there are $x_i \in \mathcal X$ for $i = 1,\ldots, N$ such that
    \begin{equation}
        \mathcal X = \bigcup_{i = 1}^N \Ball_{(\mathcal X,d)}(x_i, \delta).
    \end{equation}
\end{lemma}

\begin{proof}
    Since $(\mathcal X, d)$ is doubling, there are, for each $x \in \mathcal X$ and $r > 0$, $x_i \in K$, $i = 1, \ldots, C_d$ such that
    \[
        \Ball_{(\mathcal X,d)}\left(x,r\right) \subseteq \bigcup_{i = 1}^{C_{(\mathcal X, d)}} \Ball_{(\mathcal X,d)}\left(x_i,r / 2\right).
    \]
    Iteratively applying this reasoning $j$-times for $j \in \N$, we find
    \begin{equation}
        \label{eq:DoublingConstantKRnBalls}
        \Ball_{(\mathcal X, d)}(x,r) \subseteq \bigcup_{i_1,\ldots, i_j = 1}^{C_{(\mathcal X, d)}} \Ball_{(\mathcal X,d)}\left(x_{i_1,\ldots,i_j},r / 2^j\right),
    \end{equation}
    where $x_{i_1,\ldots,i_j}$ are elements of $\mathcal X$.
    
    To see the first statement, note that
    \[
        \Ball_{(\mathcal X, d)}(x, r) = \Ball_{(\mathcal X, d^\alpha)}(x, r^\alpha), 
        \text{ thus }
        \Ball_{(\mathcal X, d)}(x, r / 2^{1 / \alpha}) = \Ball_{(\mathcal X, d^\alpha)}(x, r^\alpha / 2).
    \]
    For this reason, we choose $j = \lceil 1 / \alpha \rceil$ and derive from \eqref{eq:DoublingConstantKRnBalls} that a ball in $(\mathcal X,d^\alpha)$ can be covered by $C_{(\mathcal X,d)}^j$  balls half the radius, which yields \eqref{eq:DoublingConstantKSnowflakeToNormal}.
    
    To see the second statement, note that $\mathcal X \subset \Ball_{(\mathcal X, d)}(x, \operatorname{diam}(\mathcal X))$ for any $x \in \mathcal X$.
    Choose $j = \lceil \log_2( \operatorname{diam}(\mathcal X)) - \log_2(\delta) \rceil$ and $ r = \operatorname{diam}(\mathcal X)$.
    We have $r \le 2^j \delta$ and therefore, by \eqref{eq:DoublingConstantKRn}, $\mathcal X$ can be covered by $C_d^j = N$ balls of radius $\delta$.
\end{proof}

\begin{lemma}[Extension of H\"{o}lder Functions on Compact Subsets of Euclidean Space]
    \label{lem:HolderExtensionRnRm}
    Let $K\subset \R^n$ be compact, $\alpha \in (0,1]$, and $f \in C^\alpha(K,\R^m)$ with \Holderseminorm{} $L_f$.
    Then there exists an extension $F \in C^\alpha(\R^n, \R^m)$ of $f$ with constant $L_F$, such that
    \begin{equation} \label{eq:HolderExtensionRnRm}
        L_F \le c \cdot   \lceil \alpha^{-1} \rceil \log_2\left( \kappa_K(5^{-1}) \right) L_f,
    \end{equation}
    for some universal constant $c > 0$.
\end{lemma}

\begin{proof}

    Since the identity $(\R^n,d_n)\ni x \mapsto x\in (\R^n,d_n^{\alpha})$ is a quasisymmetry \citep[page 78]{HeinonenClassicBook_AnalysisinMet2001}, it is by definition a homeomorphism.
    Thus, $K$ is closed in $(\R^n,d_n^{\alpha})$ if and only if it is closed in $(\R^n,d_n)$.
    Moreover, a function $f$ is an element of $C^\alpha(K,\R^m)$ if and only if $f \in C^1((K, d_n^\alpha), (\R^m, d_m))$, since
    \[
        \lVert f(x_1) - f(x_2) \rVert \le L_f \lVert x_1 - x_2 \rVert^\alpha = L_f d_n^\alpha(x_1, x_2)\quad \text{for all $x_1, x_2 \in K$.}
    \]
    In order to find an extension $F$ of $f$ with \Holderseminorm{} bounded by \eqref{eq:HolderExtensionRnRm}, we want to apply \citep[Theorem 4.1]{BRUE2021LipExtRP}.
    It remains to show that the doubling number
    $C_{(K,d_n^\alpha)}$ of $(K,d_n^\alpha)$ is bounded as
    \begin{equation}
        \label{eq:DoublingConstantKRn}
        C_{(K,d_n^\alpha)} \le \kappa_K(5^{-1})^{\lceil \alpha^{-1} \rceil}.
    \end{equation}
    For this reason, we relate $C_{(K,d_n^\alpha)}$ with the doubling number $C_{(K,d_n)}$ of $(K, d_n)$.
    By \citep[Proposition 1.7 (i)]{BRUE2021LipExtRP} we have the bound
    \begin{equation}
        \label{PROOF_prop_universal_tree__eq_doubling_constant_bound}
        C_{(K,d_n)}\leq \kappa_K(5^{-1}).
    \end{equation}
    Now, by Lemma \ref{lem:DoublingOfSnowFlakesAndCoveringNumber} and combining \eqref{PROOF_prop_universal_tree__eq_doubling_constant_bound} with \eqref{eq:DoublingConstantKSnowflakeToNormal}, we obtain \eqref{eq:DoublingConstantKRn}.
    Therefore, we can apply \citep[Theorem 4.1]{BRUE2021LipExtRP} and find an extension $F \in C^1((\R^n,d_n^\alpha),(\R^m,d_m))$ of $f$ with Lipschitz constant $L_F$ satisfying \eqref{eq:HolderExtensionRnRm}.
    Clearly, $F \in C^\alpha(\R^n,\R^m)$ with \Holderseminorm{} $L_F$ which completes the proof.
\end{proof}

\begin{proof}[{Proof of Proposition~\ref{prop_universal_approximation_improved_rates}}]
If $f(x)=c$ for some 
constant $c>0$, 
then the statement holds with the neural network $\hat{f}(x)=c$, which can be represented as in \eqref{eq_definition_ffNNrepresentation_function}
with $[d]=(n,m)$, where $A^j$ is the $0$ matrix for all $j$, and the ''$c$'' in \eqref{eq_definition_ffNNrepresentation_function} is taken to be this constant $c$.
Therefore, we henceforth only need to consider the case where $f$ is not constant.
Let us observe that, if we pick some $x^{\star}\in K$, then for any multi-index $[d]$ and any neural network $\hat{f}_{\theta}\in \NN[[d]]$, $\hat{f}_{\theta}(x)-f(x^{\star})\in \NN[[d]]$, since $\NN[[d]]$ is invariant to post-composition by affine functions.  
Thus, we represent $\hat{f}_{\theta}(x)-f(x^{\star})=\hat{f}_{\theta^{\star}}(x)$, for some $\theta^{\star}\in \rr^{P([d])}$.  Consequently:
\[
    \sup_{x\in K}
    \left|
    \|
    (f(x)-f(x^{\star})) - \hat{f}_{\theta^{\star}}(x)
    \|
    -
    \|
    f(x) - \hat{f}_{\theta}(x)
    \|
    \right|
    =0
    .
\]
Therefore, without loss of generality, we assume that 
$f(x^*)=0$ for some $x^*\in K$.
By Lemma \ref{lem:HolderExtensionRnRm} we can extend $f$ to $F \in C^{\alpha}(\R^n,\R^m)$ with \Holderseminorm{} bounded by \eqref{eq:HolderExtensionRnRm}.

\hfill\\
\textbf{Step 1 -- Normalizing $\tilde{f}$ to the Unit Cube:}
    First, we identify a hypercube ``nestling'' $K$.  To this end, let
    \begin{equation}
        \label{PROOF_prop_universal_tree__good_cube_via_Jungs_Radius}
        r_K  \eqdef  \operatorname{diam}(K) \sqrt{ \frac{n}{2(n+1)} }.
    \end{equation}
    By Jung's Theorem (see \cite{Jung1901}), there exists $x_0\in \rr^n$ such that the closed Euclidean ball $\overline{\Ball_{(\R^n,d_n)}\left(x_0,r_K\right)}$ contains $K$.
    Therefore, by H\"{o}lder's inequality, we have that the $n$-dimensional hypercube\footnote{For $x,y \in \R^n$ we we denote by $[x,y]$ the hypercube defined by $\prod_{i = 1}^n [x_i, y_i]$. } $[x_0-r_K\bar{1},x_0+r_K\bar{1}]$ contains $\overline{ \Ball_{({\rr^n},d_n)}\left( x_0,r_K \right) }$, where $\bar{1}=(1,\dots,1)\in \rr^n$.
    Consequently, $K\subseteq [x_0-r_K\bar{1},x_0+r_K\bar{1}]$.  
    Let $\tilde{f} \eqdef  F|_{[x_0-r_K\bar{1},x_0+r_K\bar{1}]}$, 
    then $\tilde f \in C^\alpha([x_0 - r_K \bar 1, x_0 + r_K \bar 1], \R^m)$ is an extension of $f$ with \Holderseminorm{} $L_{\tilde f}$ bounded by \eqref{eq:HolderExtensionRnRm}.

Since $K$ has at least two points, then $r_K>0$.  Hence, the affine function
\[
T:\rr^n \ni x \mapsto 
(2r_K)^{-1}(x-x_0+r_K\bar{1}) \in \rr^n
\]
is well-defined, invertible,  not identically $0$, and maps $[x_0-r_K\bar{1},x_0-r_K\bar{1}]$ to $[0,1]^n$.
Note that the $\alpha$-H\"{o}lder norm of $g  \eqdef  \tilde f \circ T^{-1}$ is finite, as $g$ is $\alpha$-H\"older continuous with constant $L_g =(2r_K)^{\alpha} L_{\tilde f}$.
More explicitly, writing $u^\ast  \eqdef  T( x^\ast) \in [0,1]^n$, we have $g(u^\ast) = 0$ and find 
\begin{align}
\nonumber
    \lVert g \rVert_\alpha 
    &= \sup_{u \in [0,1]^n} \lVert g(u) \rVert + \sup_{u_1, u_2 \in [0,1]^n, u_1 \neq u_2} \frac{\lVert g(u_1) - g(u_2) \rVert}{\lVert u_1 - u_2\rVert^\alpha} \\ \nonumber
    &=
    \sup_{u \in [0,1]^n} \lVert g(u) - g(u^\ast) \rVert + L_g \le \sup_{u_1, u_2 \in [0,1]^n} \lVert g(u_1) - g(u_2) \rVert + L_g  \le 2 L_g\\ 
    & \phantom{=\sup_{u \in [0,1]^n} \lVert g(u) - g(u^\ast) \rVert + L_g} 
    \le 
    c \, 2^{\alpha + 1} r_K^\alpha \left( \lceil \alpha^{-1} \rceil \log_2\left( \kappa_K(5^{-1}) \right) \right) L_f
    \label{eq:gHolderNorm}
    ,
\end{align}
where we used that $L_{\tilde f}$ is bounded by \eqref{eq:HolderExtensionRnRm}.
We define $\tilde g  \eqdef  \lVert g \rVert_\alpha^{-1} g$, and get that, for each $i=1,\dots,m$, the function $\tilde g^{(i)}  \eqdef  \operatorname{pj}_i \circ \tilde g$ belongs to the unit ball of $C^{\alpha}([0,1]^n,\rr)$, where, for $i=1,\dots,m$, $\operatorname{pj}_i$ denotes the canonical projection $\operatorname{pj}_i:\rr^m\ni (x_1,\dots,x_m)\mapsto x_i \in \rr$.

\hfill\\
\textbf{Step 2 -- Constructing the Approximator:}
For $i=1,\dots,m$, let $\hat f_{\theta^{(i)}}\in \NN[[d^{(i)}]]$ for some multi-index $[d^{(i)}]=(d_0^{(i)},\dots,d_{J^{(i)} +1 })$ with $n$-dimensional input layer and $1$-dimensional output layer, i.e.\ $d_0^{(i)}=n$ and $d_{J^{(i)} +1 } =1$, and let $\theta^{(i)} \in \rr^{P([d^{(i)}])}$ be the parameters defining $\hat f_{\theta^{(i)}}$.
Since the pre-composition by affine functions and the post-composition by linear functions of neural networks in $\NN[[d^{(i)}]]$ are again neural networks in $\NN[[d^{(i)}]]$, we have that $g_{\theta^{(i)}}  \eqdef  \hat f_{\theta^{(i)}} \circ T^{-1}$ and $\tilde g_{\theta^{(i)}}  \eqdef  \lVert g \rVert_\alpha^{-1} g_{\theta^{(i)}}$ are neural networks in $\NN[[d^{(i)}]]$.
Note that due to bijectivity of the maps $T$ and $y \mapsto \lVert g \rVert_\alpha \, y$, the correspondence $\hat f_{\theta^{(i)}} \mapsto \tilde g_{\theta^{(i)}}$ is one-to-one.
Denote the standard basis of $\rr^m$ by $\{e_i\}_{i=1}^m$.
We compute
\allowdisplaybreaks
\begin{align}
    \nonumber
    \sup_{x\in K}\, 
    \left\|f(x)-\sum_{i=1}^m \hat{f}_{\theta^{(i)}}(x)e_i\right\|
        = &
    \sup_{x\in K}\, 
    \left\|\tilde{f}(x)
        -
    \sum_{i=1}^m \hat{f}_{\theta^{(i)}}(x)e_i\right\|
        \leq
    \sup_{x\in [x_0-r_K\bar{1},x_0+r_K\bar{1}]}\, 
    \left\|\tilde{f}(x)
        -
    \sum_{i=1}^m \hat{f}_{\theta^{(i)}}(x)e_i\right\|
\\
    \nonumber
    = &
    \sup_{x\in [x_0-r_K\bar{1},x_0+r_K\bar{1}]}\, 
    \left\|\tilde{f}\circ T^{-1}\circ T(x)
        -
    \sum_{i=1}^m \hat{f}_{\theta^{(i)}}\circ T^{-1}\circ T(x)e_i\right\|
\\
    \nonumber
        = &
    \sup_{u\in [0,1]^n}\, 
    \left\|g(u) - \sum_{i=1}^m g_{\theta^{(i)}}(u) e_i\right\|
        \leq
            \sqrt{m}
    \sup_{u\in [0,1]^n}\, 
    \max_{1\leq i\leq m}
    \left\|
        \operatorname{pj}_i\circ g(u)
            -
        g_{\theta^{(i)}} (u)
    \right\|
\\
        = &\,
            C_0
    \sup_{u\in [0,1]^n}\, 
    \max_{1\leq i\leq m}
    \left\| \tilde g^{(i)}(u) - \tilde g_{\theta^{(i)}}(u)
    \right\| 
    \label{PROOF_prop_universal_tree__main_estimate_Part_A}
    ,
\end{align}
where $C_0 \eqdef  \|g\|_{\alpha} \sqrt{m}$.
Since, for each $i=1,\dots,m$, $\tilde{g}^{(i)}$ belongs to the unit ball of $C(([0,1]^n, d_n^\alpha),\rr)$, 
for $\sigma$ as in Definition~\ref{defn_TrainableActivation_Singular} (resp.\ as in Definition~\ref{defn_TrainableActivation_Smooth}) we may apply \citep[Theorem 1]{ShenYangZhang2021} (resp.\ \citep[Proposition 59]{kratsiosuniversal2021})
to conclude that, for any $W,D\in \nn_+$ (resp.\ any $\tilde{\epsilon}>0$), and each $i=1,\dots,m$, there are $\hat f_{\hat{\theta}^{(i)}} \in \NN[[d^{(n)}]]$ where $[d^{(i)}]=(d_0^{(i)},\dots,d_{J^{(i)}  +1  }^{(i)})$
such that
\begin{equation}
\begin{cases}
    J^{(i)}
    \le
2^6n D + 3
 \mbox{ and }
\max_{1 \le j \le J^{(i)} +1} d_j{^{(i)}}
\le
\max\{n,5 W+13\} 
    &: \mbox{$\sigma$ of Definition~\ref{defn_TrainableActivation_Singular}}\\
J^{(i)}
    \in 
    \mathcal{O}\left(
  \tilde{\epsilon}^{-2n/\alpha}
  L_f^{2n/\alpha}
  (1+n/4)^{2n/\alpha}
\right)
    \mbox{ and }
    \max_{1 \le j \le J^{(i)} +1}d_j{^{(i)} } \le n + 3 %
    &: 
    \mbox{$\sigma$ of Definition~\ref{defn_TrainableActivation_Smooth}}
    \\
\end{cases}
\label{PROOF_prop_universal_tree__complexity_estimate_univariate}
\end{equation}
and
\begin{equation}
\begin{cases}
    \sup_{u\in [0,1]^n}\, 
    \max_{1\leq i\leq m}
    \|
        \tilde{g}^{(i)}(u)
            -
        \hat{f}_{\hat{\theta}^{(i)}}(u)
    \| 
    \leq 
    n^{\frac{\alpha}{2}} 
W^{-\alpha \sqrt{D}}
    +
2 n^{\frac{\alpha}{2}}
W^{-\sqrt{D}} 
    &: \mbox{$\sigma$ of Definition~\ref{defn_TrainableActivation_Singular}}\\
        \sup_{u\in [0,1]^n}\, 
        \max_{1\leq i\leq m}
        \|
            \tilde{g}^{(i)}(u)
                -
            \hat{f}_{\hat{\theta}^{(i)}}(u)
        \| 
        \leq 
        \tilde{\epsilon}
    &: 
    \mbox{$\sigma$ of Definition~\ref{defn_TrainableActivation_Smooth}}
    .
\end{cases}
\label{PROOF_prop_universal_tree__main_estimate_Part_B_application_of_Yarotskys_Result}
\end{equation}
Consequently, \eqref{PROOF_prop_universal_tree__main_estimate_Part_B_application_of_Yarotskys_Result} implies that
\allowdisplaybreaks
\begin{align}
\begin{cases}
     \sup_{x\in K}\, 
    \left\|f(x)
        -
    \sum_{i=1}^m
        \hat{f}_{\hat \theta^{(i)}}(x)
    e_i
    \right\|
        \leq 
            C_0
    n^{\frac{\alpha}{2}} W^{-\sqrt{D}}
    (W^{(1-\alpha) \sqrt{D}} +2)
    &: \mbox{$\sigma$ of Definition~\ref{defn_TrainableActivation_Singular}}\\
            \sup_{x\in K}\, 
    \left\|f(x)
        -
    \sum_{i=1}^m
        \hat{f}_{\hat \theta^{(i)}}(x)
    e_i
    \right\|
        \leq 
            C_0
        \tilde{\epsilon}
    &: 
    \mbox{$\sigma$ of Definition~\ref{defn_TrainableActivation_Smooth}}
    .
\end{cases}
    \label{PROOF_prop_universal_tree__main_estimate_Part_C_wrapping_up}
\end{align}

We express the constant $C_0$ in terms of the given $K$-dependent data by \eqref{PROOF_prop_universal_tree__good_cube_via_Jungs_Radius} and \eqref{eq:gHolderNorm}, 
thus obtaining 
\begin{equation}
C_0 
    \leq 
2^{1 + \alpha / 2}c \sqrt{m} \operatorname{diam}(K)^\alpha \left( \frac{n}{n + 1} \right)^{\alpha / 2} \left( \lceil \alpha \rceil^{-1} \log_2 \left( \kappa_K(5^{-1}) \right) \right)L_f
    \label{PROOF_prop_universal_tree__main_estimate_Part_F_Completed_Estimate}
    .
\end{equation}
Therefore, we introduce the universal constant
$\tilde C_0  \eqdef 2^{1 + \alpha / 2}\left(\frac{n}{n + 1} \right)^{\alpha / 2}c L_f$.
Rearranging and combining with~\eqref{PROOF_prop_universal_tree__complexity_estimate_univariate}, we find that
\begin{align}
    \begin{cases}
    \epsilon  \eqdef 
    \sqrt{m}
    n^{\frac{\alpha}{2}} W^{-\sqrt{D}}
    (W^{(1-\alpha) \sqrt{D}} +2)
    &: \mbox{$\sigma$ of Definition~\ref{defn_TrainableActivation_Singular}}\\
    \epsilon  \eqdef 
    \sqrt{m}\tilde{\epsilon}
    &: 
        \mbox{$\sigma$ of Definition~\ref{defn_TrainableActivation_Smooth}}
        .
\end{cases}
\end{align}
Observe that $P([d^{(i)}]) = 
    \sum_{j=0}^{J^{(i)}}d^{(i)}_{j +1 }(d_{j}^{(i)}+3) 
    -
    2d_{J^{(i)}+1}^{(i)}
\le 
    \sum_{j=0}^{J^{(i)}}\, d_{j+1}^{(i)}(d_{j}^{(i)}+3)
    \le (\operatorname{Depth}+1)\times (\operatorname{Width}+3)^2
    $, where we denote the network's depth and width respectively by $\operatorname{Depth}$ and $\operatorname{Width}$. 
    Thus, 
we deduce that the number of parameters defining each network is bounded-above by
\begin{align}
\begin{cases}
        P([d^{(i)}]) 
    \leq  
        \max\{n+3,5W+16\}^2
        (2^6n D 
       +4
        )
    &: \mbox{$\sigma$ of Definition~\ref{defn_TrainableActivation_Singular}}\\
        P([d^{(i)}]) 
    \leq  
        \mathcal{O}\left(
        (n + 6)^2
        \big(
            \tilde{\epsilon}^{-4n/\alpha}
            L_f^{4n/\alpha}
            (1+n/4)^{4n/\alpha}
      +1\big)
        \right)
    &: 
        \mbox{$\sigma$ of Definition~\ref{defn_TrainableActivation_Smooth}}
        .
\end{cases}
\label{PROOF_prop_universal_tree__main_estimate_Part_G_Rearranged}
\end{align}
\hfill\\
\textbf{Step 3 -- Counting Parameters:}
Let $g_1\bullet g_2$ 
denotes the component-wise composition of a univariate function $g_1$ with a multivariate function $g_2$.
By construction, for any $k\in \nn_+$, if $I_k$ denotes the $k\times k$-identity matrix, then
$I_k \sigma_{(1,1)}\bullet I_k \in \NN[[d_k]]$ with $P([d])=2k$, and $I_k \sigma_{(1,1)}\bullet I_k=1_{\rr^k}$.
Therefore, mutatis mutandis, $\NN[[\cdot]]$ satisfies \citep[Definition 4]{FlorianHighDimensional2021}; whence, mutatis mutandis, we may apply \citep[Proposition 5]{FlorianHighDimensional2021}.
Thus, there is a multi-index $[d]=(d_0,\dots,d_{J +1} )$ with $d_0=n$ and $d_J=m$, and a network $\hat{f}_{\theta}\in \NN[[d]]$ implementing $\sum_{i=1}^m \hat{f}_{\theta^{(i)}}$, i.e.
\[
    \sum_{i=1}^m \hat{f}_{\theta^{(i)}}e_i
        = 
    \hat{f}_{\theta},
\]
such that $\hat{f}_{\theta}$'s depth ($J$) is bounded above by
\begin{align}
 J 
    \le
\begin{cases}
    m(2^6n D 
        + 
    4
    )
    &: \mbox{$\sigma$ of Definition~\ref{defn_TrainableActivation_Singular}}\\
    \mathcal{O}\left(
    m\Big(  \tilde{\epsilon}^{-2n/\alpha}
              L_f^{2n/\alpha}
              (1+n/4)^{2n/\alpha}
             + 1
            \Big)
    \right)
    &: 
    \mbox{$\sigma$ of Definition~\ref{defn_TrainableActivation_Smooth}},
\end{cases} 
\end{align}
its width can be upper-bounded by
\begin{align}
\max_{0\leq j\leq J +1} \, d_j 
        \leq 
\begin{cases}
    n(m-1) 
        + 
    \max\{n,5W+13\}
    &: \mbox{$\sigma$ of Definition~\ref{defn_TrainableActivation_Singular}}\\
    nm+ 3
    &: 
    \mbox{$\sigma$ of Definition~\ref{defn_TrainableActivation_Smooth}},
\end{cases}
\end{align}
and it total number of parameters can be upper-bounded by
\begin{align}
        P([d])
        \leq &
    \left(
        \frac{11}{16} \, 2^2 \, n^2 m^2 - 1
    \right) 
    \sum_{i=1}^m P([d^{(i)}])
\label{PROOF_prop_universal_tree__paraneter_estimate_Flo_applied_pre}
.
\end{align}
Therefore, depending on the trainable activation function $\sigma$,~\eqref{PROOF_prop_universal_tree__paraneter_estimate_Flo_applied_pre} implies that
\begin{align}
P([d])
    \leq
    \begin{cases}
    \left(
        \frac{11}{4}n^2m^2 - 1
    \right) 
    m
    \max\{n+3,5W+16\}^2
    (2^6n D 
  +4 
    )
    &: \mbox{$\sigma$ of Definition~\ref{defn_TrainableActivation_Singular}}\\
    \mathcal{O}\left(
    \left( \frac{11}{4} n^2m^2 - 1\right) 
    m
    (n + 6)^2
       \big(
            \tilde{\epsilon}^{-4n/\alpha}
            L_f^{4n/\alpha}
            (1+n/4)^{4n/\alpha}
           + 1
        \big)
    \right)
    &:
    \mbox{$\sigma$ of Definition~\ref{defn_TrainableActivation_Smooth}}
    .
    \end{cases}
\label{PROOF_prop_universal_tree__paraneter_estimate_Flo_applied}
\end{align}
Combining~\eqref{PROOF_prop_universal_tree__paraneter_estimate_Flo_applied} with~\eqref{PROOF_prop_universal_tree__main_estimate_Part_G_Rearranged} yields our result for the cases where $\sigma$ is as in Definition~\ref{defn_TrainableActivation_Singular} and as in Definition~\ref{defn_TrainableActivation_Smooth}.
The case where $\sigma$ is as in Definition~\ref{defn_TrainableActivation_ClassicalNonTrainable} follows from \citep[Theorem 3.2]{kidger2021neural}. 
\end{proof}

\subsection{Proof of Theorem~\ref{thrm_main_StaticCase}}
\begin{proof}[{Proof of Theorem~\ref{thrm_main_StaticCase}}]
Let $\alpha \in (0,1]$ and fix $\epsilon_A, \epsilon_Q > 0$.
\hfill \\

\textbf{Step 1 -- The Random Projection:} 
Let $A \subset \xxx$ be closed and denote by $C_{(A, d^\alpha_\xxx)}$ the doubling number of the $\alpha$-snowflake of $A$, where $d_{\xxx}$ is the base metric on $A$ defining the Wasserstein distance on $\mathcal P_1(A)$.
By \cite[Theorem 3.2]{BRUE2021LipExtRP}, there exists a random projection $\Pi \colon \xxx \to \mathcal P_1(A)$ with Lipschitz constant $L_{\Pi}$ such that
\begin{equation}
    \label{PROOF_eq_thrm_measure_valued_quantitative_projection_Lipschitzbound1}
    L_{\Pi} \le c \cdot \log_2(C_{(A, d^\alpha_\xxx)}),
\end{equation}
where $c > 0$ is a universal constant.
To make the bound of $L_{\Pi}$ more explicit, we estimate $C_{(A, d^\alpha_\xxx)}$.
For this reason, consider $x \in A$ and $r > 0$.
We may externally cover $\Ball_{(A, d^\alpha_\xxx)}(x,r)$ by $C_{(\xxx, d^\alpha_\xxx)}^2$ balls of radius $r / 4$ centered at points in $\xxx$.
Consider a ball centered at $\tilde x \in \xxx$ with
\[
    \Ball_{(\xxx, d^\alpha_\xxx)}(\tilde x, r/4) \cap A \neq \emptyset,
\]
then we can find a ball of radius $r / 2$ centered in $A$ covering it.
Therefore, $\Ball_{(\xxx, d^\alpha_\xxx)}(x,r)$ may be internally covered by at most $C_{(\xxx,d^\alpha_\xxx)}^2$-balls of radius $r / 2$.
By inserting this into \eqref{PROOF_eq_thrm_measure_valued_quantitative_projection_Lipschitzbound1} and applying Lemma \ref{lem:DoublingOfSnowFlakesAndCoveringNumber}, we thus get
\begin{equation}
    \label{PROOF_eq_thrm_measure_valued_quantitative_projection_Lipschitzbound2}
    L_{\Pi} 
\le 
    2 c \cdot \log_2(C_{(\xxx, d^\alpha_\xxx)}) 
\le 
    2 c \lceil 1 / \alpha \rceil \cdot \log_2(C_{(\xxx, d_\xxx)}) 
=: 
    C_\Pi.
\end{equation}
Note that the right-hand side of \eqref{PROOF_eq_thrm_measure_valued_quantitative_projection_Lipschitzbound2} is independent of the choice of $A$, and we may assume from now on without loss of generality that $C_\Pi \ge 1$.
\par
\noindent

\textbf{Step 2 -- Estimating the External Covering Number of $\boldsymbol{f(\xxx)}$:}
By passing on to the snowflake $(\xxx, d^\alpha_\xxx)$, we may apply the reasoning of Lemma \ref{lem:DoublingOfSnowFlakesAndCoveringNumber} and cover $(\xxx, d^\alpha_\xxx)$, for any $\delta > 0$, by
\begin{equation}
    \label{PROOF_eq_thrm_measure_valued_quantitative_UAT__Ndeltabound}
    N_\delta 
        \le 
    C_{(\xxx,d_{\xxx})}^{\lceil
    \left(
        \log_2(\operatorname{diam}(\xxx)) 
            - 
        \log_2(\delta) 
    \right)
        / \alpha \rceil}
\end{equation}
distinct balls of radius $\delta$ centered in $\{ x_i \}_{i = 1}^{N_\delta} \subseteq \xxx$.
From now on, let
\begin{equation}
\label{PROOF_eq_thrm_measure_valued_quantitative_UAT__delta_DEF}
    \delta 
 \eqdef 
    \frac{
        \epsilon_A
    }{
        3L_f C_{\eta}C_{\Pi}
    }
,
\end{equation}
where we recall that $C_\eta \ge 1$ is the constant of the mixing function, see Definition \ref{ass_approximately_simplicial},
and write $N \eqdef N_\delta$ and $\mathcal{X}_N  \eqdef  \{ x_i \}_{i = 1}^N$. Then
\begin{align}
    \label{PROOF_eq_thrm_measure_valued_quantitative_UAT__Ndeltabound_BEGIN}
    \max_{x \in \xxx} \min_{1 \le i \le N} d_\yyy(f(x), f(x_i)) 
    &\le
    L_f \max_{x \in \xxx} \min_{1 \le i \le N} d_\xxx^\alpha(x, x_i)
    \\
    \label{PROOF_eq_thrm_measure_valued_quantitative_UAT__Ndeltabound_END}
    &\le
    L_f \delta \le \epsilon_A / 3.
\end{align}
In particular, the $\epsilon_A / 3$-external covering number of $f(\xxx)$ is at most $N$.\\
\par
\noindent

\textbf{Step 3 -- $\mathcal P_1(\mathcal{X}_N)$ and the simplex $\Delta_N$:} 
Let $\Pi_{\mathcal P_1(\mathcal{X}_N)} \colon (\xxx, d_\xxx^\alpha) \to (\mathcal P_1(\mathcal{X}_N),\mathcal W_1)$ be a $C_\Pi$-Lipschitz random projection, which exists by what was recalled in Step 1. 
We find by the triangle inequality and \eqref{PROOF_eq_thrm_measure_valued_quantitative_UAT__delta_DEF} that
\begin{equation*}
    d_\xxx^\alpha(x_1,x_2) \le \delta \text{ and } d_\xxx^\alpha(x_2,x_3) \le \delta \implies d_\xxx^\alpha(x_1,x_3) \le (2\epsilon_A) / (3 L_f).
\end{equation*}
Choose $\tilde{\mathcal{X}}_N  \eqdef  \{ \tilde x_i \}_{i = 1}^{\tilde N}$, $\tilde N \le N$, to be a subset of $\mathcal{X}_N$ with
\begin{gather}
    \label{PROOF_eq_thrm_measure_valued_quantitative_covering_points_not_too_close}
    \delta / 2 \le \min_{i \neq j} d_\xxx^\alpha(\tilde x_i, \tilde x_j) \le \operatorname{diam}(\xxx)^\alpha,
    \\
    \label{PROOF_eq_thrm_measure_valued_quantitative_covering_points_not_too_far}
    \max_{1 \le i \le N} \min_{1 \le j \le \tilde N} d^\alpha_\xxx(x_i, \tilde x_j) \le \delta.
\end{gather}
Then $f(\tilde{\mathcal{X}}_N)$ is a $(2 \epsilon_A) / 3$-covering of $f(\xxx)$ consisting of at most $N$ points, because, by \eqref{PROOF_eq_thrm_measure_valued_quantitative_covering_points_not_too_far},
\begin{align}
    \max_{x \in \xxx} \min_{1 \le i \le \tilde N} d_\yyy(f(x), f(\tilde x_i))
    &\le
    \max_{x \in \xxx} \min_{1 \le i \le N} d_\yyy(f(x), f(x_i)) + \min_{1 \le j \le \tilde N} d_\yyy(f(x_i), f(\tilde x_i))
    \\
    \label{PROOF_eq_thrm_measure_valued_quantitative_23_covering}
    &\le
    L_f (\delta + \delta) \le (2 \epsilon_A) / 3.
\end{align}
From now on we will solely use $\tilde{\mathcal{X}}_N$ and, to ease of notation, rename $\tilde{\mathcal{X}}_N$ as $\mathcal{X}_N$, $\tilde{N}$ as $N$, and $\{\tilde{x}_i\}_{i=1}^{\tilde{N}}$ as $\{x_i\}_{i=1}^N$.
Clearly, the Wasserstein space $(\mathcal P_1(\mathcal{X}_N), \mathcal W_1)$ is homeomorphic to $(\mathcal P_1(\mathcal{X}_N), \operatorname{TV})$, where $\operatorname{TV}$ is the total variation distance on $\mathcal P_{1}(\mathcal{X}_N)$, and the two metrics are equivalent thereon, since
\begin{equation}
    \label{PROOF_eq_thrm_measure_valued_quantitative_TV_to_W1}
    (\delta / 2) \cdot \operatorname{TV}(\mu,\nu) 
    \le 
    \mathcal W_1(\mu,\nu)
    \le
    \operatorname{diam}(\xxx)^\alpha \cdot \operatorname{TV}(\mu,\nu).
\end{equation}
Let $\lVert \cdot \rVert_1$ be the $\ell^1$-norm and $\lVert \cdot \rVert$ the Euclidean norm.
Note that $(\mathcal P_1(\mathcal{X}_N), \operatorname{TV})$ is isometric to the simplex $(\Delta_N, \lVert \cdot \rVert_1)$ and homeomorphic to $(\Delta_N, \lVert \cdot \rVert)$.
Indeed, there is a map $\iota_N \colon \mathcal P_1(\mathcal{X}_N) \to \Delta_N$ with
\begin{equation}
    \label{PROOF_eq_thrm_measure_valued_quantitative_euclidean_to_TV}
    \lVert \iota_N(\mu) - \iota_N(\nu) \rVert
    \le
    \operatorname{TV}(\mu,\nu)
    = 
    \lVert \iota_N(\mu) - \iota_N(\nu) \rVert_1 
    \le
    \sqrt{N} \cdot \lVert \iota_N(\mu) - \iota_N(\nu) \rVert,
\end{equation}
for all $\mu,\nu \in \mathcal P_1(\mathcal{X}_N)$.
Combining \eqref{PROOF_eq_thrm_measure_valued_quantitative_projection_Lipschitzbound2} with \eqref{PROOF_eq_thrm_measure_valued_quantitative_TV_to_W1} and \eqref{PROOF_eq_thrm_measure_valued_quantitative_euclidean_to_TV} yields an estimate on the Lipschitz constant, denoted by $L_{f_{N}}$, of $f_N  \eqdef  \iota_N \circ \Pi_{\mathcal P_1(\mathcal{X}_N)} \colon (\xxx, d_\xxx^\alpha)\to(\Delta_N, \lVert \cdot \rVert) \subseteq (\mathbb R^N, \lVert \cdot \rVert)$:
\begin{align*}
     \lVert f_N(x) - f_N(\tilde x) \rVert
    &\le 
    \frac{2}{\delta} \mathcal W_1(\Pi_{\mathcal P_1(\mathcal{X}_N)}(x), \Pi_{\mathcal P_1(\mathcal{X}_N)}(\tilde x))
    \\
    &\le 
    \frac{4 c \lceil 1 / \alpha \rceil }{\delta} \cdot \log_2(C_{(\xxx, d_\xxx)})
    \cdot
    d_\xxx^\alpha(x,\tilde x)
    \\
    &=
    \left( 12 c^2 \lceil 1 / \alpha \rceil^2 L_f \right) \cdot 
    \left( 
    \frac{C_{\eta}
    }{\epsilon_A} 
    \right)^{1 / r} \cdot \log_2(C_{(\xxx, d_\xxx)})^2 
    \cdot 
    d_\xxx^\alpha(x,\tilde x)
\end{align*}
for all $x, \tilde x \in \xxx$.
\par
\noindent
\textbf{Step 4 -- Approximation of $f$:}
We denote the elements of $f(\mathcal{X}_N)$ by $y_i  \eqdef  f(x_i)$, $i = 1,\ldots, N$.
By Definition~\ref{def_disc_modulus_quantification}, we have that $f(\mathcal{X}_N)$ is quantizable, i.e.\
there exists $q\in \nn_+$ with $D_q \eqdef \,
\mathcal \mathcal{Q}_{f(\mathcal{X}_N)}(\epsilon_Q)$ and
there are $\{ \tilde y_i \}_{i = 1}^N \subseteq \yyy$ such that
\begin{equation}
     \label{PROOF_eq_thrm_measure_valued_quantitative_UAT__quantization}
   d_\yyy(y_i, \tilde y_i) 
        \le 
        \min\left\{
            \frac{\epsilon_Q}{3 C_{\eta}}
                ,
            \frac{\epsilon_Q}{3}
        \right\}
        ,\quad 1 \le i \le N,
\end{equation}
and $\tilde{\mathcal{Y}}_N  \eqdef  \{ \tilde y_i \}_{i = 1}^N$ are representable with $
\mathcal \mathcal{Q}_{f(\mathcal{X}_N)}(\epsilon_Q)$-parameters, i.e., 
there are $z_1,\dots,z_N\in \rr^{\mathcal \mathcal{Q}_{f(\mathcal{X}_N)}(\epsilon_Q)}$ such that 
$Q_q(z_i)=\tilde{y_i}$ for $i=1,\dots,N$ and the estimate~\eqref{PROOF_eq_thrm_measure_valued_quantitative_UAT__quantization} holds.  
Since $\mathcal{Y}_N=(y_1,\dots,y_N)$ and  $\tilde{\mathcal{Y}}_N=(\tilde y_1,\dots,\tilde y_N)$ are fixed, to ease the notation, from here and till the end of the proof we will use
\begin{align*}
    \eta_N &\colon \Delta_N \to \yyy \colon w \mapsto \eta( w, \mathcal{Y}_N), \\
    \widetilde{\eta}_N & \colon \Delta_N \to \yyy \colon w \mapsto \eta( w, \tilde{\mathcal{Y}}_N).
\end{align*}

Let $1 \le i \le N$, $w \in \Delta_N$ and $p\geq 1$.
Since, for every $x_i$, the function $x \mapsto d_\yyy(f(x), f(x_i))$ is $L_f$-Lipschitz on $(\xxx, d^\alpha)$, then we have
\[
    \mathcal W_p^p( \delta_{x_i}, \iota_N^{-1}(w)) 
   =
    \sum_{j=1}^N w_j d^\alpha_\xxx(x_i,x_j)^p
   \geq\frac{1}{L_f^p} \sum_{j=1}^N w_j d_\yyy(f(x_i),f(x_j))^p
   =\frac{1}{L_f^p} \sum_{j=1}^N w_j d_\yyy(y_i,y_j)^p.
\]
Using this, together with Definition~\ref{ass_approximately_simplicial} 
and \eqref{PROOF_eq_thrm_measure_valued_quantitative_UAT__quantization}, yields
\begin{align}
\nonumber
    d_\yyy
    (
        \widetilde{\eta}_N(e_i)
        ,
        \widetilde{\eta}_N(w)
    )
    &\le
    C_{\eta} \left(\sum_{j = 1}^N w_j d_\yyy ( \tilde y_i, \tilde y_j)^p\right)^{1/p}
    \\
    \nonumber
    &\le 
    C_{\eta} \left(\sum_{j = 1}^N w_j \left( d_\yyy ( \tilde y_i, y_i) + d_\yyy (y_i, y_j) + d_\yyy (y_j, \tilde y_j) \right)^p\right)^{1/p}
       \\
       \nonumber
    &\le 
    C_{\eta} \left(\sum_{j = 1}^N w_j \left( \frac{2\epsilon_Q}{3 C_{\eta}} + d_\yyy (y_i, y_j) \right)^p\right)^{1/p}
    \\
    \nonumber
     &\le 
 2\epsilon_Q/3+C_{\eta} \left(\sum_{j = 1}^N w_j d_\yyy (y_i, y_j)^p\right)^{1/p}
 \\
\label{PROOF_eq_thrm_measure_valued_quantitative_UAT__quantization_estimate}
     &\le 
 2\epsilon_Q/3+C_{\eta} L_f \mathcal W_p(\Pi_{\mathcal P_1(\mathcal{X}_N)}(x_i), \iota_N^{-1}(w)),
\end{align}
where we used Minkowski's inequality in the second to last step.
Now we want to obtain an estimate 
of $\mathcal W_p(\Pi_{\mathcal P_1(\mathcal{X}_N)}(x_i), \Pi_{\mathcal P_1(\mathcal{X}_N)}(x))$, for $1 \le i \le N$ and $x \in \xxx$. If $x\in\mathcal X_N$, then $\mathcal W_p(\Pi_{\mathcal P_1(\mathcal{X}_N)}(x_i), \Pi_{\mathcal P_1(\mathcal{X}_N)}(x))=d_\xxx(x,x_i)$.
For the case of $x\notin\mathcal X_N$, we note from the proof of \cite[Theorem~3.2]{BRUE2021LipExtRP} that the projection $\Pi_{\mathcal P_1(\mathcal{X}_N)}$ even satisfies
\begin{equation}
\label{eq:lippr}
    \mathcal W_p(\Pi_{\mathcal P_1(\mathcal{X}_N)}(x_i),\Pi_{\mathcal P_1(\mathcal{X}_N)}(x)) 
    \le
    C_\Pi d_\xxx(x_i,x),
\end{equation}
where crucially enters $x_i \in \mathcal X_N$.

Since, for $1 \le i \le N$, $\eta_N(e_i) = y_i = f(x_i)$, $\widetilde{\eta}_N(e_i) = \tilde y_i$ and $f_N(x_i) = e_i$,
we find by \eqref{PROOF_eq_thrm_measure_valued_quantitative_23_covering}, \eqref{PROOF_eq_thrm_measure_valued_quantitative_UAT__quantization}, \eqref{PROOF_eq_thrm_measure_valued_quantitative_UAT__quantization_estimate}, \eqref{eq:lippr} and Lipschitz continuity of $\Pi_{\mathcal P_1(\mathcal{X}_N)}$, that
\begin{align}
    \sup_{x \in \xxx} d_\yyy(f(x), \widetilde{\eta} ( f_N(x) ))
    &\le
    \sup_{x \in \xxx} \min_{1 \le i \le N}\big\{ d_\yyy(f(x),y_i) + d_\yyy(y_i, \tilde{y}_i)+  d_\yyy ( \tilde y_i, \widetilde{\eta}_N(f_N(x)))\big\}
    \\
    \nonumber
    &\le
    (2 \epsilon_A) / 3 + \epsilon_Q / 3 + \sup_{x \in \xxx} d_\yyy ( \widetilde{\eta}_N(e_{i(x)}), \widetilde{\eta}_N(f_N(x)))
    \\
     \nonumber
    &\le
    (2 \epsilon_A) / 3 + \epsilon_Q + \sup_{x \in \xxx} C_{\eta} L_f  \mathcal W_p(\Pi_{\mathcal P_1(\mathcal{X}_N)}(x_{i(x)}), \iota_N^{-1} \circ f_N(x) )
    \\
    \nonumber
    &=
    (2 \epsilon_A) / 3 + \epsilon_Q + \sup_{x \in \xxx} C_{\eta} L_f  \mathcal W_p(\Pi_{\mathcal P_1(\mathcal{X}_N)}(x_{i(x)}), \Pi_{\mathcal P_1(\mathcal{X}_N)}(x))
    \\
        \label{PROOF_eq_thrm_measure_valued_quantitative_quantization_estimate_pt_2}
    &\le
    (2 \epsilon_A) / 3 + \epsilon_Q +  C_{\eta} L_f  C_\Pi \delta
    ,
\end{align}
where $i(x)$ satisfies $i(x) \in \arg \min_{1 \le i \le N} d_\xxx(x,x_i)$  and $b$ is the universal constant in \eqref{eq:lippr}.
Recalling the definition of $\delta$, see \eqref{PROOF_eq_thrm_measure_valued_quantitative_UAT__delta_DEF}, yields
\[
    \sup_{x \in \xxx} d_\yyy ( f(x), \widetilde{\eta}_N(f_N(x)) ) \le \epsilon_A + \epsilon_Q.
\]
Hence, $\widetilde{\eta}_N \circ f_N$ is an $(\epsilon_A + \epsilon_Q)$-approximation of $f$.
\hfill\\

\textbf{Step 5 -- Approximating $f_N$ by neural networks:} 
Fix an approximation error $\epsilon_{\mathcal{NN}}>0$ of the neural network. 
Consider either scenario:
\begin{itemize}
    \item \textbf{Scenario A (Trainable Activation Function: Singular-ReLU-Type)}: Under Assumption~\ref{defn_TrainableActivation_Singular}, Proposition~\ref{prop_universal_approximation_improved_rates} guarantees that there is  a network $\hat{f}_{\theta}\in \NN[[d]][\sigma]$ where, $d=(d_0,\dots,d_J)$, $d_0=d$, and $d_J=N$
    of depth and width as recorded in Table~\ref{tab_spacecomplexity_universalmetrictransformer} satisfying~\eqref{PROOF_eq_thrm_measure_valued_quantitative_UAT__NN_approximation} below.
    \item \textbf{Scenario B (Trainable Activation Function: Smooth-ReLU-Type)}: Under Assumption~\ref{defn_TrainableActivation_Smooth}, we may apply \citep[Proposition 58]{kratsios2021universal} to conclude that for each $i=1,\dots,N$ there is a feedforward network with activation function $\sigma^{\star}$ of width $d+N+2$ and of depth 
    $\mathcal{O}\left(
(1+\frac{d}{4})^{\frac{2d}{\alpha}}
\epsilon^{\frac{-2d}{\alpha}}
\right)
$
approximating each $\operatorname{pj}_i\circ f_N$.  Since $\sigma_{(1,1)}$ is the identity on $\rr$, then we may (mutatis mutandis) apply the parallelization construction of \citep[Proposition 5]{FlorianHighDimensional2021} to conclude that there is a network $\hat{f}_{\theta}\in \NN[[d]][\sigma]$ where $d=(d_0,\dots,d_J)$, $d_0=d$, $d_J=N$, and of depth and width as recorded in Table~\ref{tab_spacecomplexity_universalmetrictransformer}, such that~\eqref{PROOF_eq_thrm_measure_valued_quantitative_UAT__NN_approximation} below holds.
    \item \textbf{Scenario C (Classical Activation Function)}: Under Assumption~\ref{defn_TrainableActivation_ClassicalNonTrainable}, \citep[Theorem 3.2]{KidgerLyons2020} applies, whence, there exists a network $\hat{f}_{\theta}\in \NN[[d]][\sigma]$ where $d=(d_0,\dots,d_J)$, $d_0=d$, $d_J=N$, of width at most $d+N+2$ satisfying~\eqref{PROOF_eq_thrm_measure_valued_quantitative_UAT__NN_approximation} below.
\end{itemize}
In either case, we may choose $\hat{f}$ such that
\begin{equation}
    \sup_{x\in \xxx}\,
    \|
        f_N(x) - \hat{f}_{\theta}(x)
    \| \leq \frac{\epsilon_{\mathcal{NN}}}{C_{\eta}L_f \operatorname{diam}(\xxx)^\alpha \sqrt{N}}
    .
    \label{PROOF_eq_thrm_measure_valued_quantitative_UAT__NN_approximation}
\end{equation}
Since the projection $
\Pi_{\Delta_N} \colon (\RR^N, \lVert \cdot \rVert_2) \to (\Delta_N, \lVert \cdot \rVert_2)
$ onto the convex subset $\Delta_N \subseteq \RR^N$ is 1-Lipschitz, and noting that $\Pi_{\mathcal P_1(\mathcal{X}_N)}(x) = \iota_N^{-1} \circ f_N(x)$,
we may compute
\allowdisplaybreaks
\begin{align}
    \label{PROOF_eq_thrm_measure_valued_quantitative_UAT__NN_approximation_application_to_delta______BEGIN}
    \sup_{x\in \xxx}
     \,
    \mathcal W_1 
    (
        \iota_N^{-1} \circ \Pi_{\Delta_N} \circ \hat f_\theta(x)
            ,
        \iota_N^{-1} \circ f_N(x)
    )
=&
    \sup_{x\in \xxx}
    \,
    \mathcal W_1 
    (
        \iota_N^{-1} \circ \Pi_{\Delta_N} \circ \hat f_\theta(x)
            ,
        \iota_N^{-1} \circ \Pi_{\Delta_N} \circ f_N(x)
    )
\\
\nonumber
\leq &
    \sup_{x\in \xxx}
    \,
    \operatorname{diam}(\xxx)^\alpha \sqrt{N}
    \cdot
    \|
        \Pi_{\Delta_N} \circ \hat f_\theta(x)
            -
        \Pi_{\Delta_N} \circ f_N(x)
    \|
\\
\nonumber
\leq &
    \sup_{x\in \xxx}
     \,
    \operatorname{diam}(\xxx)^\alpha \sqrt{N}
    \cdot
    \|
        \hat f_\theta(x)
            -
        f_N(x)
    \|
\\
\leq &
    \frac{\epsilon_{\mathcal{NN}}}{C_{\eta} L_f}
    ,
    \label{PROOF_eq_thrm_measure_valued_quantitative_UAT__NN_approximation_application_to_delta______END}
\end{align}
where we applied \eqref{PROOF_eq_thrm_measure_valued_quantitative_TV_to_W1} and \eqref{PROOF_eq_thrm_measure_valued_quantitative_euclidean_to_TV} to obtain the first, and \eqref{PROOF_eq_thrm_measure_valued_quantitative_UAT__NN_approximation} to obtain the last inequality.

Now let $x \in \xxx$ and $i = i(x)$. Then, analogously to Step 4, \eqref{PROOF_eq_thrm_measure_valued_quantitative_23_covering} and \eqref{PROOF_eq_thrm_measure_valued_quantitative_UAT__quantization_estimate} give
\begin{align}
    \nonumber
    d_\yyy 
    (f(x)
        ,
        \widetilde{\eta}_N \circ \Pi_{\Delta_N} \circ \hat f_\theta(x)
    )
    &\le
    d_\yyy
    (f(x)
        ,
        \widetilde{\eta}_N(e_i))
    )
    +
    d_\yyy
    (
        \widetilde{\eta}_N(e_i)
        ,
        \widetilde{\eta}_N \circ \Pi_{\Delta_N} \circ \hat f_\theta(x)
    )
    \\
    \label{PROOF_eq_thrm_measure_valued_quantitative_UAT__NN_approx1}
    &\le
    (2 \epsilon_A) / 3 +  2\epsilon_Q / 3 
    +
    C_{\eta} L_f \mathcal W_1( \Pi_{\mathcal P_1(\mathcal{X}_N)}(x_i), \iota_N^{-1} \circ \Pi_{\Delta_N} \circ  \hat f_\theta(x)).
\end{align}
Note that in the last term we can use the estimate
\begin{align*}
    \mathcal W_1( \Pi_{\mathcal P_1(\mathcal{X}_N)}(x_i), \iota_N^{-1} \circ \Pi_{\Delta_N} \circ  \hat f_\theta(x))
    &\le
    \mathcal W_1 ( \Pi_{\mathcal P_1(\mathcal{X}_N)}(x_i), \Pi_{\mathcal P_1(\mathcal{X}_N)}(x) )
    +
    \mathcal W_1( \Pi_{\mathcal P_1(\mathcal{X}_N)}(x), \iota_N^{-1} \circ \Pi_{\Delta_N} \circ  \hat f_\theta(x))
    \\
    &\le
    C_\Pi \delta
    +
    \frac{\epsilon_{\mathcal{NN}}}{C_{\eta} L_f}.
\end{align*}
Finally, plugging this into \eqref{PROOF_eq_thrm_measure_valued_quantitative_UAT__NN_approx1} yields
\begin{align*}
    \sup_{x \in \xxx} d_\yyy (f(x), \widetilde{\eta}_N\circ \Pi_{\Delta_N} \circ \hat f_\theta(x))
    &\le
    (2\epsilon_A)/ 3 +  2\epsilon_Q / 3 + C_{\eta} L_f C_\Pi \delta + \epsilon_{\mathcal{NN}}
    \\
    & \leq
    \epsilon_A + \epsilon_Q + \epsilon_{\mathcal{NN}}.
\end{align*}
This completes the approximation. \hfill\\
\textbf{Step 6 -- Expressing $N$ and $Q$ using metric entropy:} 
It remains to derive simple expressions for $N$ and $Q$.  Together, \eqref{PROOF_eq_thrm_measure_valued_quantitative_projection_Lipschitzbound2},~\eqref{PROOF_eq_thrm_measure_valued_quantitative_UAT__Ndeltabound} and~\eqref{PROOF_eq_thrm_measure_valued_quantitative_UAT__delta_DEF} imply that
\begin{align}
\nonumber
    N
\leq &
    C_{(\xxx,d_{\xxx})}
    ^{
    \left\lceil
    \left(
        \log_2(\operatorname{diam}(\xxx))
            -
        \log_2(\delta)
    \right)/\alpha
    \right\rceil
    }
\\
\nonumber
= &
    C_{(\xxx,d_{\xxx})}
    ^{
    \left\lceil
    \left(
        \log_2(\operatorname{diam}(\xxx))
            -
        \log_2\left(
         \epsilon_A/\left(3L_f
                    C_{\eta}C_{\Pi}   
            \right)
        \right)
    \right)/\alpha
    \right\rceil
    }
\\
\nonumber
= &
    C_{(\xxx,d_{\xxx})}
    ^{
    \left\lceil
    \alpha^{-1}
    \left(
        \log_2(\operatorname{diam}(\xxx))
            -
        \log_2\left(
            \epsilon_A/(3L_f)
        \right)
        +
        \log_2\left(    
                (
                    C_{\eta}C_{\Pi}
                )    
            \right)
    \right)
    \right\rceil
    }
\\
\label{PROOF_expressing_N_using_metric_entropy}
= &
C_{(\xxx,d_{\xxx})}
    ^{
    \left\lceil
    \alpha^{-1}
    \left(
        \log_2(\operatorname{diam}(\xxx))
            -
        \log_2\left(
            \epsilon_A/(3L_f)
        \right)
        +
        \log_2\left(
                    C_{\eta}
                    2 c \lceil 1 / \alpha \rceil \cdot \log_2(C_{(\xxx, d_\xxx)})
            \right)
    \right)
    \right\rceil
    }
.
\end{align}
By \citep[Proposition 1.7 (i)]{Brue2021Extension}, we have the estimate
$
C_{(\xxx,d_{\xxx})}
    \leq 
\kappa_{\xxx}(5^{-1})$. Thus,~\eqref{PROOF_expressing_N_using_metric_entropy} can be bounded above as
\small
\[
    \ln(N)
\leq %
    \ln\left(\kappa_{\xxx}\left(5^{-1}\right)\right)
    {
    \left\lceil
    \alpha^{-1}
    \left(
        \log_2(\operatorname{diam}(\xxx))
            -
        \log_2\left(
            \epsilon_A/3L_f
        \right)
        +
        \log_2\left(
                    C_{\eta}
                    2 c \lceil 1 / \alpha \rceil \cdot \log_2\left(\kappa_{\xxx}\left(5^{-1}
                    \right)\right)
                 \right)
    \right)
    \right\rceil
    }
.
\]
\normalsize
This completes the simplified estimate on $N$.  

For $Q$, following~\eqref{PROOF_eq_thrm_measure_valued_quantitative_UAT__quantization} we have that $Q=\mathcal \mathcal{Q}_{f(\mathcal{X}_N)}(\epsilon_Q)$.  Moreover, by definition, since $\xxx_N\subseteq \xxx$, we have that 
\[
Q = \mathcal \mathcal{Q}_{f(\mathcal{X}_N)}(\epsilon_Q)
    \leq 
\mathcal \mathcal{Q}_{f(\xxx)}(\epsilon_Q).
\]
This completes the estimate on $Q$ as well as our proof.  
\end{proof}

\subsection{Proofs for the Dynamic Case}
\label{s_Proofs__ss_DynamicCase}

\begin{proof}[{Proof of Proposition~\ref{prop:HighProbConstainement}}]
Under the assumed condition on the $\mathbb{F}$-progressively measurable stochastic processes $(\alpha_t)_{t\ge 0}$ and $(\beta_t)_{t\ge 0}$, we can conclude from \citep[Theorem 2.2]{Nizar_OSC_STP_BSDE_2013} that
\begin{equation}
\label{eq:GrowthRate_SDE__NizarControl}
        \mathbb{E} \left[\underset{{t\in [0,t_n]}}{\operatorname{sup}}\,\, ||X_t||^2 \right]
    \le 
        C_n e^{C_n\,t_n}
\end{equation}
holds for every $n\in \mathbb{N}_+$, for some $C_n>0$, where each $C_n$ depends only on the prescribed local-Lipschitz constant $L_n$, on $t_n$, and on the law of $X_0$.  For every $n\in \mathbb{N}_+$, let $\tilde{C}_n\ge C_n$ be chosen large enough such that the exponential sum $\sum_{n=1}^{\infty}\, e^{(C_n-\tilde{C}_n)\,n\,{\delta_+}}$ converges.
Now, we may apply Tonelli's theorem 
and deduce that
\allowdisplaybreaks
\begin{align}
\begin{split}
        \mathbb{E} \left[\sum_{n=1}^\infty \underset{{t\in [0,t_n]}}{\operatorname{sup}}\,\, ||X_t||^2 \frac{e^{-\tilde{C}_n\,t_n}}{C_n}\right]  
    = &
        \sum_{n=1}^\infty \mathbb{E} \left[\underset{{t\in [0,t_n]}}{\operatorname{sup}}\,\, ||X_t||^2\right] \frac{e^{-\tilde{C}_n\,t_n}}{C_n}
    \le 
        \sum_{n=1}^\infty C_n\, e^{\,C_n\,t_n}\,  \frac{e^{-\tilde{C}_n\,t_n}}{C_n}
    \\
    \label{eq:GrowthRate_SDE__MarkovSetup}
    =&  
        \sum_{n=1}^\infty \, e^{(C_n-\tilde{C}_n)\,t_n}
    \le   
        \sum_{n=1}^\infty \, e^{(C_n-\tilde{C}_n)\,n\,{\delta_-}}
    <
        \infty.
        \end{split}
\end{align}
Since~\eqref{eq:GrowthRate_SDE__MarkovSetup} implies that $\sum_{n=1}^\infty \underset{{t\in [0,t_n]}}{\operatorname{sup}}\,\, ||X_t||^2 \frac{e^{-( C_n-\tilde{C}_n)\,t_n}}{C_n}$ is a.s.\ finite, and since this quantity is non-negative, we may apply the Markov inequality to deduce that, for any given positive constant $\gamma$, the following concentration bound holds
\begin{equation}
\label{eq:GrowthRate_SDE__MarkovApplied}
    \mathbb{P}\left(
            \sum_{n=1}^\infty 
                \underset{{t\in [0,t_n]}}{\operatorname{sup}}\,\, ||X_t||^2 
                \frac{e^{ -\tilde{C}_n \,t_n}}{C_n}
            \ge 
                \gamma
        \right) 
    \le 
    \frac1{
        \gamma
    }\,
        \sum_{n=1}^\infty \, e^{(C_n-\tilde{C}_n)\,t_n}
    \le 
    \frac{1}{\gamma}\,
    \sum_{n=1}^{\infty}\, e^{( C_n-\tilde{C}_n)\,n\,\delta_- }
.
\end{equation}
In particular,~\eqref{eq:GrowthRate_SDE__MarkovApplied} implies the looser bound
\allowdisplaybreaks
\begin{align}
\label{eq:GrowthRate_SDE__MarkovApplied___looser}
\begin{split}
    \mathbb{P}\left(
                \sup_{n\in \mathbb{N}_+}\,
                    \frac{e^{- \tilde{C}_n \,t_n / 2}}{C_n^{1/2}} \, 
                    \underset{{t\in [0,t_n]}}{\operatorname{sup}}\,\, ||X_t|| 
            \ge 
                \gamma^\frac1{2}
        \right) 
    = &
        \mathbb{P}\left(
                \sup_{n\in \mathbb{N}_+}\,
                \frac{e^{-  \tilde{C}_n \, t_n}}{C_n}\,
                    \underset{{t\in [0,t_n]}}{\operatorname{sup}}\,\, ||X_t||^2 
            \ge 
                \gamma
        \right) 
    \\
    \le & 
        \mathbb{P}\left(
                \sum_{n=1}^{\infty}
                    \underset{{t\in [0,t_n]}}{\operatorname{sup}}\,\, ||X_t||^2\,
                    \frac{e^{ - \tilde{C}_n \, t_n }}{C_n}
            \ge 
                \gamma
        \right) 
    \\ 
    \le &
        \frac{1}{\gamma}\,
            \sum_{n=1}^{\infty}\, e^{( C_n-\tilde{C}_n)\,n\,\delta_- }
    .
    \end{split}
\end{align}
Define the positive constant 
$C_{(\delta_-,\mathbb{L})} \eqdef  \biggl(\sum_{n=1}^{\infty}\, e^{(C_n-\tilde{C}_n) \,n\,\delta_-}\biggr)^{1/2}$, and note that 
$C_{(\delta_-,\mathbb{L})}$ 
only depends on the local-Lipschitz constant $\mathbb{L}$ and on the minimal time-grid spacing $\delta_-$ of the time-grid $\mathbb{T}$.
Define $\gamma  \eqdef  \frac1{\varepsilon}\, \sum_{n=1}^{\infty}\, e^{(C_n-\tilde{C}_n) n\,\delta_-}$ where $0<\eps \le 1$.  
Rearranging~\eqref{eq:GrowthRate_SDE__MarkovApplied___looser} yields 
\begin{equation}
\label{eq:GrowthRate_SDE__Concentration}
\begin{aligned}
    \mathbb{P}\left(
    (\forall n\in\mathbb{N}_+)\,
    \|X_{t_n}\|
    \le 
    \frac{C_{(\delta_-,\mathbb{L})}}{
    \varepsilon^{1/2}}\,
    C_n^{1/2}
    \, 
    e^{-C_n \,\delta_-  n/2 } 
    \right) 
    \ge &
    \mathbb{P}\left(
    \sup_{n\in \mathbb{N}_+}\, 
    \frac{e^{-\tilde{C}_n \,t_n/2}}
    {C_n^{1/2}}\,
    \underset{{t\in [0,t_n]}}{\operatorname{sup}}\,\, ||X_t||
    \le 
    \frac{C_{(\delta_-,\mathbb{L})}}{
    \varepsilon^{1/2}}
    \right) 
    \\
    \ge &   1   -  \eps
.
\end{aligned}
\end{equation}

$\bullet$ \textbf{Case $1$ - Deterministic $X_0$:} If $X_0$ is deterministic, i.e.\ $X_0=x\in \mathbb{R}^d$, then the concentration inequality in~\eqref{eq:GrowthRate_SDE__Concentration} implies that the discretized stochastic process $(X_{t_n})_{n\in \mathbb{N}}$ belongs to the compact subset $K_{\mathbb{C},C^{\star},\varepsilon}^{\operatorname{exp}} $ of the product space $(\mathbb{R}^d)^{\mathbb{N}}$ with probability at least $(1-\varepsilon)$, where $C^{\star}\eqdef 1$ and $C_0 \eqdef \|x\|$. 

$\bullet$ \textbf{Case $2$ - Sub-Gaussian $X_0$:} If $X_0$ is sub-Gaussian, then there are positive constants $c_0$ and $c_1$, both depending only on the law of $X_0$, such that
\begin{equation}
\label{eq:subGaussian__concentration}
        \mathbb{P}\left(
            \|X_0\| \le L_0
        \right)
    \ge 
       1 - c_0\, e^{c_1\, L_0^2}.
\end{equation}
By the independence of $X_0$ and $W_0$, we may combine the concentration inequalities~\eqref{eq:GrowthRate_SDE__Concentration} and~\eqref{eq:subGaussian__concentration} to conclude that the discretized stochastic process $(X_{t_n})_{n\in \mathbb{N}}$ belongs to the compact subset $K_{\mathbb{C},C^{\star},\varepsilon}^{\operatorname{exp}} $ of the product space $(\mathbb{R}^d)^{\mathbb{N}}$ with probability at least $C^{\star}\,(1-\varepsilon)$, where $C^{\star}\eqdef (1-c_0\, e^{c_1\, L_0^2})$ and $C_0\eqdef L_0$.  
\end{proof}
Next, before embarking on the proof of our main result, namely Theorem~\ref{thrm_main_DynamicCase}, we take a moment to explain the intuition behind it.  
The proof ultimately reduces to controlling the distance between the target causal map $F$ against an aptly chosen geometric hypertransformer model $\hat{F}$ defined on the same input and output spaces as $F$.  Our proof works by decomposing an upper-bound for the distance between $F$ and $\hat{F}$ into four terms, each of which is either controlled by our choice of  model $\hat{F}$ or by the regularity of the causal map $F$.

Let us briefly explain each of the four terms appearing in the proof.  
The first term, ~\eqref{PROOF_thrm_main__2_main_inequality_to_control___AC} within the proof, is given by $F$'s approximable complexity and reduces the approximation problem of an infinitely complicated causal map to a causal map of finite complexity $F^{\rho_{\epsilon},f_{\epsilon}}$ (implied by Definition~\ref{defn_compression_memory_property}).  
The second term,~\eqref{PROOF_thrm_main__2_main_inequality_to_control_Term_A_FINITE_TIME_HORIZON} within the proof, controls the error between our candidate  geometric hypertransformer $\hat{F}$ and the causal map of finite complexity $F^{\rho_{\epsilon},f_{\epsilon}}$ within a prespecified finite-time horizon $[-T,T]$.
The third term,~\eqref{PROOF_thrm_main__2_main_inequality_to_control_Term_B_Extrapolation_quantity} within the proof, constructs the extrapolation quality function $c$ in Table~\ref{tab_extrapolation_function} to control the deviation of $F^{\rho_{\epsilon},f_{\epsilon}}$ outside the time window $[-T,T]$ from the values it takes within the time window $[-T,T]$.
Similarly, the fourth and last term,~\eqref{PROOF_thrm_main__2_main_inequality_to_control_Term_C_Ergodicity_Type_Term} within the proof, controls the deviation of our candidate geometric hypertransformer model $\hat{F}$ outside, from its behaviour within, the time window $[-T,T]$.

\begin{proof}[{Proof of Theorem~\ref{thrm_main_DynamicCase}}]

\textbf{Step 1 - Decomposing Approximation Error:}\hfill\\
Since $\kkk\subseteq (\rr^d)^{\zz}$ is compact and $F$ has approximable complexity, then
there exists a function $c_{AC}:\zz\times(0,\infty)\rightarrow [1,\infty)$ such that, 
for every $\epsilon >0$, there exist $f_{\epsilon}\in C^{\alpha}(\rr\times \rr^{d m(\epsilon)},\rr^{L(\epsilon)})$
and 
$\rho_{\epsilon}\in C^{\alpha}(\rr^{L(\epsilon)},\yyy)$ as in Definition~\ref{defn_compression_memory_property}, satisfying
\begin{equation}
    \sup_{n\in \zz}
    \sup_{\xb\in\kkk}
    \,
    \frac{
d_\yyy\left(
        F(\xb)_{t_n}
    ,
        F^{\rho_{\epsilon},f_{\epsilon}}(\xb)_{t_n}
\right)
    }{
    c_{AC}(n,\epsilon)
    }
    < \frac{\epsilon}{4}
    \label{PROOF_thrm_main__0_PMP_definition}
    .
\end{equation}
\textit{With an abuse of notation, we opt to write $f,\rho,m,L$ depending on $\epsilon$ rather than $\epsilon/4$ throughout the proof to ease the reading.}

For any other causal map $\hat{F}:\xxx^{\zz}\rightarrow \yyy^\zz$ and any function $c=c_{\epsilon}:\zz\rightarrow [1,\infty)$,
by letting 
 \[
 c'(n,\epsilon) \eqdef \max\{c_{AC}(n,\epsilon/4),c(n)\},
 \]
the following estimate holds:
\begin{align}
\nonumber
    \sup_{n\in \zz}\sup_{\xb\in\kkk} 
    \,
    \frac{
d_\yyy\left(
        F(\xb)_{t_n}
    ,
        \hat{F}(\xb)_{t_n}
\right)
    }{
 c'(n,\epsilon)
    }
\leq &
\sup_{n\in \zz}\sup_{\xb\in\kkk}
 \left(   \frac{
d_\yyy\left(
        F(\xb)_{t_n}
    ,
        F^{\rho_{\epsilon},f_{\epsilon}}(\xb)_{t_n}
\right)
    }{
  c'(n,\epsilon)
    }
+
    \frac{
d_\yyy\left(
        F^{\rho_{\epsilon},f_{\epsilon}}(\xb)_{t_n}
            ,
        \hat{F}(\xb)_{t_n}
\right)
    }{
   c'(n,\epsilon)
    }
    \right)
\\
\leq &
 \frac{\epsilon}{4}
+
\sup_{n\in \zz}\sup_{\xb\in\kkk}
    \frac{
d_\yyy\left(
        F^{\rho_{\epsilon},f_{\epsilon}}(\xb)_{t_n}
            ,
        \hat{F}(\xb)_{t_n}
\right)
    }{
    c'(n,\epsilon)
    }
    \label{PROOF_thrm_main__1_triangleInequality_Setup}
    .
\end{align}
Therefore,
\begin{align}
\label{PROOF_thrm_main__2_main_inequality_to_control___AC}
\tag{FC. Approx}
    & \sup_{n\in \zz}\sup_{\xb\in\kkk} 
    \,
    \frac{
d_\yyy\left(
        F(\xb)_{t_n}
    ,
        \hat{F}(\xb)_{t_n}
\right)
    }{
   c'(n,\epsilon)
    }
\leq 
 \frac{\epsilon}{4} + \\
\label{PROOF_thrm_main__2_main_inequality_to_control_Term_A_FINITE_TIME_HORIZON}
\tag{Window}
&
\max_{|n|\leq N_T}\sup_{\xb\in\kkk}
    \frac{
d_\yyy\left(
        F^{\rho_{\epsilon},f_{\epsilon}}(\xb)_{t_n}
            ,
        \hat{F}(\xb)_{t_n}
\right)
    }{
    c'(n,\epsilon)
    }
\\
\label{PROOF_thrm_main__2_main_inequality_to_control_Term_B_Extrapolation_quantity}
\tag{Growth F.A}
& +
\sup_{n > N_T}\sup_{\xb\in\kkk}
    \frac{
d_\yyy\left(
        F^{\rho_{\epsilon},f_{\epsilon}}(\xb)_{t_n}
            ,
        \hat{F}(\xb)_{t_{N_T}}
\right)
    }{
   c'(n,\epsilon)
    }
+
    \sup_{n <- N_T}\sup_{\xb\in\kkk}
    \frac{
d_\yyy\left(
        F^{\rho_{\epsilon},f_{\epsilon}}(\xb)_{t_n}
            ,
        \hat{F}(\xb)_{t_{-N_T}}
\right)
    }{
   c'(n,\epsilon)
    }
\\
\label{PROOF_thrm_main__2_main_inequality_to_control_Term_C_Ergodicity_Type_Term}
\tag{Growth $\hat{F}$}
& +
\sup_{n> N_T}\sup_{\xb\in\kkk}
    \frac{
d_\yyy\left(
        \hat{F}(\xb)_{t_n}
            ,
        \hat{F}(\xb)_{t_{N_T}}
\right)
    }{
     c'(n,\epsilon)
    }
+
\sup_{n< -N_T}\sup_{\xb\in\kkk}
    \frac{
d_\yyy\left(
        \hat{F}(\xb)_{t_n}
            ,
        \hat{F}(\xb)_{t_{-N_T}}
\right)
    }{
     c'(n,\epsilon)
    },
\end{align}
for $N_T$ defined in \eqref{def:GHT}.
The remainder of the proof is devoted to controlling terms~\eqref{PROOF_thrm_main__2_main_inequality_to_control_Term_A_FINITE_TIME_HORIZON}, \eqref{PROOF_thrm_main__2_main_inequality_to_control_Term_B_Extrapolation_quantity} and \eqref{PROOF_thrm_main__2_main_inequality_to_control_Term_C_Ergodicity_Type_Term} for a particular systems $\hat{F}$, implemented by a recurrent probabilistic transformer model that will be specified later.
We start by noticing that, by the triangle inequality,
\begin{align}
\label{PROOF_thrm_main__2_main_inequality_to_control_Term_B_Extrapolation_quantity___plut_Term_A}
    \sup_{n > N_T}\sup_{\xb\in\kkk}
    \frac{
d_\yyy\left(
        F^{\rho_{\epsilon},f_{\epsilon}}(\xb)_{t_n}
            ,
        \hat{F}(\xb)_{t_{N_T}}
\right)
    }{
     c'(n,\epsilon)
    }
\leq &\,
\sup_{\xb\in\kkk}\,
        \frac{
    d_\yyy\left(
            F^{\rho_{\epsilon},f_{\epsilon}}(\xb)_{t_{N_T}}
                ,
            \hat{F}(\xb)_{t_{N_T}}
    \right)
        }{
         c'(n,\epsilon)
        }
\\
\label{PROOF_thrm_main__2_main_inequality_to_control_Term_B_Extrapolation_quantity___pure_term_B}
\tag{Growth F.B}
+ &
    \sup_{n > N_T}
    \sup_{\xb\in\kkk}\,
        \frac{
    d_\yyy\left(
            F^{\rho_{\epsilon},f_{\epsilon}}(\xb)_{t_n}
                ,
            F^{\rho_{\epsilon},f_{\epsilon}}(\xb)_{t_{N_T}}
    \right)
        }{
         c'(n,\epsilon)
        }
.
\end{align}
And mutatis mundanes for the case where $n<-N_T$.
Since the term on the RHS of \eqref{PROOF_thrm_main__2_main_inequality_to_control_Term_B_Extrapolation_quantity___plut_Term_A} is 
bounded by \eqref{PROOF_thrm_main__2_main_inequality_to_control_Term_A_FINITE_TIME_HORIZON}, then the control of the latter (see Step~2) together with the control of~\eqref{PROOF_thrm_main__2_main_inequality_to_control_Term_B_Extrapolation_quantity___pure_term_B}, will imply the control of~\eqref{PROOF_thrm_main__2_main_inequality_to_control_Term_B_Extrapolation_quantity} (see Step~4).  
For the particular system $\hat{F}$ specified in Step~2, we will control the terms \eqref{PROOF_thrm_main__2_main_inequality_to_control_Term_A_FINITE_TIME_HORIZON} and \eqref{PROOF_thrm_main__2_main_inequality_to_control_Term_B_Extrapolation_quantity___pure_term_B} each by $\epsilon/4$, and show that \eqref{PROOF_thrm_main__2_main_inequality_to_control_Term_C_Ergodicity_Type_Term} vanishes. This will conclude the proof.

 From this is clear that, in the statement of Theorem~\ref{thrm_main_DynamicCase}, we could separately specify the ``AC approximation error''
 (dominating \eqref{PROOF_thrm_main__2_main_inequality_to_control___AC}), the ``approximation error within the time window $[-T,T]$'' (dominating \eqref{PROOF_thrm_main__2_main_inequality_to_control_Term_A_FINITE_TIME_HORIZON}), the ``extrapolation error'' (dominating \eqref{PROOF_thrm_main__2_main_inequality_to_control_Term_B_Extrapolation_quantity}), and the ``growth error''  (dominating \eqref{PROOF_thrm_main__2_main_inequality_to_control_Term_C_Ergodicity_Type_Term}), so that the total error is bounded by their sum.

\textbf{Step 2 - Construction of Model and Control of the finite-time horizon term ~\eqref{PROOF_thrm_main__2_main_inequality_to_control_Term_A_FINITE_TIME_HORIZON}:}\hfill\\
We first control the term in ~\eqref{PROOF_thrm_main__2_main_inequality_to_control_Term_A_FINITE_TIME_HORIZON}.  
For integers $m,n \in \Z$ with $m < n$ and $\xb \in \kkk$ we use the shorthand notation $x_{t_m:t_n}$ to refer to the vector $(x_{t_m}, x_{t_{m + 1}}, \ldots, x_{t_n})$.
For each $n\in \zz$ with $|n|\leq N_T$, the continuity of the projection maps
$\operatorname{pj}_n:(\rr^d)^{\zz}\ni (x_{t_n})_{n\in \zz}\mapsto x_{t_n}\in \rr^d$ implies that, for every $\tilde{n}\le n\in\zz$, the subset $K_{t_{\tilde{n}}:t_n} \eqdef \prod_{i=\tilde{n}}^n\,\operatorname{pj}_i(\kkk)$ is non-empty and compact in $\rr^d$.  
Therefore, for  
\[
\epsilon_{A,1} \eqdef \min_{n=-N_T,\dots,N_T}\;\frac{1}{C_{K_{t_n-dm(\epsilon):t_n}}}
           \left(\frac{\epsilon}{8  L_{\alpha,\rho_{\epsilon}}}\right)^{1/\alpha},
\]
where $C_{K_{t_n-dm(\epsilon/4):t_n}}$ is defined as in~\eqref{def:ck}, and for every $n\in \zz$ with $|n|\leq N_T$,
Proposition~\ref{prop_universal_approximation_improved_rates} implies that there exist a multi-index $[d^{(n)}]$
and a neural network $\hat{f}_{\theta_n}:\rr^{d\, m(\epsilon)}\rightarrow \rr^{L(\epsilon)}$ satisfying
\begin{equation}
    \sup_{\xb\in K_{
                t_n-dm(\epsilon):t_n
            }}\,
    \left\|
    f_{\epsilon}(t_n,x_{t_n}) 
        -
    \hat{f}_{\theta_n}(x_{t_n})
    \right\|
    \leq 
        C_{K_{
            t_n-dm(\epsilon):t_n
        }}\epsilon_{A,1}
    \leq 
        \left(\frac{\epsilon}{8  L_{\alpha,\rho_{\epsilon/4}}}\right)^{1/\alpha},
    \label{PROOF_thrm_main__3_1_first_estimate_control_of_f____Finite_horizon_term}
\end{equation}
where the complexity of $\hat{f}$ is recorded in Table~\ref{tab_spacecomplexity_universalFFNN}.  
Further, we may assume that $\{\theta_n\}_{-T\leq t_n\leq T}$ are all distinct. Indeed, let us first consider the case of a singular trainable activation function as in Definition~\ref{defn_TrainableActivation_Singular}.  
With the notation of~\eqref{eq_hypernetwork_association}, by setting $\theta_n' \eqdef \left(
\theta_n,(I_{L(\epsilon)},b_n',\alpha')
\right)$, with $b_n'\in \rr^{L(\epsilon)}$ such that $
\|b_n'\|< \left(\frac{\epsilon}{8  L_{\alpha,\rho_{\epsilon/4}}}\right)^{1/\alpha}$and $b_{-N_T}',\dots,b_{N_T}'$ all distinct, and with $\alpha'_i = (0,1)$ for each $i=1,\dots,L(\epsilon)$, then the $\{\theta_n'\}_{-T\leq t_n\leq T}$ are also all distinct and the estimate in \eqref{PROOF_thrm_main__3_1_first_estimate_control_of_f____Finite_horizon_term} holds up to a factor of $2$.
In other words, we add an additional layer of depth which approximates the identity on $\rr^{L(\epsilon)}$, whilst making 
all $\theta_n$'s distinct for $|n|\leq N_T$.  
In the case of $\sigma$ as in Definition~\ref{defn_TrainableActivation_Smooth} or as in Definition~\ref{defn_TrainableActivation_ClassicalNonTrainable}, by the continuity of $\sigma$, we may pick $\theta_n'$ arbitrarily close to $\theta_n$ such that the $\theta_{-N_T},\dots,\theta_{N_T}$ are all distinct.

Since $\sigma_{(1,1)}(x)=x$ for all $x\in \rr$, then, without loss of generality, we may assume that $d_{J_{(n)}}=d_{J_{(m)}}$ for all $n,m=-N_T,\dots,N_T$, by adding identity layers in~\eqref{eq_definition_ffNNrepresentation_function}.  Similarly, by adding $0$ rows to the matrices $A^{(j)}$ and the vectors $b^{(j)}$ in~\eqref{eq_definition_ffNNrepresentation_function}, we may without loss of generality assume that $d_j^{(n)}=d_j^{(0)}$ for all $j=1,\dots,J_{(n)}$ and $n=-N_T,\dots,N_T$.  Consequently, $P([d^{(n)}])=P([d^{(0)}])$ for all $-N_T,\dots,N_T$.

Now, let $\mathcal{K}\subset\rr^{L(\epsilon)}$ be defined by
\[
\mathcal{K}
     \eqdef  
        \bigcup_{n=-N_T}^{N_T} 
\left\{
z\in \rr^{L(\epsilon)}:\, \|z-f(t_n,K_{t_n-dm(\epsilon):t_n})\|\leq \epsilon_{A,1}
\right\}
.
\]
Since each $f(t_n,\cdot)$ is continuous and each $K_{t-dm(\epsilon):t}$ is compact, then, for each $n=-N_T,\dots,N_T$, the sets $\left\{
z\in \rr^{L(\epsilon)}:\, \|z-f(t_n,K_{t-dm(\epsilon):t})\|\leq \epsilon_{A,1}
\right\}$ are compact. Moreover, as the union of finitely many compacts is again compact, then $\mathcal{K}\subseteq \rr^{L(\epsilon)}$ is compact as well.  
Therefore, we may apply Theorem~\ref{thrm_main_StaticCase} to conclude the existence of a network $\hat{\rho}:\mathcal{K}\rightarrow \mathcal{P}_1(\mathbb{R}^m)$ with representation
\[
\hat{\rho}(\cdot) 
    = 
\operatorname{attention}_{N,q}(\hat{r}_{\theta}(\cdot),Y),
\]
where $Y$ is an $N\times q\times m$-array, $\hat{r}\in \NN[\cdot]$ maps from $\rr^d$ to $\rr^N$, and estimates for the number of parameters defining both are given in Table~\ref{tab_spacecomplexity_universalmetrictransformer}, and such that $\hat{\rho}$ satisfies the estimate
\begin{equation}
    \sup_{\lambda\in \mathcal{K}}\,
    d_{\yyy}
    \left(
            \rho_{\epsilon}(\lambda)
        ,
            \hat{\rho}(\lambda)
    \right)
        \leq   
    \frac{\epsilon}{8}
    \label{PROOF_thrm_main__3_2_first_estimate_control_of_rho____Finite_horizon_term}
        .
\end{equation}

Thus, together with \eqref{PROOF_thrm_main__3_1_first_estimate_control_of_f____Finite_horizon_term}, \eqref{PROOF_thrm_main__3_2_first_estimate_control_of_rho____Finite_horizon_term} and the monotonicity of the modulus of continuity $\omega_{\rho_{\epsilon}}$, we obtain the following estimate for every $n\in \nn_+$ with $|n|\leq N_T$ and any $\xb\in K_{t_n-dm(\epsilon):t_n}$:
\begin{align}
\begin{split}
\label{eq_bound_TERM_A_BEGIN}
 &       d_{\yyy}
    \left(
            \rho_{\epsilon}\circ f_{\epsilon}
                (t_n,x_{t_n-dm(\epsilon):t_n})
        ,
            \hat{\rho}\circ \hat{f}_{\theta_{n}}
                (x_{t_n-dm(\epsilon):t_n})
    \right)\\
  &  \leq 
        d_{\yyy}
    \left(
            \rho_{\epsilon}\circ f_{\epsilon}
                (t_n,x_{t_n-dm(\epsilon):t_n})
        ,
            \rho_{\epsilon}\circ \hat{f}_{\theta_{n}}
                (x_{t_n-dm(\epsilon):t_n})
    \right)+
        d_{\yyy}
    \left(
            \rho_{\epsilon}\circ \hat{f}_{\theta_{n}}
                (x_{t_n-dm(\epsilon):t_n})
        ,
            \hat{\rho}\circ \hat{f}_{\theta_{n}}
                (x_{t_n-dm(\epsilon):t_n})
    \right)
        \\
&    \leq 
    \omega_{\rho_{\epsilon}}\left(
    \left\|
        f_{\epsilon}
                (t_n,x_{t_n-dm(\epsilon):t_n})
        -
            \hat{f}_{\theta_{n}}
                (x_{t_n-dm(\epsilon):t_n})
    \right\|
    \right)
         +
        d_{\yyy}
    \left(
            \rho_{\epsilon}\circ \hat{f}_{\theta_{n}}
                (x_{t_n-dm(\epsilon):t_n})
        ,
            \hat{\rho}\circ \hat{f}_{\theta_{n}}
                (x_{t_n-dm(\epsilon):t_n})
    \right)
    \\
    & \leq 
    \omega_{\rho_{\epsilon}}{\left(\left(\frac{\epsilon}{8  L_{\alpha,\rho_{\epsilon/4}}}\right)^{1/\alpha}
    \right)}
     +
    d_{\yyy}\left(
            \rho_{\epsilon}\circ \hat{f}_{\theta_{n}}
                (x_{t_n-dm(\epsilon):t_n})
        ,
            \hat{\rho}\circ \hat{f}_{\theta_{n}}
                (x_{t_n-dm(\epsilon):t_n})
    \right)
\\
&    \leq 
    \frac{\epsilon}{8}+\frac{\epsilon}{8}=\frac{\epsilon}{4}
    .
\end{split}
\end{align}

Next, we build the hypernetwork $h$ which implements our transformer network's recursive structure.  Let $P \eqdef  P([d^{(0)}])$ and, for each $n=-N_T,\dots,N_T$, define
\begin{equation}
    \tilde{\theta}_{n}  \eqdef 
    a
    \,
    \theta_n-b
    \label{PROOF_thrm_main__4_hypernetwork_construction_eq_data_to_memorize}
    ,
\end{equation}
where $a>0$ and $b\in \rr^{P}$ are such that $\tilde{\theta}_{n}\in [0,1]^P$ for each $n=-N_T,\dots,N_T$, and where $\{\theta_n\}_{n=-N_T}^{N_T}$ are as in~\eqref{PROOF_thrm_main__3_1_first_estimate_control_of_f____Finite_horizon_term}.
We may therefore apply \citep[Theorem 3.1 (ii)]{YunSraJadbabaie2019GoodMemoryCapacity__Memorization} to conclude that there is a feedforward neural network $\tilde{h}\in \NN[[P,M,M,P]][\operatorname{ReLU}]$ such that, for each $|n|\leq N_T$,
\[
\tilde{h}
(\tilde{\theta}_{n})
    = 
\tilde{\theta}_{{n+1}},
\]
and where $M$ is the smallest positive integer satisfying
\[
2
\left\lfloor
\frac{M}{2}
\right\rfloor
\left\lfloor
\frac{M}{4P}
\right\rfloor
    \geq 
N_T
.
\]
Define 
$h(z) \eqdef  a^{-1} (\tilde{h}(az-b) + b)$.
Then, we extend $\{\theta_n\}_{n=-N_T}^{N_T}$ to an infinite sequence $(\theta_n)_{n\in \zz}$ by%
\[
\theta_n \eqdef \theta_{-N_T} \; \mbox{ for } n< -N_T
    \quad\mbox{ and }\quad
\theta_{n}
     \eqdef 
    \theta_{N_T} \; \mbox{ for } n> N_T
.
\]
We henceforth set $\hat{F} 
     \eqdef 
F^{(\hat{\rho},h,
        \theta_{-N_T},N_T
)}$.
By construction, then, \eqref{eq_bound_TERM_A_BEGIN} and the fact that $c(n)=1$ for $|n|\leq N_T$ imply that
\[
\max_{|n|\leq N_T}\sup_{\xb\in\kkk}
    \frac{
d_\yyy\left(
        F^{\rho_{\epsilon},f_{\epsilon}}(\xb)_{t_n}
            ,
        \hat{F}(\xb)_{t_n}
\right)
    }{
     c'(n,\epsilon)
    }
    \leq 
    \max_{|n|\leq N_T}\sup_{\xb\in\kkk}
d_\yyy\left(
        F^{\rho_{\epsilon},f_{\epsilon}}(\xb)_{t_n}
            ,
        \hat{F}(\xb)_{t_n}
\right)
\leq \frac{\epsilon}{4}
.
\]

\textbf{Step 3 - Control of the RGT decay term ~\eqref{PROOF_thrm_main__2_main_inequality_to_control_Term_C_Ergodicity_Type_Term}:}\hfill\\

Observe that $H(t_n,\theta_n)=\theta_{-N_{T}}$ for all $n\leq -N_T$, 
and $H(t_n,\theta_n)=\theta_{N_{T}}$ for all $n\geq N_T$.  
Moreover, since we have defined $\hat{F}$, we may define the map $c^{\hat{F}}_{\kkk,N_T, 8/\eps }(n)$ as in~\eqref{eq:self_compression_function}.  
Therefore, as long as $c:\mathbb{Z}\rightarrow [1,\infty)$ satisfies the following property for every $|n|>N_T$:
\begin{equation}
\label{eq:proof__compressionfunction_growthcontrol}
        c(n)
    \ge 
        c^{\hat{F}}_{\kkk,N_T, 8/\eps  }(n)
,
\end{equation}
then we can control term~\eqref{PROOF_thrm_main__2_main_inequality_to_control_Term_C_Ergodicity_Type_Term}.  This is because, writing out $c^{\hat{F}}_{\kkk,N_T, 8/\eps }$ we obtain
\[
    c^{\hat{F}}_{\kkk,N_T,8/\eps}(n)
    =
        \max\biggl\{1,
            \sup_{x\in \kkk}\,
                \frac{8}{\eps}\,
                d_{\yyy}\big(
                    \hat{F}(x)_{t_n}
                ,
                    \hat{F}(x)_{t_{N_T}}
                \big)
                I_{\{n\ge N_T\}}
                +
                 \frac{8}{\eps} \,
                d_{\yyy}\big(
                    \hat{F}(x)_{t_n}
                ,
                    \hat{F}(x)_{t_{-N_T}}
                \big)
                I_{\{n\le -N_T\}}
            \biggr\}
    .
\]
Thus, we may now control term~\eqref{PROOF_thrm_main__2_main_inequality_to_control_Term_C_Ergodicity_Type_Term} by
\allowdisplaybreaks
\begin{align}
    \nonumber
    &
    \sup_{n> N_T}\sup_{\xb\in \kkk}
        \frac{
    d_\yyy\left(
            \hat{F}(\xb)_{t_n}
                ,
            \hat{F}(\xb)_{t_{N_T}}
    \right)
        }{
         c'(n,\epsilon)
        }
    +
    \sup_{n< -N_T}\sup_{\xb\in \kkk}
        \frac{
    d_\yyy\left(
            \hat{F}(\xb)_{t_n}
                ,
            \hat{F}(\xb)_{t_{-N_T}}
    \right)
        }{
          c'(n,\epsilon)
        }
    \\
\label{PROOF_thrm_main__4__DONE__second_estimate_control___DONE}
    \le &
    \sup_{n> N_T}\sup_{\xb\in \kkk}
        \frac{
            d_\yyy\left(
                    \hat{F}(\xb)_{t_n}
                        ,
                    \hat{F}(\xb)_{t_{N_T}}
            \right)
        }{
         c^{\hat{F}}_{\kkk,N_T, 8/\eps }(n)
        }
    +
    \sup_{n< -N_T}\sup_{\xb\in \kkk}
        \frac{
            d_\yyy\left(
                    \hat{F}(\xb)_{t_n}
                        ,
                    \hat{F}(\xb)_{t_{-N_T}}
            \right)
        }{
          c^{\hat{F}}_{\kkk,N_T, 8/\eps }(n)
        }
    \\
    = &
    \nonumber
    \frac{\eps}{8}\,
    \sup_{n> N_T}\sup_{\xb\in \kkk}
        \frac{
            d_\yyy\left(
                    \hat{F}(\xb)_{t_n}
                        ,
                    \hat{F}(\xb)_{t_{N_T}}
            \right)
        }{
         \max\big\{
                1
            ,
                d_\yyy\left(
                    \hat{F}(\xb)_{t_n}
                        ,
                    \hat{F}(\xb)_{t_{N_T}}
                \right)
            \big\}
        }
    +
    \frac{\eps}{8}\,
    \sup_{n< -N_T}\sup_{\xb\in \kkk}
        \frac{
            d_\yyy\left(
                    \hat{F}(\xb)_{t_n}
                        ,
                    \hat{F}(\xb)_{t_{-N_T}}
            \right)
        }{
          \max\big\{
                    1
                ,
                    d_\yyy\left(
                        \hat{F}(\xb)_{t_n}
                        ,
                        \hat{F}(\xb)_{t_{-N_T}}
                    \right)
            \big\}
        }
    \\
    \nonumber
    \le & \frac{\epsilon}{4}
    .
\end{align}
\textbf{Step 4 - Control of the infinite-time horizon term~\eqref{PROOF_thrm_main__2_main_inequality_to_control_Term_B_Extrapolation_quantity}:}\hfill\\
It remains to control the terms in~\eqref{PROOF_thrm_main__2_main_inequality_to_control_Term_B_Extrapolation_quantity} by controlling~\eqref{PROOF_thrm_main__2_main_inequality_to_control_Term_B_Extrapolation_quantity___pure_term_B}.
This part is divided in different cases, as depending on the particular form of $\kkk$ we get a different compression rate $c$.
For notational simplicity, we may now define the compression rate $c$ by
\[
    c(n)
        \eqdef 
    \max\biggl\{
        c^{\hat{F}}_{\kkk,N_T, 8/\eps }
    ,
        \tilde{c}(n)
    \biggr\}
,
\]
for $n\in \mathbb{Z}$, where $\tilde{c}$ is defined on a case-by-case basis in the following Cases ($1$-$5$).

\textit{Case} $1$:
Let $w:\mathbb{Z}\rightarrow [0,\infty)$ be a weighting function, $K\subseteq \mathbb{R}^d$, and consider the case where $\kkk$ is
\[
        \kkk
    =
        K^{w} 
    \eqdef  
        \big\{
        \xb\in {\xxx}^{\mathbb{Z}}:\, \exists y \in K \mbox{ s.t.\ }
        \|x_{t_i}-y\|\le w(|i|)
        \big\}
    .
\]
Since $F$ is an AC map, $f_{\epsilon}$ and $\rho_{\epsilon}$ are uniformly continuous with respective moduli of continuity $\omega_{f_{\epsilon}}(u)=L_{\alpha,f_{\epsilon}}|u|^{\alpha}$ and $\omega_{\rho_{\epsilon}}(u)=L_{\alpha,\rho_{\epsilon}}|u|^{\alpha}$ for some $L_{\alpha,f_{\epsilon}},L_{\alpha,\rho_{\epsilon}}\geq 0$ and some $\alpha \in (0,1]$.  Thus, for every $n,n'\in \zz$ with $n<n'$ and every $\xb\in \kkk$, we compute the estimate
\begin{equation}
\label{eq:PROOF_thrm_main__Step4_Case0_first_inequality__setup_definitionKw}
\begin{aligned}
d_{\yyy}\left(
F^{\rho_\epsilon,f_\epsilon}(\xb)_{t_n}
,
F^{\rho_\epsilon,f_\epsilon}(\xb)_{t_{n'}}
\right)
& = 
d_{\yyy}\left(
\rho_{\epsilon}\circ f_{\epsilon}(t_n,x_{t_{n-dm(\epsilon)}:t_n})
,
\rho_{\epsilon}\circ f_{\epsilon}(t_{n'},x_{t_{n'-dm(\epsilon)}:t_{n'}})
\right)
\\
& \le 
\omega_{\rho_{\epsilon}}\circ \omega_{f_{\epsilon}}\big(
|t_n - t_{n'}|
+
\|
x_{t_{n-dm(\epsilon)}:t_n} 
-
x_{t_{n'-dm(\epsilon)}:t_{n'}}
\|
\big)
\\
& \le 
    \omega_{\rho_{\epsilon}}\circ \omega_{f_{\epsilon}}\biggl(
            |n-n'|\delta_{+}
        +
            dm(\epsilon)\, 
            \max_{i=1,\dots,dm(\epsilon)}\,
            \|
                    [x_{t_{n-dm(\epsilon)}:t_n}]_i
                -
                    [x_{t_{n'-dm(\epsilon)}:t_{n'}}]_i
            \|
    \biggr)
,
\end{aligned}
\end{equation}
where $[\cdot]_i$ denotes the $i^{th}$ canonical projection onto the Cartesian product $(\mathbb{R}^d)^{m(\epsilon)}\rightarrow \mathbb{R}^d$.  
Since each $[x_{t_{n-dm(\epsilon)}:t_n}]_i$ (resp.\ $[x_{t_{n'-dm(\epsilon)}:t_{n'}}]_i$) belongs to $K^w$, there is some $y_{i,n}$ (resp.\ $y_{i,n'}$) in $K$ satisfying the bound $\|[x_{t_{n-dm(\epsilon)}:t_{n}}]_i-y_{i,n}\|\le w(|n|)$ (resp.\ $\|[x_{t_{n'-dm(\epsilon)}:t_{n'}}]_i-y_{i,n'}\|\le w(|n'|)$).  
Therefore, the right-hand side of~\eqref{eq:PROOF_thrm_main__Step4_Case0_first_inequality__setup_definitionKw} can be bounded from above as
\begin{equation}
\label{eq:PROOF_thrm_main__Step4_Case0_first_inequality__CompletedRHSBound}
\begin{aligned}
    d_{\yyy}\left(
        F^{\rho_\epsilon,f_\epsilon}(\xb)_{t_n}
            ,
        F^{\rho_\epsilon,f_\epsilon}(\xb)_{t_{n'}}
    \right)
& \le 
    \omega_{\rho_{\epsilon}}\circ \omega_{f_{\epsilon}}\biggl(
            |n-n'|\delta_{+}
        + 
            dm(\epsilon)\, 
            \max_{i=1,\dots,dm(\epsilon)}\,
            \Big(
                \|
                        [x_{t_{n-dm(\epsilon)}:t_n}]_i
                    -
                        y_{i,n}
                \|
        \\ \qquad\qquad &+
                \|
                        y_{i,n}
                    -
                        y_{i,n'}
                \|
        +
                \|
                        [x_{t_{n'-dm(\epsilon)}:t_{n'}}]_i
                    -
                        y_{i,n'}
                \|
            \Big)
    \biggr)
\\
& \le 
    \omega_{\rho_{\epsilon}}\circ \omega_{f_{\epsilon}}\biggl(
            |n-n'|\delta_{+}
        + 
            dm(\epsilon)\, 
            \Big(
                w(|n|) 
            +
                \operatorname{diam}(K) 
            +
                w(|n'|)
            \Big)
    \biggr)
.
\end{aligned}
\end{equation}
Define
\[
    \tilde{c}
    :\zz \ni n 
        \mapsto
    \frac{4}{\epsilon}
        \omega_{\rho_{\epsilon}}\circ \omega_{f_{\epsilon}}\biggl(
                (|n|-N_T)\delta_{+}
            + 
               d m(\epsilon)\, 
                \Big(
                    w(|n|) 
                +
                   \operatorname{diam}(K) 
                +
                    w(N_T)
                \Big)
        \biggr)
.
\]
Then \eqref{eq:PROOF_thrm_main__Step4_Case0_first_inequality__CompletedRHSBound} implies that 
\begin{equation}\label{eq:eps4a}
        \sup_{|n|>N_T}\sup_{\xb\in \kkk}\,
        \frac{
        d_{\yyy}\left(
        F^{\rho_\epsilon,f_\epsilon}(\xb)_{t_n}
            ,
        F^{\rho_\epsilon,f_\epsilon} (\xb)_{t_{N_T}}
        \right)}{c'(n,\epsilon)}
    \leq
        \frac{\epsilon}{4}
.
\end{equation}
Thus, in this case, term~\eqref{PROOF_thrm_main__2_main_inequality_to_control_Term_C_Ergodicity_Type_Term} is controlled by $ \epsilon/4$.  

\textit{Case} $2$:
Let $K$ be a compact subset of $\rr^d$, $C,p>0$, and consider the case where
\[
\kkk
    =
K_{C,p}^{\infty}
    =
\left\{\xb\in (\rr^d)^{\zz}:\,
x_0\in K,\,(\forall n\in \zz)\,
\|\Delta_n \xb\|^p \leq C 
    |\Delta t_n|
\right\}
.
\]
If $\xb\in \kkk$, then for every $n,n'\in \zz$ with $n< n'$ we have that
\allowdisplaybreaks
\begin{align}
\nonumber
\|x_{t_n} - x_{t_{n'}}\|
\leq &
\sum_{k=n-1}^{n'}\|\Delta_k\xb\|
\\
\nonumber
\leq &
 C^{\frac1{p}} \sum_{k=n}^{n'-1}|\Delta t_k|^{\frac1{p}}
\\
\leq &
(n'-n) C^{\frac1{p}} {\delta_+}^{\frac1{p}}<\infty,
\label{PROOF_thrm_main__Step4_Case1_first_inequality}
\end{align}
where we have applied Assumption~\ref{assumption_of_regular_gridmesh} to derive the last inequality.

Since $F$ is an AC map then, $f_{\epsilon}$ and $\rho_{\epsilon}$ are uniformly continuous, with respective moduli of continuity $\omega_{f_{\epsilon}}(u)=L_{\alpha,f_{\epsilon}}|u|^{\alpha}$ and $\omega_{\rho_{\epsilon}}(u)=L_{\alpha,\rho_{\epsilon}}|u|^{\alpha}$ for some $L_{\alpha,f_{\epsilon}},L_{\alpha,\rho_{\epsilon}}\geq 0$ and some $\alpha \in (0,1]$.
Thus, for every $n,n'\in \zz$ with $n<n'$ and every $\xb\in \kkk$, we estimate
\allowdisplaybreaks
\begin{align}
\nonumber
d_{\yyy}\left(
    F^{\rho_\epsilon,f_\epsilon}(\xb)_{t_n}
        ,
    F^{\rho_\epsilon,f_\epsilon}(\xb)_{t_{n'}}
\right)
& = 
d_{\yyy}\left(
\rho_{\epsilon}\circ f_{\epsilon}(t_n,x_{t_{n-dm(\epsilon)}:t_n})
    ,
\rho_{\epsilon}\circ f_{\epsilon}(t_{n'},x_{t_{n'-dm(\epsilon)}:t_{n'}})
\right)
\\
\nonumber
& \leq 
\omega_{\rho_{\epsilon}}\circ \omega_{f_{\epsilon}}\left(
|t_{n'}-t_n|
+
\|
x_{t_{n-dm(\epsilon)}:t_n}
    - 
x_{t_{n'-dm(\epsilon)}:t_{n'}}
\|
\right)
\\
\label{PROOF_thrm_main__Step4_Case1_first_inequality_implied_estimateinY}
& \leq 
\omega_{\rho_{\epsilon}}\circ \omega_{f_{\epsilon}}\left(
|t_{n'}-t_n|
+
\sum_{i=n-dm(\epsilon)}^n
\|
x_{t_i} 
    - 
x_{t_{i+n'-n}}
\|
\right)
\\
&
\leq 
\omega_{\rho_{\epsilon}}\circ \omega_{f_{\epsilon}}\left(
(n'-n)\delta_+
+
(dm(\epsilon) +1)(n'-n) C^{\frac1{p}} {\delta_+}^{\frac1{p}}
\right)
\label{PROOF_thrm_main__Step4_Case1_first_inequality_x}
\\
& = 
\label{PROOF_thrm_main__Step4_Case1_first_inequality_implied_estimateinY_completed_estimate}
\omega_{\rho_{\epsilon}}\circ \omega_{f_{\epsilon}}\left(
|n'-n|\delta_+
\left[
    1
+
    (dm(\epsilon)+1) C^{\frac1{p}} {\delta_+}^{\frac{1-p}{p}}
\right]
\right),
\end{align}
where we used the monotonicity of the moduli of continuity $\omega_{f_{\epsilon}}$ and $\omega_{\rho_{\epsilon}}$ in~\eqref{PROOF_thrm_main__Step4_Case1_first_inequality_implied_estimateinY}, and the estimate in \eqref{PROOF_thrm_main__Step4_Case1_first_inequality} to get \eqref{PROOF_thrm_main__Step4_Case1_first_inequality_x}.  
Now define
\[
    \tilde{c}
    :\zz \ni n 
        \mapsto
    \frac{4}{\epsilon}
    \omega_{\rho_{\epsilon}}\circ \omega_{f_{\epsilon}}\left(
    (|n|-N_T)
    \delta_+
        [
        1
    +
        (dm(\epsilon) +1) C^{\frac1{p}} {\delta_+}^{\frac{1-p}{p}}
    ]
    \right)
    .
\]
Then \eqref{PROOF_thrm_main__Step4_Case1_first_inequality_implied_estimateinY_completed_estimate} implies that 
\begin{equation}\label{eq:eps4}
\sup_{|n|>N_T}\sup_{\xb\in \kkk}\,
\frac{
d_{\yyy}\left(
F^{\rho_\epsilon,f_\epsilon}(\xb)_{t_n}
    ,
F^{\rho_\epsilon,f_\epsilon} (\xb)_{t_{N_T}}
\right)}{c'(n,\epsilon)}
\leq
 \frac{\epsilon}{4}
.
\end{equation}
Thus, in this case, term~\eqref{PROOF_thrm_main__2_main_inequality_to_control_Term_C_Ergodicity_Type_Term} is controlled by $ \epsilon/4$.  

\textit{Case} $3$:
Let $K$ be a compact subset of $\rr^d$, $C>0$, $p\geq 1$, and $\alpha<1-p$.  Consider the case where
\[
\kkk=K_{C,p}^{\alpha}
=\left\{\xb\in (\rr^d)^{\zz}:\,
x_0\in K,\,
\sum_{n\in \zz}
\,
\frac{
\|\Delta_n \xb\|^p}
{
    |\Delta t_n| {|n|_{++}}^{\alpha}
} \leq C
\right\}
.
\]
For $\xb\in \kkk$, then, for every $n,n'\in \zz$ with $n< n'$, we have that
\allowdisplaybreaks
\begin{align}
\nonumber
    \|x_{t_n} - x_{t_{n'}}\|
\leq &
    \sum_{k=n-1}^{n'}
        \|\Delta_k\xb\|
\\
\nonumber
= &
    \sum_{k=n-1}^{n'} 
        \frac{
            \|\Delta_k\xb\|
        }{
            |\Delta t_k|^{\frac1{p}}{|k|_{++}}^{\frac{\alpha}{p}}
        }   
            \,
        |\Delta t_k|^{\frac1{p}}{|k|_{++}}^{\frac{\alpha}{p}}
\\
\nonumber
\leq &
    \left(
    \sum_{k=n-1}^{n'} 
        \left(
            \frac{
                \|\Delta_k\xb\|
            }{
                |\Delta t_k|^{\frac1{p}}{|k|_{++}}^{\frac{\alpha}{p}}
            } 
        \right)^p
    \right)^{\frac1{p}}
            \,
    \left(
    \sum_{k=n-1}^{n'} 
        \left(
            |\Delta t_k|^{\frac1{p}}{|k|_{++}}^{\frac{\alpha}{p}}
        \right)^{\frac{p}{p-1}}
    \right)^{\frac{p-1}{p}}
\\
\nonumber
= &
    \left(
    \sum_{k=n-1}^{n'} 
            \frac{
                \|\Delta_k\xb\|^p
            }{
                |\Delta t_k|{|k|_{++}}^{\alpha}
            } 
    \right)^{\frac1{p}}
            \,
    \left(
    \sum_{k=n-1}^{n'} 
            |\Delta t_k|^{\frac1{p-1}}{|k|_{++}}^{\frac{\alpha}{p-1}}
    \right)^{\frac{p-1}{p}} 
    \\
\nonumber
\leq &
    C^{\frac{1}{p}}
    \left(
    \sum_{k=n-1}^{n'} 
            |\Delta t_k|^{\frac1{p-1}}{|k|_{++}}^{\frac{\alpha}{p-1}}
    \right)^{\frac{p-1}{p}} 
\\
\label{PROOF_thrm_main__Step4_Case2_fist_inequality_a}
\leq &
    C^{\frac{1}{p}}
    {\delta_+}^{\frac{1}{p}}
    \left(
    \sum_{k=n-1}^{n'} 
            {|k|_{++}}^{\frac{\alpha}{p-1}}
    \right)^{\frac{p-1}{p}}
    .
\end{align}
Since $\frac{\alpha}{p-1}<-1$, then $\sum_{z\in \zz} {|z|_{++}}^{\frac{\alpha}{p-1}}$ is a p-series and it converges to $1+2\zeta(\frac{\alpha}{p-1})<\infty$, where $\zeta$ denotes the Riemann zeta-function.\footnote{
The Riemann zeta-function is defined by 
$\zeta:\rr\ni s\mapsto \sum_{n=1}^{\infty} n^{-s} \in [-\infty,\infty]$
. }
Therefore, the right-hand side of~\eqref{PROOF_thrm_main__Step4_Case2_fist_inequality_a} implies the estimate
\begin{equation}
    \|x_{t_n} - x_{t_{n'}}\|    
\leq 
    C^{\frac{1}{p}}
    {\delta_+}^{\frac{1}{p}}
 \left( 1+2 \zeta  \alpha/(p-1)
    \right)^{\frac{p-1}{p}}
    \label{PROOF_thrm_main__Step4_Case2_fist_inequality_b}
    .
\end{equation}

Analogously to the previous case, we have that for every $n,n'\in \zz$ and every $\xb\in \kkk$,~\eqref{PROOF_thrm_main__Step4_Case2_fist_inequality_b} implies the estimate
\begin{equation}
    \label{PROOF_thrm_main__Step4_Case2_first_inequality_implied_estimateinY_completed_estimate}
d_{\yyy}\left(
F^{\rho_\epsilon,f_\epsilon}(\xb)_{t_n}
    ,
F^{\rho_\epsilon,f_\epsilon}(\xb)_{t_{n'}}
\right)\leq \omega_{\rho_{\epsilon}}\circ \omega_{f_{\epsilon}}\left(
    |n'-n|
    \delta_+ +(dm(\epsilon)+1)C^{\frac{1}{p}}
    {\delta_+}^{\frac{1}{p}}
     \left( 1+2 \zeta  \alpha/(p-1)
    \right)^{\frac{p-1}{p}}
    \right).
\end{equation}

By defining the extrapolation function as
\[
    \tilde{c}
    :\zz \ni n
        \mapsto
     \frac{4}{\epsilon}
    \omega_{\rho_{\epsilon}}\circ \omega_{f_{\epsilon}}\left(
        (|n|-N_T)
    \delta_+ +(dm(\epsilon)+1)C^{\frac{1}{p}}
        {\delta_+}^{\frac{1}{p}}
         \left( 1+2 \zeta  \alpha/(p-1)
        \right)^{\frac{p-1}{p}}
        \right),
\]
we obtain the same estimate as in \eqref{eq:eps4}.

\textit{Case} $4$:
Let $K$ be a compact subset of $\rr^d$, and consider the case where $\kkk=K^{\zz}$. For $\xb\in \kkk$ then, for every $n,n'\in \zz$, we have that
\begin{align}
    d_{\yyy}\left(
F^{\rho_\epsilon,f_\epsilon}(\xb)_{t_n}
    ,
F^{\rho_\epsilon,f_\epsilon}(\xb)_{t_{n'}}
\right)
\leq &
        \omega_{\rho_{\epsilon}}\circ \omega_{f_{\epsilon}}\left(
            |t_n-t_{n'}|
                +
            \|x_{t_{n-dm(\epsilon)}:t_n} - x_{t_{n'-dm(\epsilon)}:t_{n'}}\|
        \right)
    \\
\leq &
        \omega_{\rho_{\epsilon}} \circ \omega_{f_{\epsilon}}\left(
            |n-n'| \delta_+
                +
            (dm(\epsilon)+1)
            \operatorname{diam}(K)
        \right)
.
\label{eq_PROOF_MAIN_Cylender_case_3}
\end{align}
Therefore, we may define the extrapolation function by
\[
    \tilde{c}
    :\zz \ni n
        \mapsto
        \frac{4}{\epsilon}
    \omega_{\rho_{\epsilon}} 
        \circ 
    \omega_{f_{\epsilon}}\left(
            (|n|-N_T){\delta_+}
                +
             (dm(\epsilon)+1)\operatorname{diam}(K)
        \right),
\]
and obtain the same estimate as in \eqref{eq:eps4}.

\textit{Case} $5$:
Let $\kkk\subseteq (\rr^d)^{\zz}$ be an arbitrary compact set. In this case we define the extrapolation function by
\[
    \tilde{c}
    :\zz \ni n\mapsto 
    \max_{\xb\in \kkk,\, k\leq |n|}\,
    \,
    \frac{4}{\epsilon}
    \max\left\{1,
    d_{\yyy}\left(
    F^{\rho_\epsilon,f_\epsilon}(\xb)_{t_k}
        ,
    F^{\rho_\epsilon,f_\epsilon} (\xb)_{t_{n}}
    \right)
    \right\},
\]
and obtain the same estimate as in \eqref{eq:eps4}.
\end{proof}

\begin{proof}[{Proof of Corollary~\ref{cor_reservoir_analogue___time_homogeneous_case}}]
The proof is exactly as in Theorem~\ref{thrm_main_DynamicCase} without the term in time in~\eqref{eq_PROOF_MAIN_Cylender_case_3}:
\[
d_{\yyy}\left(
F^{\rho_\epsilon,f_\epsilon}(\xb)_{t_n}
    ,
F^{\rho_\epsilon,f_\epsilon}(\xb)_{t_{n'}}
\right)
    \leq 
\omega_{\rho_{\epsilon}}\circ 
\omega_{f_{\epsilon}}\left(
(dm(\epsilon)+1)
\operatorname{diam}(K)
\right).
\]
\end{proof}

\section{Acknowledgements}
\label{s_Acknowledgements}

This research was supported by the ETH Z\"{u}rich Foundation and the European Research Council (ERC) Starting Grant 852821—SWING.  The authors would like to thank Juan-Pablo Ortega for his helpful discussion surrounding the fading memory property, as well as Behnoosh Zamanlooy, Giulia Livieri, and Ivan Dokmani\'{c} for their valuable feedback.  The authors would also like to thank Valentin Debarnot for his very helpful model-visualization advice.

\bibliographystyle{plainnat}
\bibliography{2_References}
\end{document}